%% file: main.tex
\documentclass{article}

\usepackage[accepted]{icml2025} 
\usepackage{thmtools}

\usepackage{graphicx} %
\usepackage[colorlinks=True,citecolor=blue,urlcolor=blue,pagebackref=true,backref=true,linkcolor=blue,filecolor=blue]{hyperref}
\renewcommand*{\backrefalt}[4]{%
    \ifcase #1 \footnotesize{(Not cited.)}%
    \or        \footnotesize{(Cited on page~#2.)}%
    \else      \footnotesize{(Cited on pages~#2.)}%
    \fi}
\usepackage{amssymb,amsthm, amsmath}
\usepackage{url,nicefrac}
\usepackage[ruled,algo2e]{algorithm2e}%
\usepackage{etoolbox}
\makeatletter
\patchcmd{\@algocf@start}%
  {-1.5em}%
  {0pt}%
  {}{}%
\makeatother

\usepackage{subcaption}   %
\usepackage{float}        %

\usepackage{enumitem}
\usepackage[usestackEOL]{stackengine}
\usepackage{booktabs} %
\usepackage{graphicx} %
\usepackage{mathtools}
\usepackage{etoc} %
\usepackage{dsfont}
\usepackage{enumitem}  
\usepackage{wrapfig}
\usepackage{graphicx}
\usepackage{multirow}

\input{cref_setup}
\input{macros}
\newcommand{\mytitle}{Low-Rank Thinning}
\title{\mytitle}%
\icmltitlerunning{\mytitle}

\begin{document}

\twocolumn[
\icmltitle{\mytitle}
\begin{icmlauthorlist}
\icmlauthor{Annabelle Michael Carrell}{cambridge}
\icmlauthor{Albert Gong}{cornell}
\icmlauthor{Abhishek Shetty}{mit}
\icmlauthor{Raaz Dwivedi}{cornell}
\icmlauthor{Lester Mackey}{msr}
\end{icmlauthorlist}

\icmlaffiliation{cambridge}{University of Cambridge}
\icmlaffiliation{cornell}{Cornell Tech}
\icmlaffiliation{msr}{Microsoft Research New England}
\icmlaffiliation{mit}{MIT}%

\icmlcorrespondingauthor{Annabelle Carrell}{ac2411@cam.ac.uk}
\icmlcorrespondingauthor{Albert Gong}{agong@cs.cornell.edu}
\icmlcorrespondingauthor{Abhishek Shetty}{ashetty1995@gmail.com}
\icmlcorrespondingauthor{Raaz Dwivedi}{dwivedi@cornell.edu}
\icmlcorrespondingauthor{Lester Mackey}{lmackey@microsoft.com}

\icmlkeywords{sub-gaussian,thinning,distribution compression,kernel maximum mean discrepancy,low-rank,fast attention,sgd reordering,hypothesis testing}

\vskip 0.3in
]
\printAffiliationsAndNotice{} %
\etoctocstyle{1}{Table of contents}
\etocdepthtag.toc{mtchapter}
\etocsettagdepth{mtchapter}{section}
\input{sections/abstract}
\input{sections/introduction}
\input{sections/related_work}
\input{sections/problem_setup}
\input{sections/transformers}
\input{sections/sgd}
\input{sections/testing_new}
\input{sections/impact}
\input{sections/acknowledgments} %
\bibliographystyle{icml2025}

\vspace{\baselineskip}
{\small \bibliography{refs}}
\newpage\clearpage
\appendix
\input{appendix}
\end{document}

%% file: cref_setup.tex
\usepackage{thmtools,thm-restate}
\usepackage[capitalize]{cleveref}  

\usepackage{autonum}

\newtheorem{theorem}{Theorem}
\newtheorem{corollary}{Corollary}
\newtheorem{proposition}{Proposition}
\newtheorem{lemma}{Lemma}
\newtheorem{assumption}{Assumption}
\newtheorem{definition}{Definition}
\theoremstyle{definition}
\newtheorem{remark}{Remark}

\crefname{appendix}{App.}{Apps.}
\crefname{subsubsubappendix}{App.}{Apps.}
\crefname{equation}{}{}
\crefname{assumption}{Assump.}{Assumps.}
\crefname{lemma}{Lem.}{Lems.}
\crefname{theorem}{Thm.}{Thms.}
\crefname{corollary}{Cor.}{Cors.}
\crefname{algorithm}{Alg.}{Algs.}
\crefname{algocf}{Alg.}{Algs.} %
\crefname{section}{Sec.}{Secs.}
\crefname{table}{Tab.}{Tabs.}
\crefname{remark}{Rem.}{Rems.}
\crefname{example}{Ex.}{Exs.}
\crefname{definition}{Def.}{Defs.}
\crefname{proposition}{Prop.}{Props.}
\crefname{mytheorem}{Thm.}{Thms.}
\crefname{myremark}{Rem.}{Rems.}
\crefname{mylemma}{Lem.}{Lems.}
\crefname{mydefinition}{Def.}{Defs.}
\crefname{myproposition}{Prop.}{Props.}
\crefname{mycorollary}{Cor.}{Cors.}
\crefname{enumi}{}{}
\crefname{name}{}{} %
\Crefformat{footnote}{#2\footnotemark[#1]#3}

%% file: macros.tex
\DeclareMathAlphabet{\mathbsf}{OT1}{cmss}{bx}{n}%
\DeclareMathAlphabet{\mathssf}{OT1}{cmss}{m}{sl}%

\newcommand{\numperm}{\mc B}

\newcommand{\defeq}{\triangleeq}

\newcommand{\ball}{\mathbb{B}}
\newcommand{\indicator}{\mbf 1}

\newcommand{\eventnotag}{\mc{E}}
\newcommand{\event}[1][]{\eventnotag_{#1}}

\newcommand{\snorm}[1]{\Vert #1 \Vert}

\newcommand{\cover}{\mc C}
\newcommand{\extcover}{\mc C_{\textup{ext}}} %
\newcommand{\intcover}{\mc C_{\textup{int}}} %
\newcommand{\coveringnumber}{\mc N}

\newcommand{\real}{\ensuremath{\mathbb{R}}}

\newcommand{\sumn}[1][i]{\sum_{#1=1}^n}
\newcommand{\sless}[1]{\stackrel{#1}{\leq}}

\newcommand{\seq}[1]{\stackrel{#1}{=}}

\newcommand{\w}{\mbi{w}}
\renewcommand{\v}{\mbi{v}}
\newcommand{\x}{\mbi{x}}
\newcommand{\y}{\mbi{y}}
\newcommand{\z}{\mbi{z}}

\newcommand{\vareps}{\varepsilon}
\newcommand{\dirac}{\mbi{\delta}}

\newcommand{\attention}{\textsc{Attention}}
\DeclareMathOperator{\mmd}{MMD}

\newcommand{\kernel}{\mbf{k}}
\renewcommand{\k}{\kernel}

\newcommand{\mkernel}{\mbf{K}}
\newcommand{\K}{\mkernel}
\newcommand{\Ktilde}{\widetilde{\mkernel}}

\newcommand{\rkhs}{\mc{H}_{\kernel}}

\newcommand{\knorm}[1]{\Vert{#1}\Vert_{\kernel}}

\newcommand{\vmax}[1][i]{\mathfrak{b}_{#1}}

\newcommand{\coreset}[1][j]{\mathcal{S}^{(#1)}}
\newcommand{\ind}{\mathcal{I}}
\newcommand{\indout}{\ind_{\textup{out}}}

\newcommand{\concat}{\texttt{concatenate}}

\newcommand{\brackets}[1]{\left[ #1 \right]}

\newcommand{\parenth}[1]{\left( #1 \right)}

\newcommand{\braces}[1]{\left\{ #1 \right \}}

\newcommand{\abss}[1]{\left| #1 \right |}

\newcommand{\angles}[1]{\left\langle #1 \right \rangle}

\newcommand{\floor}[1]{\left\lfloor #1 \right \rfloor}
\newcommand{\tp}{^\top}
\newcommand{\inv}{^{-1}}

\newcommand{\sig}{\sigma}
\newcommand{\lam}{\lambda}

\def\balign#1\ealign{\begin{align}#1\end{align}}
\def\baligns#1\ealigns{\begin{align*}#1\end{align*}}
\def\balignat#1\ealign{\begin{alignat}#1\end{alignat}}
\def\balignats#1\ealigns{\begin{alignat*}#1\end{alignat*}}
\def\bitemize#1\eitemize{\begin{itemize}#1\end{itemize}}
\def\benumerate#1\eenumerate{\begin{enumerate}#1\end{enumerate}}

\newenvironment{talign*}
 {\let\displaystyle\textstyle\csname align*\endcsname}
 {\endalign}
\newenvironment{talign}
 {\let\displaystyle\textstyle\csname align\endcsname}
 {\endalign}

\def\balignst#1\ealignst{\begin{talign*}#1\end{talign*}}
\def\balignt#1\ealignt{\begin{talign}#1\end{talign}}
\newcommand{\qtext}[1]{\quad\text{#1}\quad} 
\newcommand{\stext}[1]{\text{ #1 }} 
\newcommand{\sstext}[1]{\text{\ \ #1\ \ }} 

\let\originalleft\left
\let\originalright\right
\renewcommand{\left}{\mathopen{}\mathclose\bgroup\originalleft}
\renewcommand{\right}{\aftergroup\egroup\originalright}

\def\Holder{H\"older\xspace}

\def\tinycitep*#1{{\tiny\citep*{#1}}}
\def\tinycitealt*#1{{\tiny\citealt*{#1}}}
\def\tinycite*#1{{\tiny\cite*{#1}}}
\def\smallcitep*#1{{\scriptsize\citep*{#1}}}
\def\smallcitealt*#1{{\scriptsize\citealt*{#1}}}
\def\smallcite*#1{{\scriptsize\cite*{#1}}}

\def\mbi#1{\boldsymbol{#1}} %
\def\mbf#1{\mathbf{#1}}
\def\mbb#1{\mathbb{#1}}
\def\mc#1{\mathcal{#1}}
\def\mrm#1{\mathrm{#1}}

\def\tbf#1{\textbf{#1}}

\def\wtil#1{\widetilde{#1}}
\def\mfk#1{\mathfrak{#1}}

\newcommand{\boldzero}{{\boldsymbol{0}}}

\def\reals{\mathbb{R}} %
\def\R{\mathbb{R}}
\def\Z{\mathbf{Z}}
\def\Q{\mathbb{Q}}
\def\naturals{\mathbb{N}} %

\def\<{\left\langle} %
\def\>{\right\rangle}

\def\defeq{\triangleq} %
\def\half{\frac{1}{2}}
\def\quarter{\frac{1}{4}}

\newcommand{\textfrac}[2]{{\textstyle\frac{#1}{#2}}}
\newcommand{\boldone}{\mbf{1}} %
\newcommand{\ident}{\mbf{I}} %
\def\norm#1{\left\|{#1}\right\|} %
\def\snorm#1{\|{#1}\|} %
\newcommand{\twonorm}[1]{\staticnorm{#1}_2} %
\newcommand{\infnorm}[1]{\staticnorm{#1}_{\infty}} %
\newcommand{\maxnorm}[1]{\staticnorm{#1}_{\max}} %
\newcommand{\Kmax}{\maxnorm{\K}} %
\newcommand{\opnorm}[1]{\staticnorm{#1}_{\operatorname{op}}} %
\newcommand{\specnorm}[1]{\opnorm{#1}} %
\newcommand{\rownorm}[1]{\staticnorm{#1}_{2,\infty}} %

\def\staticnorm#1{\|{#1}\|} %
\newcommand{\statictwonorm}[1]{\staticnorm{#1}_2} %
\newcommand{\inner}[2]{\langle{#1},{#2}\rangle} %
\def\what#1{\widehat{#1}}

\newcommand{\maxeig}{\lambda_{\max}}

\def\indic#1{\indicator\left[{#1}\right]} %
\def\polylog{\operatorname{polylog}}
\def\E{\mbb{E}} %
\def\Earg#1{\E\left[{#1}\right]}
\def\Esub#1{\E_{#1}}
\def\Esubarg#1#2{\E_{#1}\left[{#2}\right]}
\def\bigO#1{O(#1)} %
\def\P{\mbb{P}} %

\def\Psubarg#1#2{\P_{#1}\left({#2}\right)}

\newcommand{\Unif}{\textnormal{Unif}}

\newcommand{\grad}{\nabla} %

\providecommand{\argmax}{\mathop\mathrm{arg max}} %
\providecommand{\argmin}{\mathop\mathrm{arg min}}

\providecommand{\diag}{\mathop\mathrm{diag}}
\providecommand{\tr}{\mathop\mathrm{tr}}

\providecommand{\diam}{\mathop\mathrm{diam}}

\def\rank#1{\mathrm{rank}({#1})}
\newcommand\epsrank[1][\eps]{\mathrm{rank}_{#1}}

\def\supp#1{\mathrm{supp}({#1})}

\ifdefined\nonewproofenvironments\else
\ifdefined\ispres\else
\newenvironment{proof-sketch}{\noindent\textbf{Proof Sketch}
  \hspace*{1em}}{\qed\bigskip\\}
\newenvironment{proof-idea}{\noindent\textbf{Proof Idea}
  \hspace*{1em}}{\qed\bigskip\\}
\newenvironment{proof-of-lemma}[1][{}]{\noindent\textbf{Proof of Lemma {#1}}
  \hspace*{1em}}{\qed\\}
\newenvironment{proof-of-theorem}[1][{}]{\noindent\textbf{Proof of Theorem {#1}}
  \hspace*{1em}}{\qed\\}
\newenvironment{proof-attempt}{\noindent\textbf{Proof Attempt}
  \hspace*{1em}}{\qed\bigskip\\}

\newcommand{\cset}{\mc{S}}
\newcommand{\alg}{\textsc{Alg}\xspace}

\newcommand{\ktswap}{\hyperref[algo:ktswap]{\color{black}{\textsc{kt-swap}}}\xspace}

\newcommand{\halve}{\textsc{Halve}\xspace}
\newcommand{\osname}{compression level\xspace}
\newcommand{\ossymb}{\ensuremath{\mfk{g}}\xspace}
\newcommand{\compress}{\textsc{Compress}\xspace}
\newcommand{\ktcompress}{\hyperref[app:ktcompress]{\color{black}{\textsc{KT-Compress}}}\xspace}

\newcommand{\ktcompressd}{\text{$\ktcompress(\delta)$}\xspace}

\newcommand{\tout}{\textup{out}}

\newcommand{\ctt}{\hyperref[algo:ctt]{\color{black} \textup{CTT}}\xspace}

\newcommand{\cttname}{{Compress Then Test}\xspace}

\newcommand{\tmmd}{\textsc{CoresetMMD}\xspace}

\newcommand{\nin}{n_{\textup{in}}}
\newcommand{\nout}{n_{\textup{out}}}

\newcommand{\ncref}[1]{\cref{#1}: \nameref*{#1}} %

\newcommand{\pcref}[1]{Proof of \ncref{#1}} %

\newcommand{\terror}[1][\kernel]{\mathbsf{\tilde{R}}_{#1}}
\newcommand{\error}[1][\kernel]{\mathbsf{R}_{#1}}

\newcommand{\errorhat}[1][\kernel]{\what{\mathbsf{R}}_{#1}}

\newcommand{\sblock}{s} %
\newcommand{\ckk}[1][\Ktilde]{C_{#1}}
\newcommand{\mkk}[1][\Ktilde]{\mathfrak{M}_{#1}}

\newcommand{\cnew}[1][i]{\mfk{a}_{#1}}

\newcommand{\bu}{\boldsymbol{u}}
\newcommand{\bQ}{\mathbf{Q}}
\newcommand{\bv}{\boldsymbol{v}}
\newcommand{\bz}{\boldsymbol{z}}

\newcommand{\X}{\mbf X}

\newcommand{\bPhi}{\boldsymbol{\Phi}}

\newcommand{\Rd}{\reals^d}
\newcommand{\Rn}{\reals^n}

\newcommand{\simplex}[1][n-1]{\mbi{\Delta}_{#1}}

\newcommand{\subg}{\nu}
\newcommand{\fsubg}{\phi}
\newcommand{\ksubg}{\mathcal{G}_{\nu}(\mkernel)}
\newcommand{\ksubge}[1][\mkernel]{\mathcal{G}_{\nu,\delta}(#1)}

\newcommand{\D}{\mbf{D}}
\newcommand{\hD}{\mbf{\hat{D}}}
\newcommand{\Dinv}{\D^{-1}}
\newcommand{\hDinv}{\hat{\D}^{-1}}
\newcommand{\dav}{\mbf D^{-1}\mbf A \mbf V}

\newcommand{\hdav}{\hat{\D}^{-1} \widehat{\A} \V}
\newcommand{\hddav}{\hat{\D}^{-1}\A \V}

\newcommand{\av}{\A \V}
\newcommand{\hav}{\widehat{\A} \V}
\newcommand{\ha}{\widehat{\A}}

\newcommand{\Katt}[1][]{\K_{\mathrm{att}#1}}
\newcommand{\attindnorm}[1][\ind]{\snorm{\Katt (\pin-\qout)}_{#1}}

\newcommand{\qout}[1][\textup{out}]{\mbi{p}_{#1}}
\newcommand{\pin}[1][\textup{in}]{\mbi{p}_{#1}}

\newcommand{\e}{\mbi{e}}

\newcommand{\indnorm}[1][\K (\pin-\qout)]{\norm{#1}_{\ind}}

\newcommand{\inddiam}{D_{\ind}}

\newcommand{\zznorm}[1][\Zset,\Zset]{\norm{(\Pin-\Qout)\k}_{#1}}
\newcommand{\fnorm}[1][\Fset]{\norm{\Pin-\Qout}_{#1}}

\newcommand{\bC}{\mathbf{C}}

\newcommand{\bLam}{\mbi{\Lambda}}
\newcommand{\diff}{\mbi{w}} %
\newcommand{\eps}{\epsilon}

\newcommand{\A}{\mbf A}

\newcommand{\V}{\mbf V}
\newcommand{\U}{\mbf U}
\newcommand{\Sig}{\mbf \Sigma}
\newcommand{\T}{\mbf T}
\newcommand{\That}{\mbf{\hat{T}}}
\newcommand{\Tthin}{\mbf{\hat{T}}_{\textrm{thin}}}
\newcommand{\Tkde}{\mbf{\hat{T}}_{\textrm{kde}}}
\newcommand{\Thyp}{\mbf{\hat{T}}_{\textrm{hyp}}}

\newcommand{\Fset}{\mathcal{F}}
\newcommand{\Zset}{\mathcal{Z}}
\newcommand{\Xset}{\mathcal{X}}
\newcommand{\Aset}{\mathcal{A}}
\newcommand{\Lset}{\mathcal{L}}

\renewcommand{\hat}[1]{\widehat{#1}}

\newcommand{\Eevent}{\Esub{\event}}
\newcommand{\Pevent}{\P_{\event}}

\newcommand{\mcJ}{\mc{J}}

\newcommand{\katt}{\kernel_{\mathrm{att}}}

\newcommand{\key}{\mbi{k}}
\newcommand{\augkey}{\tilde{\key}}
\newcommand{\val}{\mbi{v}}
\newcommand{\augval}{\tilde{\val}}
\newcommand{\augV}{\widetilde{\V}}
\newcommand{\query}{\mbi{q}}
\newcommand{\augquery}{\tilde{\query}}

\newcommand{\subsampling}{\hyperref[sub:subsampling]{\color{black}{\textsc{Subsampling}}}\xspace}
\newcommand{\ktsplit}{\hyperref[sub:kt-split]{\color{black}{\textsc{KT-Split}}}\xspace}

\newcommand{\kh}{\hyperref[algo:khd]{\color{black}{\textsc{KH}}}\xspace}
\newcommand{\khd}{\text{$\kh(\delta)$}\xspace}

\newcommand{\khsgd}{\text{$\khlin(\frac{1}{2K})$}\xspace}
\newcommand{\rkh}{\hyperref[algo:rkhd]{\color{black}{\textsc{RKH}}}\xspace}
\newcommand{\rkhd}{\text{$\rkh(\delta)$}\xspace}

\newcommand{\khcompress}{\hyperref[algo:khcompressd]{\color{black}{\textsc{KH-Compress}}}\xspace}
\newcommand{\khcompressd}{\text{$\khcompress(\delta)$}\xspace}
\newcommand{\khcompresshalf}{\text{$\khcompress(0.5)$}\xspace}
\newcommand{\gsthin}{\hyperref[algo:gs_thin]{\color{black}{\textsc{GS-Thin}}}\xspace}
\newcommand{\gshalve}{\hyperref[algo:gs_halve]{\color{black}{\textsc{GS-Halve}}}\xspace}
\newcommand{\gshalvecubic}{\hyperref[algo:gs_halve_cubic]{\color{black}{\textsc{GS-Halve-Cubic}}}\xspace}
\newcommand{\gscompress}{\hyperref[algo:gs_compress]{\color{black}{\textsc{GS-Compress}}}\xspace}

\newcommand{\khlin}{\hyperref[algo:khlin]{\color{black}{\textsc{LKH}}}\xspace}
\newcommand{\khlind}{\text{$\khlin(\delta)$}\xspace}

\newcommand{\xset}{\mc X}
\newcommand{\xin}{\xset_{\textup{in}}}
\newcommand{\xinj}[1][j]{\xset_{\textup{in},#1}}
\newcommand{\xoutj}[1][j]{\xset_{\textup{out},#1}}
\newcommand{\xp}{\xset'}
\newcommand{\xout}{\xset_{\textup{out}}}

\newcommand{\Pin}{\mbb P_{\textup{in}}}
\newcommand{\Qout}{\mbb P_{\textup{out}}}

\newcommand{\Otilde}{\widetilde{O}}

\newcommand{\fstar}{f^\star}
\newcommand{\jstar}{j^\star}
\newcommand{\xbar}{\bar{\x}}

\newcommand{\wtpsi}{\widetilde{\psi}} 
\newcommand{\thresh}{\mathfrak{a}}
\newcommand{\multiplier}{\mathfrak{b}}
\newcommand{\perm}{\pi}
\newcommand{\lfront}{\Pi}%
\newcommand{\lback}{\Pi'}%

\makeatletter
\newcommand\footnoteref[1]{\protected@xdef\@thefnmark{\ref{#1}}\@footnotemark}
\makeatother

%% file: sections/abstract.tex
\begin{abstract}
The goal in thinning is to summarize a dataset using a small set of representative points. Remarkably, sub-Gaussian thinning algorithms like Kernel Halving and Compress can match the quality of uniform subsampling while substantially reducing the number of summary points. However, existing guarantees cover only a restricted range of distributions and kernel-based quality measures and suffer from pessimistic dimension dependence. To address these deficiencies, we introduce a new low-rank analysis of sub-Gaussian thinning that applies to any distribution and any kernel, guaranteeing high-quality compression whenever the kernel or data matrix is approximately low-rank. To demonstrate the broad applicability of the techniques, we design practical sub-Gaussian thinning approaches that improve upon the best known guarantees for approximating attention in transformers, accelerating stochastic gradient training through reordering, and distinguishing distributions in near-linear time. 

\end{abstract}

%% file: sections/introduction.tex
\section{Introduction}
\label{sec:intro}

This work is about thinning,  %
finding a small set of representative points to accurately summarize a larger dataset. Here, we use the term ``dataset'' liberally to refer to any collection of points, be they experimental observations (as in \cref{sec:ctt}), stochastic gradients (as in \cref{sec:SGD}), or key-value pairs (as in \cref{sec:attn}).
State-of-the-art thinning 
techniques provably improve upon uniform subsampling but only for restricted classes of kernel-based quality measures 
and with pessimistic dependence on the data dimension \citep[see, e.g.,][]{harvey2014near,phillips2020near,alweiss2020discrepancyminimizationselfbalancingwalk,dwivedi2024kernel,dwivedi2021generalized,shetty2022distributioncompressionnearlineartime,li2024debiased}. 
We introduce a new analysis for sub-Gaussian thinning algorithms that applies to any kernel and shows that one can efficiently identify a better-than-uniform set of representative points whenever the 
kernel or data matrix 
is nearly low-rank.
This opens the door to a variety of impactful applications including approximate dot-product attention in transformers, accelerating model training through gradient reordering, %
and distinguishing distributions with deep kernels in near-linear time. 

In \cref{sec:problem-setup}, we introduce our formal definition of thinning, two kernel-based measures of thinning quality, and a suite of candidate thinning algorithms. 
\cref{sec:low-rank} presents our main theorem relating thinning quality to low-rank properties of a dataset and its induced kernel matrix. In \cref{sec:attn}, we translate the problem of approximating attention into a thinning problem and develop a practical solution, Thinformer, with state-of-the-art quality guarantees. In \cref{sec:SGD}, we develop a thinned stochastic gradient reordering rule that provably accelerates training and bridges the theory-practice gaps left open by prior work. 
Finally, in \cref{sec:ctt} we derive new and improved power guarantees for practical thinned hypothesis tests that distinguish distributions in near-linear time. 
Each section also includes its own discussion of related work.

\paragraph{Notation.}
For each $n\in\naturals$ and $a,b\in\reals$,  we define $[n] \defeq \{1,\dots,n\}$, $a \wedge b \defeq \min(a,b)$, and $a \vee b \defeq \max(a,b)$.
We let $\specnorm{\A}$, $\maxnorm{\A}$, and $\rownorm{\A}$ respectively represent the maximum singular value, absolute entry, and row Euclidean norm of a matrix $\A$ and let $\lam_{r}(\K)$ denote the $r$-th largest eigenvalue of a suitable matrix $\K$. 
We also define the Euclidean norm balls
$\ball^m 
    \defeq 
\{ \bu \in\reals^{m} : \twonorm{\bu} \leq 1\}$ 
and $\ball^m(R)\defeq R\,\ball^m$ for each $m\in\naturals$ and $R>0$.
For an event $\event$ and an integrable random variable $X$, we define $\Esub{\event}[X] \defeq \E[X\cdot\indic{\event}]$.
We write $a_n \leq \Otilde(b_n)$ to mean $a_n \leq b_n \polylog(n)$.

%% file: sections/problem_setup.tex
\section{Sub-Gaussian Thinning}\label{sec:problem-setup}
\begin{table*}[t]
    \centering
    \caption{\tbf{Examples of $\mbi{(\K,\subg,\delta)}$-sub-Gaussian thinning algorithms.} 
    For input size $\nin$, output size $\nout\geq \sqrt{\nin}$, %
    and $\Kmax=1$ 
    we report each 
    sub-Gaussian parameter $\subg$ and runtime up to constants independent of $(\nin,\nout,\delta,\K)$.
    }
    \resizebox{\textwidth}{!}{
    \begin{tabular}{ccc ccc}
    \toprule
    \bf{Algorithm}
    
    & \Centerstack{\bf \subsampling \\
    \tiny{\cref{prop:uniformsubg} } }
    
    & \Centerstack{\bf $\khd$ \\ \tiny{\cref{khd-sub-gaussian}} }
    
    & \Centerstack{\bf $\khcompress(\delta)$ \\ \tiny{\cref{khcompressd-sub-gaussian}}}
    
    & \Centerstack{\bf \gsthin \\ 
    \tiny{\cref{prop:gs_thin}}}
    
    & \Centerstack{\bf \gscompress  \\ \tiny{\cref{prop:gs_thin_compress}} } 
    \\[1mm]
    
    \midrule
    
    \Centerstack{\bf Sub-Gaussian \\ \bf parameter $\nu$} 
    & 
    {\large$\frac{1}{\sqrt{\nout}}$}
    & {\large $\frac{\sqrt{\log(\nout/\delta)}}{\nout}$} 
    & {\large$\frac{\sqrt{\log(\nout)\log(\nout/\delta)}}{\nout}$ }
    & {\large$\frac{1}{\nout}$ }
    &{\large$\frac{ \sqrt{\log(\nout)}}{\nout}$} 
    \\[3mm]
    
    \Centerstack{\bf Runtime}
    & $\nout$
    & $\nin^2$ 
    & $\nout^2\log\nout$ %
    & $\nin^3$
    & $\nout^{3}$
    \\[1mm]
    
    \bottomrule
    \end{tabular}
    }
    \label{tab:subg_thinning_algorithms}
\end{table*}

We begin with a formal definition of our problem setting. Consider a fixed collection of $\nin$ input points $\xin$ belonging to a potentially larger universe of datapoints $\xset\defeq\{\x_1,\dots,\x_n\}$.
The aim of a thinning algorithm is to select $\nout$ points from $\xin$ that together accurately summarize $\xin$. 
This is formalized by the following definition.

\begin{definition}[\tbf{Thinning algorithms}]
\label{def:thinning_algo}
A \emph{thinning algorithm} \alg takes as input $\xin$ and returns a possibly random subset $\xout$ of size $\nout$. 
We denote the input and output empirical distributions by $\Pin \defeq \frac1\nin \sum_{\x\in\xin} \dirac_{\x}$  and $\Qout \defeq \frac1\nout \sum_{\x\in\xout}\dirac_{\x}$ and 
define the induced probability vectors $\pin, \qout\in\simplex$ over the indices $[n]$ by
    \begin{talign}
        \pin[\textup{in}, i] = \frac{\indic{\x_i\in\xin}}{\nin}
        \sstext{and}
        \qout[\textup{out}, i] = \frac{\indic{\x_i\in\xout}}{\nout}
        \sstext{for all} i \in [n].
    \end{talign}
    When $\xset\subset\real^d$, we use $\X \defeq [\x_1, \ldots, \x_n]\tp  \in \real^{n\times d}$ to denote the input point matrix so that
    \begin{talign}
    \E_{\x\sim\Pin}[\x] = \X\tp \pin
    \qtext{and}
    \E_{\x\sim\Qout}[\x] = 
    \X\tp\qout.
\end{talign}
\end{definition}

We will make use of two common measures of summarization quality.\\[\baselineskip]

\begin{definition}[\tbf{Kernel MMD and max seminorm}]\label{def:mmd}
     Given two distributions $\mu, \tilde{\mu}$ and a reproducing kernel $\kernel$ \citep[Def.~4.18]{Steinwart2008SupportVM}, the associated kernel \emph{maximum mean discrepancy (MMD)} is the worst-case difference in means for functions in the unit ball $\ball_{\kernel}$ of the associated reproducing kernel Hilbert space:
    \begin{talign}
        \mmd_{\kernel}(\mu, \tilde{\mu}) &\defeq \sup_{f\in\ball_{\kernel}} |\E_{\x\sim \mu} f(\x) - \E_{\x\sim \tilde{\mu}} f(\x)|. 
\end{talign}
When $\mu = \Pin$ and $\tilde{\mu}=\Qout$ as in \cref{def:thinning_algo} and $\mbf K \defeq (\kernel(\x_i,\x_j))_{i, j=1}^n \in \real^{n \times n}$ denotes the induced kernel matrix, then the MMD can be expressed as a Mahalanobis distance between $\pin$ and $\qout$:
\begin{talign}
\mmd_{\kernel}(\Pin, \Qout) 
    &= \sqrt{(\pin-\qout)\tp\mkernel(\pin-\qout)} \\
    &\defeq \mmd_{\mkernel}(\pin, \qout).
    \label{eq:mmd_maha}
\end{talign}
For any indices $\ind\subseteq [n]$, we further define the \emph{kernel max seminorm (KMS)}
\begin{talign}
     \indnorm \defeq \max_{i\in\ind} |\e_i^\top \K (\pin-\qout)|.
    \label{eq:kms}
\end{talign}
\end{definition}

Notably, when the input points lie in $\real^d$ and $\kernel(\x_i,\x_j)$ is the linear kernel $\inner{\x_i}{\x_j}$ (so that $\mkernel = \X\X\tp$), MMD measures the Euclidean discrepancy in datapoint means between the input and output distributions:
\begin{talign}
    \mmd_{\mkernel}(\pin, \qout) = \twonorm{\X\tp\pin -\X\tp\qout}.
\end{talign}

A common strategy for bounding the error of a thinning algorithm is to establish its sub-Gaussianity.
\begin{definition}[\tbf{Sub-Gaussian thinning algorithm}]\label{def:alg-subg}
    We write $\alg \in \ksubge$ and say $\alg$ is \emph{$(\K,\subg,\delta)$-sub-Gaussian}, if $\alg$ is a thinning algorithm, $\K$ is a symmetric positive semidefinite (SPSD) matrix, $\subg >0$, $\delta\in[0, 1)$, and there exists an event $\mc E$ with probability at least $1-\delta/2$ such that, the input and output probability vectors satisfy
    \begin{talign}
    \Esub{\event}[\exp\parenth{\angles{\bu,\K(\pin-\qout)}}] \leq \exp\big(\frac{\subg^2}{2}  \bu^\top \K\bu\big),  \forall \bu \in \Rn.
\end{talign}
\end{definition}
Here, the sub-Gaussian parameter $\subg$ controls the summarization quality of the thinning algorithm, and
we see from \cref{tab:subg_thinning_algorithms} 
that a variety of practical thinning algorithms are $(\K,\subg,\delta)$-sub-Gaussian for varying levels of $\subg$. 
\subsection{Examples of sub-Gaussian thinning algorithms}
Perhaps the simplest sub-Gaussian thinning algorithm is \emph{uniform subsampling}: by \cref{prop:uniformsubg}, selecting $\nout$ points from $\xin$ uniformly at random (without replacement) is $(\K,\subg,0)$-sub-Gaussian with $\subg = {\sqrt{\Kmax}/\sqrt{\nout}}$. 
Unfortunately, uniform subsampling suffers from relatively poor summarization quality.
As we prove in \cref{proof:uniform-subsampling}, its root-mean-squared MMD and KMS are both $\Omega(1/\sqrt{\nout})$ meaning that $\nout=10000$ points are needed to achieve $1\%$ relative error.

\begin{proposition}[\tbf{Quality of uniform subsampling}]\label{prop:uniform-subsampling} 
For any $\ind\subseteq [n]$, 
a uniformly subsampled thinning %
satisfies
\begin{talign}
\E[\mmd^2_{\K}(\pin, \qout)]
    &= 
\frac{1}{\nout}\textfrac{\nin-\nout}{\nin-1}\, C_{\K}
\qtext{and}\\
\E[\indnorm^2]
    &\geq
\frac{1}{\nout}\textfrac{\nin-\nout}{\nin-1}\, \max_{i\in\ind} 
 C_{\K\e_i\e_i^\top\K}
\end{talign}
for any SPSD $\K$ with $C_{\K} \defeq \sum_{i=1}^n\pin[\textup{in},i]\K_{ii} - \pin\tp\K\pin$.
\end{proposition}

Fortunately, uniform subsampling is not the only sub-Gaussian thinning algorithm available. 
For example, the Kernel Halving (\khd) algorithm of \citet{dwivedi2024kernel} provides a substantially smaller  sub-Gaussian parameter, $\nu = O({\sqrt{\log(\nout/\delta)}}{/\nout})$, in $\nin^2$ time by selecting one out of every two points and biasing selection toward the point that yields smaller approximation error. 
Similarly,
the $\khcompress(\delta)$ algorithm of \citet[Ex.~3]{shetty2022distributioncompressionnearlineartime} delivers $\subg=O({\sqrt{\log(\nout)\log(\nout/\delta)}}{/\nout})$ in only $\nout^2\log\nout$ time by halving and concatenating coresets of geometrically increasing size until the target output size is met.   
We derive simplified versions of these algorithms with identical sub-Gaussian constants in \cref{sub:khd,sub:khcompress} and a linear-kernel variant (\khlind) with $\nin d$ runtime in \cref{sub:khlind}.
To round out our set of examples, we show in \cref{proof:gs_thin} that two new thinning algorithms based on the Gram-Schmidt walk of \citet{bansal2018gram} yield even smaller $\subg$ at the cost of increased runtime.  
We call these algorithms Gram-Schmidt Thinning 
(\gsthin) and \gscompress.

\section{Low-rank Sub-Gaussian Thinning}\label{sec:low-rank}
One might hope that the improved sub-Gaussian constants of \cref{tab:subg_thinning_algorithms} would also translate into improved quality metrics.  
Our main result, proved in \cref{proof:subg_low_rank_gen_kernel}, shows that this is indeed the case whenever the inputs are approximately low-rank.

\newcommand{\lowrankerrorbound}{Low-rank sub-Gaussian thinning}%
\begin{theorem}[\tbf{\lowrankerrorbound}]\label{thm:mmd-kernel-compression}
Fix any $\delta'\in(0, 1)$, $r \leq n$, and $\ind\subseteq [n]$. 
If $\alg \in \ksubge$, then the following bounds hold individually with probability at least $1-\delta/2-\delta'$:
\begin{talign}
\mmd^2_{\mkernel}(\pin, \qout)  
    &\leq 
\subg^2 
\brackets{e^2r+e\log(\frac{1}{\delta'})}\\ %
&+ \lambda_{r+1}(\frac{1}{\nout} -\frac{1}{\nin}) 
\qtext{and}
\label{eq:mmd_bound}\\
\indnorm
    &\leq     \subg\inddiam\sqrt{2\log(\frac{2|\ind|}{\delta'})}.
    \label{eq:ind_bound}
\end{talign}
Here, $\lambda_{j}$  denotes the $j$-th largest eigenvalue of $\mkernel$, $\lambda_{n+1}\defeq 0$, and 
$D_{\ind} \defeq \max_{i\in\ind}\sqrt{\K_{ii}}$.

Suppose that, in addition, 
$\xset\subset\reals^d$ and 
$|\K_{il} -\K_{jl}| \leq L_\K \twonorm{\x_i-\x_j}$ for some $L_\K > 0$ and all $i,j\in\ind$ and $l\in\supp{\pin}$.
Then, with probability at least $1-\delta/2-\delta'$,
\begin{talign}
&\indnorm
    \leq 
\subg\inddiam\sqrt{2\log(4/\delta')}(1+\frac{32}{\sqrt{3}})
 \\
    &\ \ \qquad+ 
\subg \inddiam\, 32\sqrt{\frac{2}{3}\,\rank{\X_\ind}\log(\frac{3e^2R_\ind L_{\K}}{\inddiam^2 \wedge (R_\ind L_{\K})})}
\label{eq:ind_bound_lipschitz}
\end{talign}
for $R_\ind \defeq \max_{i\in\ind}\twonorm{\x_i}$
and $\X_\ind \defeq [\x_i]_{i\in\ind}^\top$.
\label{thm:subg_low_rank_gen_kernel}
\end{theorem}

Let us unpack the three components of this result.
First, \cref{thm:subg_low_rank_gen_kernel} provides 
a high-probability
$O(\subg \sqrt{\log( |\ind|)})$  bound \cref{eq:ind_bound} on the KMS for any kernel and any sub-Gaussian thinning algorithm on any space. 
In particular, the non-uniform algorithms of \cref{tab:subg_thinning_algorithms} all enjoy $O(\log(\nout)\sqrt{\log(|\ind|)}/\nout)$ KMS, a significant improvement over the $\Omega(1/\sqrt{\nout})$ KMS of uniform subsampling. 
The bound \cref{eq:ind_bound} follows from the sub-Gaussianity of the thinning algorithm (\cref{def:alg-subg}) and the union bound over $\pm \e_i$ for each $i\in\ind$. 

For datapoints in $\reals^d$, \cref{thm:subg_low_rank_gen_kernel} 
also provides a refined $O(\subg \sqrt{\rank{\X_\ind}\log(R_\ind L_{\K})})$ bound \cref{eq:ind_bound_lipschitz} on KMS.
For bounded data, this  trades an explicit dependence on the number of query points $|\ind|$ for a rank factor that is never larger (and sometimes significantly smaller) than $d$. 
This refinement follows from a more elaborate chaining argument that frames $(\e_i^\top \K(\pin-\qout))_{i\in\mathcal{I}}$ as a sub-Gaussian process (\cref{subgauss_process_tails}) and uses the Lipschitzness of $K$ to control its entropy integral. 
We will make use of these results when approximating dot-product attention in \cref{sec:attn}.

Notably, Thm.~3.1 of \citet{phillips2020near} implies that any thinning algorithm must incur at least $\Omega(\sqrt{d}/\nout)$ KMS error for some dataset in $\reals^d$ and many common kernels. Meanwhile, our \cref{tab:subg_thinning_algorithms} and \cref{thm:mmd-kernel-compression} imply that \gsthin has $\nu = O(1/\nout)$ and hence KMS $O(\sqrt{d}/\nout)$. 
Taken together, these results imply that no algorithm can have sub-Gaussian constant $\nu = o(1/\nout)$ and that \gsthin enjoys minimax rate-optimal KMS and a minimax rate-optimal sub-Gaussian constant.

Finally, \cref{thm:subg_low_rank_gen_kernel} provides an $O(\subg \sqrt{r} + \sqrt{\lam_{r+1}/\nout})$ high-probability bound on kernel MMD, where the approximate rank parameter $r$ can be freely optimized.
We establish this result by projecting $\K^{1/2}(\pin-\qout)$ onto the first $r$ eigenvectors of $\K$, bounding the residual error in terms of $\lambda_{r+1}$, and bounding the projection magnitude with high probability using the sub-Gaussianity of \alg and a union bound over a finite cover of a Euclidean ball in $\R^r$.

When $\K = (\k(\x_i, \x_j))_{i,j=1}^n$ is generated by a finite-rank kernel $\k$, like a linear kernel $\inner{\x_i}{\x_j}$, a polynomial kernel $(1+\inner{\x_i}{\x_j})^p$, or a random Fourier feature kernel \citep{rahimi2007random}, this guarantee becomes $O(\subg)$ and improves upon uniform subsampling whenever $\subg = o(1/\sqrt{\nout})$.
In this case, the non-uniform algorithms of \cref{tab:subg_thinning_algorithms} all enjoy $O(\log(\nout)/\nout)$ MMD, a significant improvement over the $\Omega(1/\sqrt{\nout})$ MMD of uniform subsampling.
We will revisit this finite-rank setting when studying stochastic gradient acceleration strategies in \cref{sec:SGD}.

More generally, \cref{thm:subg_low_rank_gen_kernel} guarantees improved MMD even for full-rank $\K$, provided that the eigenvalues of $\K$ decay sufficiently rapidly.  
For example, optimizing over the approximate rank parameter $r$ yields an $O(\subg \log^{p/2}(\nout) )$ bound under exponential eigenvalue decay $\lam_{r+1} = O(n e^{-c r^{1/p}})$ and an $O(\subg^{\frac{p}{p+1}} (\frac{n}{\nout})^{\frac{1}{2(p+1)}})$ bound under polynomial eigenvalue decay $\lam_{r+1} = O(n/ r^{p})$. 
Fortunately, some of the most commonly-used kernels  generate kernel matrices with rapid eigenvalue decay.

For example, the popular Gaussian kernel on $\Rd$, %
\begin{talign}\label{eq:gaussian_kernel}
\textsc{Gauss}(\eta):\ 
    \kernel(x,y) = \exp(-\eta \statictwonorm{x-y}^2)\stext{for} \eta > 0,
\end{talign}
 generates $\K = (\k(\x_i, \x_j))_{i,j=1}^n$ satisfying 
\begin{talign}\label{eq:gsn_lambda_ball}
    \lambda_{r+1} \leq n e^{-\frac{d}{2e} r^{1/d} \log\parenth{\frac{d r^{1/d}}{4 e^2 \eta R^2}}}\sstext{for} (2e)^d \leq r < n
\end{talign}
 whenever 
 $\Xset\subset\ball^d(R)$ %
 \citep[Thm.~3]{altschuler2019massivelyscalablesinkhorndistances}.
 Combined with \cref{thm:subg_low_rank_gen_kernel}, this fact immediately yields an MMD guarantee for each algorithm in \cref{tab:subg_thinning_algorithms}. We present a representative guarantee for \khd.
\newcommand{\Rin}{R_{\textup{in}}}
\begin{corollary}[\tbf{Gaussian MMD of \kh}]\label{cor:gaussian_mmd}
If $\xin \subset \ball^d (R)$ for $R>0$, then $\khd$ with $\kernel=\textsc{Gauss}(\eta)$, %
and $n=\nin$ delivers
\begin{talign}\label{eq:gaussian_mmd_ball}
&\mmd_{\mkernel}^2(\pin,\qout) 
    \leq \\ 
&\quad O \bigl( \frac{\log({\nout}{/\delta})}{\nout^2} 
\big(\big(\frac{\log(\nout)\vee(R^2\eta)}{d}\big)^{d}
+ \log(\frac{1}{\delta'})\big)\bigr)
\end{talign}
with probability at least $1-\delta/2-\delta'$. 
\end{corollary}

The proof in \cref{proof:gaussian_mmd} provides a fully explicit 
and easily computed bound on the Gaussian MMD. 
Under the same assumptions, the  distinct analysis of 
\citet[Thm.~2, Prop.~3]{dwivedi2021generalized} provides a squared MMD bound of size 
$\Theta\big(\frac{\log(\nout/\delta)}{\nout^2}(\frac{\log^{d+1}(\nout)R^d\eta^{d/2}}{(\log\log(\nout))^d}+ \log(\frac{1}{\delta'}))\big)$. 
Notably, \cref{cor:gaussian_mmd} improves upon this best known \khd guarantee whenever the datapoint radius $R=O(\log \nout)$, a property that holds almost surely for any bounded, sub-Gaussian, or  subexponential data sequence \citep[see][Prop.~2]{dwivedi2024kernel}.

\citet[Thm.~4]{altschuler2019massivelyscalablesinkhorndistances} additionally showed that
Gaussian kernel matrix eigenvalues satisfy
\begin{talign}\label{eq:gsn_lambda_manifold}
    \lambda_{r+1} \leq n e^{- cr^{{2}{/(5d^\star)}}}\qtext{for} 1 \leq r < n
\end{talign}
for a constant $c$ independent of $\Xset$ when $\Xset$ belongs to a smooth compact manifold of dimension $d^\star < d$.
In this case, our low-rank analysis %
yields adaptive MMD guarantees that scale with the potentially much smaller intrinsic dimension $d^\star$.
We use \cref{thm:subg_low_rank_gen_kernel} to prove the first such intrinsic-dimension guarantee for \khd in \cref{proof:gaussian_mmd_manifold}.

\begin{corollary}[\tbf{Intrinsic Gaussian MMD of \kh}]
\label{cor:gaussian_mmd_manifold}
If $\xin$ lies on a smooth manifold $ \Omega \subset \ball^d$ of dimension $d^\star < d$ (\cref{assum:manifold}), then $\khd$ with $\kernel=\textsc{Gauss}(\eta)$, 
and $n=\nin$ delivers
\begin{talign}\label{eq:gaussian_mmd_manifold}
\mmd_{\mkernel}^2(\pin,\qout) 
    \leq 
O\big( \frac{ \log(\frac{\nout}{\delta})}{\nout^2} \big( (\frac{\log(\nout)}{c})^{\frac{5  d^\star}{2}} \!+\log(\frac{1}{\delta'})\big)\big)
\end{talign}
with probability at least $1-
\frac{\delta}{2}-\delta'$ for $c$ independent of $\xin$.
\end{corollary}

In \cref{sec:ctt}, we will use \cref{cor:gaussian_mmd,cor:gaussian_mmd_manifold} to establish new guarantees for distinguishing distributions in near-linear time.

With our core theory in hand, we now turn our attention to a series of impactful applications.

%% file: sections/transformers.tex
\section{Approximating Attention}\label{sec:attn}
We will first use our analysis to accelerate attention approximation in transformers. 
Dot-product attention lies at the heart of the transformer neural network architecture that has revolutionized natural language processing, computer vision, and speech recognition over the last decade \citep{vaswani2017attention,dosovitskiy2021an,dong2018speech-transformer}.
Given a collection of query, key, and value vectors $(\query_i,\key_i,\val_i)_{i=1}^n$ each in $\Rd$, dot-product attention computes the \emph{softmax matrix}
\begin{talign}
\label{def:att}
&\T 
    \defeq 
\attention((\query_i)_{i=1}^n, 
           (\key_j, \val_j)_{j=1}^n) 
    \defeq
\Dinv \A \V  \\
    &\text{for}\  
\A_{ij} 
    \defeq 
\exp(\frac{\inner{\query_i}{\key_j}}{\sqrt{d}}),
\D = \diag(\A\boldone_n),
    \stext{and}
\V_{ij}
    \defeq
\val_{ij}.
\end{talign}
While attention has enjoyed unprecedented success in capturing long-range dependencies amongst datapoints, its computation is expensive, requiring $\Theta(d\,n^2)$ time to construct and multiply the matrix $\A$. 
This quadratic-time bottleneck has inspired a plethora of practical approximate attention mechanisms \citep[e.g.,][]{kitaev2020reformer,choromanski2021rethinking,chen2021scatterbrain}, but, to our knowledge, only two guarantee accurate reconstruction of the softmax matrix $\T$ \citep{zandieh2023kdeformer,han2024hyperattention}.\footnote{A third remarkable work \citep{alman2024fast} establishes upper and lower bounds for attention approximation but without a practical implementation.}  In this section, we design a new fast attention approximation based on sub-Gaussian thinning and derive guarantees that improve upon the prior art.

\subsection{Thinning attention in theory}
\begin{algorithm2e}[htb]
\caption{Thinformer}
\label{algo:thinformer}
\small{
  \KwIn{%
  Queries, keys, and values $(\query_i,\key_i,\val_i)_{i=1}^n$ in $\Rd$, $\nout$
  }
  \BlankLine
// Define key-value attention kernel\\[2pt]
    $\katt((\augkey,\augval),(\augkey',\augval')) \defeq \exp\big(\inner{\augkey}{\augkey'}\big)\inner{\augval}{\augval'}$  
    \BlankLine
    // Thin augmented key-value pairs using $\katt$ \\[2pt]
    $v_{\max} \gets \displaystyle\max_{i\in[n]}\infnorm{\val_{i}}$;\ \  $(\augkey_i,\augval_i)_{i=1}^n \gets ({\key_i}{/d^{\quarter}},(\val_i, v_{\max}))_{i=1}^n$\\[2pt] %
    $\xout \gets \khcompresshalf(\xin=(\augkey_i,\augval_i)_{i=1}^n, \katt, \nout)$ \\

    \BlankLine
    // Return exact attention on selected key-value subset 
    \BlankLine
    \KwRet{$\That \defeq \attention\big((\query_i)_{i=1}^n, \{(\key,\val): (\augkey,\augval)\in\xout\}\big)$}
    \BlankLine   
}
\end{algorithm2e}

\cref{algo:thinformer} summarizes our new \emph{Thinformer} module. 
At its heart is a new key-value attention kernel $\katt$ that mimics the special structure of the softmax matrix $\T$.
\cref{algo:thinformer} uses the attention kernel and a high-quality thinning algorithm, \khcompresshalf, to subselect key-value pairs and then computes exact attention \cref{def:att} for the key-value subset. 
In total, this requires only $O(d\,\nout^2\log\nout)$ time to run \khcompresshalf and $O(d\,n\,\nout)$ time to compute $\attention$ with $n$ queries and $\nout$ key-value pairs. %
In contrast, computing the exact softmax matrix $\T$ with standard matrix multiplication requires $\Theta(d\,n^2)$ time.
Our next result, proved in \cref{proof:att-err}, shows that \cref{algo:thinformer} also admits a strong quality guarantee for approximating  $\T$.
\newcommand{\tablelineskip}{3.5mm}
\newcommand{\tabletopskip}{-4mm}
\begin{table}[tb]
    \caption{\tbf{Practical approximations with  guarantees.} %
    For each approximation $\That\in\reals^{n\times d}$ to the softmax matrix $\T$ \cref{def:att}, 
    we report, up to a constant factor, the best  worst-case error guarantee for  $\maxnorm{\That-\T}$ given $O(d\,n^{1+a})$ running time
    and 
    $\gamma$-bounded \cref{eq:gamma-bounded} queries and keys. %
    Here, the ratio ${\opnorm{\V}}{/\rownorm{\V}}$ lies in $[1,\sqrt{n}]$ and
    $\tau = 0.173+o(1)$.
    \label{tab:att-assumptions-bound}}%
     {\centering
    \resizebox{.48\textwidth}{!}{
    \begin{tabular}{cc}
        \toprule
        \Centerstack{\bf Approximation}
         & \Centerstack{\bf Guarantee}
         \\\midrule\\[\tabletopskip] 
         \Centerstack{\bf Thinformer} 
         & %
         \!$\frac{n^{2\gamma}\sqrt{d\log( \gamma\maxnorm{\V}\log n)}\log n}{n^{a}}\cdot\rownorm{\V}$\!
         \\[\tablelineskip]
         \Centerstack{\textbf{KDEformer}}%
         & %
         $\frac{n^{2\gamma+\frac{\tau}{2} (1 + \frac{\gamma}{2})}}{n^{a/2}}\cdot\specnorm{\V}$ %
         \\[\tablelineskip]
         \Centerstack{\textbf{HyperAttention}}%
         & $\frac{n^{\frac{17\gamma}{3}} (\log n)^{\frac{1}{6}}}{n^{a/6}}\cdot\specnorm{\V}$%
         \\[2mm]
         \bottomrule
    \end{tabular}%
    }}
\end{table}
\begin{table*}[t]
\caption{\tbf{Quality of T2T-ViT attention approximations on ImageNet.} 
We report mean Top-$1$ accuracy $\pm1$ standard deviation across five random seeds and mean forward pass runtime $\pm1$ standard deviation across $50$ batches of $64$ images.}%
    \label{tab:imagenet-acc-time}
    \centering
    \begin{tabular}{cccc}
\toprule
\textbf{Attention Algorithm} & \Centerstack{\bf Top-1 Accuracy (\%)} & \Centerstack{\bf Layer 1 Runtime (ms)} & \Centerstack{\bf Layer 2 Runtime (ms)} \\
\midrule
\Centerstack{\bf Exact} & 82.55 ± 0.00 & 18.48 ± 0.12 & 1.40 ± 0.01 \\[1mm]
\Centerstack{\bf Performer} & 80.56 ± 0.30 & 2.54 ± 0.01 & 0.60 ± 0.01 \\[1mm]
\Centerstack{\bf Reformer} & 81.47 ± 0.06 & 7.84 ± 0.03 & 1.53 ± 0.01 \\[1mm]
\Centerstack{\bf KDEformer} & 82.00 ± 0.07 & 5.39 ± 0.03 & 2.28 ± 0.03 \\[1mm]
\Centerstack{\bf Scatterbrain} & 82.05 ± 0.08 & 6.86 ± 0.02 & 1.55 ± 0.03 \\[1mm]
\Centerstack{\bf Thinformer (Ours)} & 82.18 ± 0.05 & 2.06 ± 0.01 & 0.54 ± 0.00 \\
\bottomrule
\end{tabular}
\end{table*}
\begin{theorem}[\tbf{Quality of Thinformer}]\label{att-err}
With probability at least $\half$, Thinformer (\cref{algo:thinformer}) yields   
\begin{talign}
&\maxnorm{\That - \T} 
    \leq \!
\frac{c\exp(\frac{2R^2}{\sqrt{d}})\rownorm{\V}\sqrt{\log_2(\nout)\log({8\nout \log_2\frac{\nin}{\nout}})}}{\nout}
\end{talign}
for $c\defeq \frac{128}{\sqrt{3}}\sqrt{(d+1)\log(3e^2(\frac{R^2}{\sqrt{d}} + 2)\maxnorm{\V})}
    +
\sqrt{\log(8)}(4+\frac{128}{\sqrt{3}})$
and $R\defeq\max_{i\in[n]}\max(\twonorm{\key_i},\twonorm{\query_i})$.
\end{theorem}
To put this result into context, let us compare with the existing guarantees for practical attention approximation, 
summarized in \cref{tab:att-assumptions-bound}.
Under the $\gamma$-boundedness assumption, 
\begin{talign}\label{eq:gamma-bounded}
\max_{i\in[n]} \max(\twonorm{\key_i}^2,\twonorm{\query_i}^2) \leq \gamma \sqrt{d} \log n, 
\end{talign}
the KDEformer approximation \citep[Cor.~3.6]{zandieh2023kdeformer} with $\tau = 0.173+o(1)$, 
the HyperAttention approximation \citep[Thm.~1]{han2024hyperattention} with no masking, 
and the Thinformer approximation (\cref{att-err}) guarantee the $\maxnorm{\That-\T}$ bounds of \cref{tab:att-assumptions-bound} 
with $O(dn^{1+a})$ runtime and probability at least $\half$.
The Thinformer guarantee exhibits four improvements over its predecessors. 
First, it establishes a significantly faster error decay rate ($n^{-a}$ versus $n^{-a/2}$ or $n^{-a/6}$) for a given subquadratic runtime $n^{1+a}$. 
Second, it reduces the  dependence on the error inflation factor $\gamma$. 
Third, like the HyperAttention guarantee, it eliminates all dependence on the KDEformer penalty parameter $\tau$. 
Finally, it reduces  dependence on the value matrix by a factor of $\frac{\opnorm{\V}}{\rownorm{\V}} \in [1, \sqrt{n}]$.

Put otherwise, with bounded $\rownorm{\V}$, $\Tthin$ can provide consistent (i.e., $\maxnorm{\Tthin-\T}\to0$ as $n\to\infty$) subquadratic estimation  whenever $\gamma$ is bounded away from $1/2$ and guarantee, for example, $O(\frac{1}{\sqrt{n}})$ error in $\Otilde(d n^{\frac{3}{2}+2\gamma})$ time. 
In contrast, the $\Tkde$ and $\Thyp$ bounds require quadratic runtime to guarantee $O(\frac{1}{\sqrt{n}})$ error in the best case (${\opnorm{\V}}=O(1)$) and cannot guarantee consistent subquadratic estimation in the worst case (${\opnorm{\V}}=\Omega(\sqrt{n})$). 
\begin{table*}[htbp]
\centering
\caption{\textbf{Quality of BigGAN attention approximations for image generation.} We report Frechet Inception Distance (FID) with the ImageNet validation set, Inception Scores (IS), and mean forward pass runtime ± 1 standard deviation across 10 batches of 32 images. A lower FID or higher IS indicates better image generation quality.}
\begin{tabular}{cccc}
\toprule
\textbf{Attention Algorithm} & \Centerstack{\bf FID ($\downarrow$)} & \Centerstack{\bf IS ($\uparrow$)} & \Centerstack{\bf Runtime (ms)} \\
\midrule
\Centerstack{\bf Exact} & 32.18 & 58.37 ± 4.21 & 5.79 ± 0.02 
\\[1mm]
\Centerstack{\bf Performer} & 33.58 & 38.07 ± 3.43 & 2.29 ± 0.01 
\\[1mm]
\Centerstack{\bf Reformer} & 72.23 & 19.14 ± 2.09 & 12.14 ± 0.02 
\\[1mm]
\Centerstack{\bf KDEformer} & 30.70 & 56.91 ± 4.16 & 6.74 ± 0.46 
\\[1mm]
\Centerstack{\bf ScatterBrain} & 38.47 & 36.84 ± 2.90 & 3.12 ± 0.02 
\\[1mm]
\Centerstack{\bf Thinformer (Ours)} & 30.54 & 57.12 ± 3.96 & 2.69 ± 0.01 
\\[1mm]
\bottomrule
\end{tabular}
\label{tab:biggan}
\end{table*}

\subsection{Thinning attention in practice}\label{sec:att-experiment}
To gauge the practical effectiveness of \cref{algo:thinformer}, we recreate the benchmark Tokens-To-Token Vision Transformer (T2T-ViT) and BigGAN image generation experiments of \citet{zandieh2023kdeformer}. 
In the T2T-ViT experiment, attention approximations are scored on their ImageNet classification accuracy and computational expense when used as drop-in replacements for the two most expensive attention layers in a pretrained T2T-ViT neural network \citep{yuan2021tokens}. 
In the BigGAN experiment, approximations are scored on their computational expense and two popular measures of image generation quality, the Frechet Inception Distance \citep[FID,][]{heusel2017gans} and Inception Score \citep[IS,][]{salimans2016improved}.  
Using the exact implementations and settings provided by \citet{zandieh2023kdeformer}, we benchmark our PyTorch implementation of Thinformer against exact attention and four leading attention approximations: Performer \citep{choromanski2021rethinking}, Reformer \citep{kitaev2020reformer}, ScatterBrain \citep{chen2021scatterbrain}, and KDEformer.

In \cref{tab:imagenet-acc-time}, we find that Thinformer ($\ossymb=2$) provides the highest Top-$1$ accuracy on the ImageNet 2012 validation set \citep{ILSVRC15}, while running faster than all of the alternatives. 
In \cref{tab:biggan}, Thinformer ($\ossymb=2$) yields better FID and IS than all of the alternatives while running significantly faster than exact, KDEformer, Reformer, and ScatterBrain. Performer runs faster still but at the expense of substantially worse FID and IS.
The final attention call of Thinformer can also be combined with optimized attention implementations like FlashAttention~\citep{dao2022flashattention,dao2024flashattention} to further reduce the time and memory footprint.
We provide PyTorch code replicating this experiment at  \url{https://github.com/microsoft/thinformer} and supplementary experiment details in  \cref{app:attention_details}.

%% file: sections/sgd.tex
\section{Faster SGD Training} %
\label{sec:SGD}
We now turn to a second application, accelerating training through gradient reordering. To train a machine learning model parameterized by $\w\in\Rd$, a standard approach is to minimize the empirical risk $f(\w) \defeq \frac{1}{n}\sum_{i=1}^n f_i(\w)$ using stochastic gradient descent (SGD) updates,%
\begin{talign}
&\w^{k+\frac{i}{n}}
= \w^{k+\frac{i-1}{n}} - \alpha \grad f_{\perm_k(i)}(\w^{k+\frac{i-1}{n}}), \label{eq:sgd}
\end{talign}
for each epoch $k\in[K]$ and datapoint $i\in[n]$.
Here, $\alpha > 0$ is a step size, each $f_i$ is a datapoint-specific loss function, and $\perm_k$ is a permutation of $[n]$ representing the order in which datapoints are processed in the $k$-th epoch.

\begin{algorithm2e}[htb]
\caption{Thinned Reordering}
\label{algo:reordering}
\small{
\KwIn{Stochastic gradients $(\x_i^k \defeq \grad f_{\perm_k(i)}(\w^{k+\frac{i-1}{n}}))_{i=1}^n$, prior ordering $\perm_k$, thinning algorithm \alg}
\BlankLine
// Select half of points using linear kernel \\[2pt]
$\xout^k \gets \alg(\xin=(\x_i^k)_{i=1}^n, \nout=\frac{n}{2}, \k(\x,\y)=\inner{\x}{\y})$ 
    \ \ \\[2pt]
$\lfront \gets [];\quad \lback \gets []$ 
    \qquad\ \ \ // Initialize empty start and end lists \\[2pt] %
\For{$i=1, \ldots, n$}
{$\lfront.\texttt{append}(\perm_k(i))$ if $\x_i^k\in\xout^k$ else  $\lback.\texttt{prepend}(\perm_k(i))$} 
\vspace{2pt}
\KwRet{\textup{$\perm_{k+1} = \concat(\lfront, \lback)$}}
\vspace{2pt}
}
\end{algorithm2e}
\begin{figure*}[tb]
    {\centering
        \includegraphics[width=.5\textwidth]{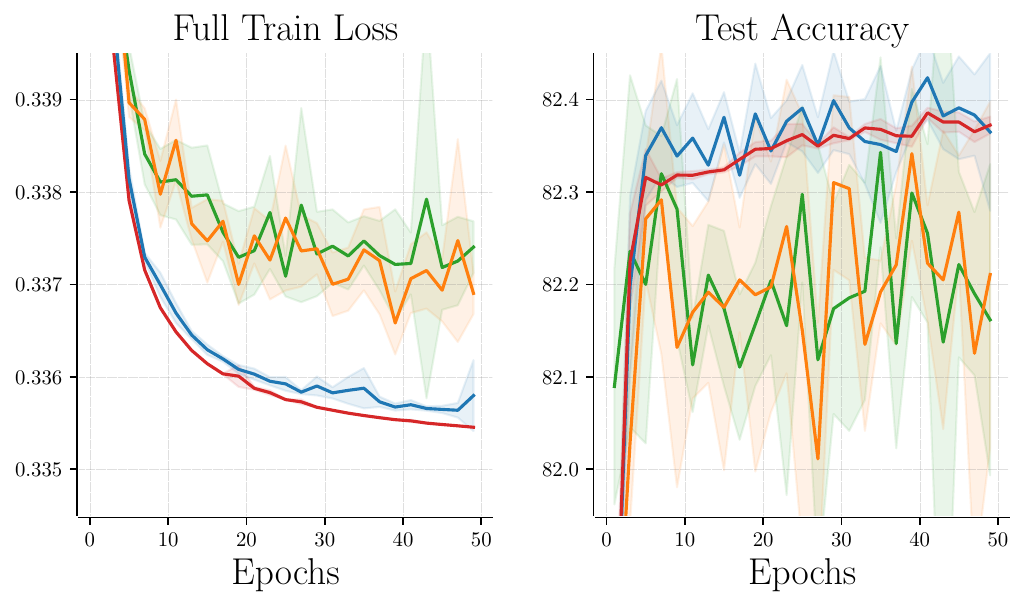}  %
        \includegraphics[width=.5\textwidth]{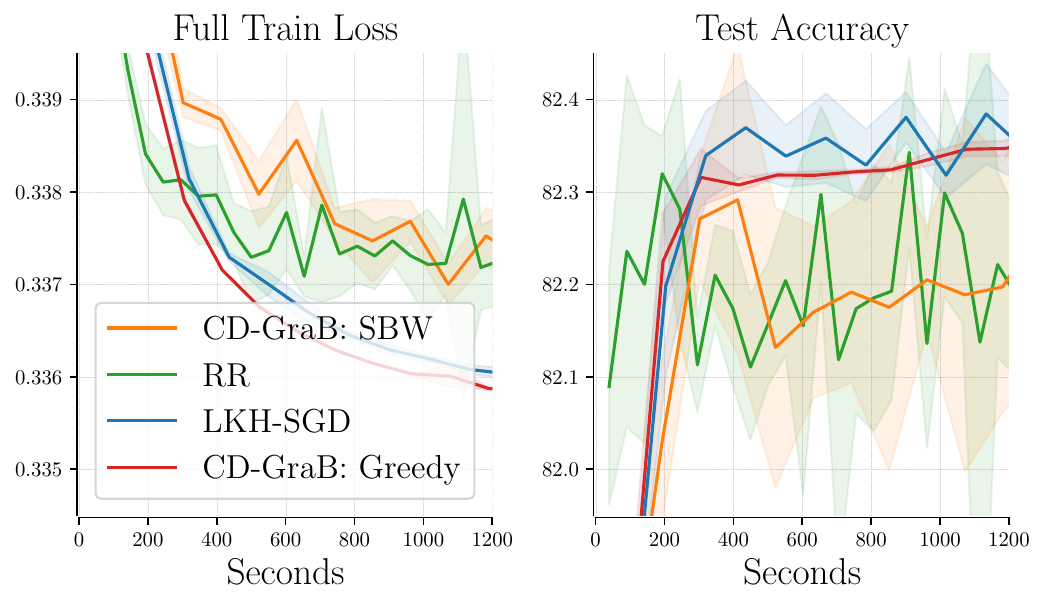}  %
}
    \caption{
    \tbf{Train and test convergence trajectories for mortgage  classification with reordered SGD variants.} We display mean values $\pm 1$ standard deviation across $5$ random seeds. See \cref{sec:theory-practice-gap} for more details.}
    \label{fig:sgd-epoch-time}
\end{figure*}

Typically, one selects the orderings $\perm_k$ uniformly at random, but recent work has demonstrated faster convergence using non-uniform, adaptively selected orderings.
Specifically, \citet{lu2023grab,cooper2023coordinatingdistributedexampleorders} show that any sufficiently accurate thinning algorithm 
can be efficiently transformed into a reordering rule that improves the convergence rate of SGD by a substantial $\Otilde(n^{-1})$ factor.
Their approach, distilled in \cref{algo:reordering}, uses an elegant  construction of \citet[Thm.~10]{harvey2014near} to translate a high-quality thinning of stochastic gradients into a higher-quality reordering. 
However, these prior studies leave two problems unaddressed.

First, while the established convergence rates of \citet{lu2023grab} nearly match the minimax lower bounds for permuted SGD algorithms  \citep[Thm.~4.5]{cha2023tighterlowerboundsshuffling}, a multiplicative gap of size $\Theta(d)$ remains in the worst case.
This led \citet{cha2023tighterlowerboundsshuffling} to declare, ``It is an open problem whether there exists a permutation-based SGD algorithm that gives a dimension-free upper bound while maintaining the same dependency on other factors.'' 

Second, \citet{lu2023grab} carry out their analysis using the self-balancing walk (SBW) thinning algorithm of \citet{alweiss2020discrepancyminimizationselfbalancingwalk} but find its overhead to be too high in practice.  Hence, in all experiments they instead employ a greedy thinning algorithm that often works well in practice but is not covered by their analysis. %
\subsection{Bridging the dimension gap}\label{sec:dimension-gap}
To address the first problem, 
we derive a new guarantee for SGD with \khlin reordering that  replaces the typical $\Theta(d)$ penalty with a soft notion of rank.
\begin{definition}[\tbf{$\eps$-rank}]\label{def:eps-rank}
The \emph{$\eps$-rank}, $\epsrank(\X)$, of a matrix $\X$ is the number of singular values greater than $\eps$.
\end{definition}

\begin{theorem}[\textbf{\khlin-SGD convergence}]\label{thm:convergence}
Suppose that, for all $i\in[n]$ and $\w,\v\in\reals^d$, 
the losses $f$ and $f_i$ satisfy
\begin{talign}
&\twonorm{\grad f_i(\w ) - \grad f(\w)}^2    
    \le 
\sig^2 \ \ \textbf{\textup{(bounded noise)}},
\\
&\twonorm{\grad f_i(\w ) - \grad f_i(\mbi v ) } 
    \leq 
L \twonorm{\w  - \v} \  \ \textbf{\textup{(smoothness)}}, 
    \stext{and}  \\
&f(\w)-\fstar
    \leq 
\frac{1}{2\mu} \twonorm{\grad f(\w)}^2\ \ \textbf{\textup{(PL)}}
\stext{for}
\fstar\defeq\inf_{\v\in\Rd} f(\v).
\end{talign}
Then, with probability at least $\half$, SGD \cref{eq:sgd} with \khsgd reordering (\cref{algo:reordering}) and step size $\alpha$ given in \cref{proof:convergence}
satisfies 
\vspace{-\baselineskip}
\begin{align}
\textstyle f(&\w_K)-\fstar
    \leq 
\textstyle\Otilde(\frac{r}{n^2 K^2})
    \qtext{for}\qquad\qquad\ \ 
\\[5pt]
r 
    \defeq
\textstyle\max_{k\in[K]}&\ \textstyle\epsrank[\eps_k]([\x_1^k,\dots,\x_n^k]),
    \ \ 
\xbar^k \defeq \frac{1}{n}\sum_{i=1}^n\x_i^k, 
    \\
\text{and}\ \ 
\eps_k \defeq \max_{i\in[n]}&\ \textstyle\frac{\sqrt{9 e \log(4Kn \log(e n/2) ) \log(4Kn) }\twonorm{\x_i^k-\xbar^k}}{\sqrt{n}}.
\end{align}
\end{theorem}
The proof of \cref{thm:convergence} in \cref{proof:convergence} simply uses \cref{thm:subg_low_rank_gen_kernel} to bound the thinning quality of \khsgd and then adapts the prior SGD analysis of \citet{cooper2023coordinatingdistributedexampleorders}. 
Notably, the standard practice of \emph{random reshuffling}, i.e., SGD with uniform reordering, can only guarantee a significantly slower $\Omega(\frac{1}{nK^2})$ rate under these assumptions \citep[Thm.~2]{rajput2020closing}, while 
\citet[Thm.~4]{lu2023grab} implies a similar but dimension-dependent  $\Otilde(\frac{d}{n^2 K^2})$   rate for SBW reordering. 
\cref{thm:convergence} matches the minimax lower bound of \citet[Thm.~4.5]{cha2023tighterlowerboundsshuffling} up to the $\epsilon$-rank parameter and shows that  dimension dependence can be avoided when the gradient update matrices 
$[\x_1^k,\dots,\x_n^k]$ 
are low-rank, or, more generally, $\eps=O(\frac{\log(Kn)}{\sqrt{n}})$-approximable by low-rank matrices.

\subsection{Bridging the theory-practice gap}\label{sec:theory-practice-gap}
Two criticisms levied by \citet{lu2023grab} against the SBW algorithm were the need to estimate the maximum Euclidean norm of any possible gradient vector in advance and the need to tune its free hyperparameter. 
\khsgd has neither of these drawbacks as it automatically adapts to the scale of each input and has no hyperparameters to tune.
Moreover, with a linear kernel, \khsgd can be run online in $O(nd)$ time. 
Hence, \khsgd is a promising substitute for the greedy thinning of \citet{lu2023grab,cooper2023coordinatingdistributedexampleorders}. 
Indeed, when we recreate the Home Mortgage Disclosure Act logistic regression experiment of \citet{cooper2023coordinatingdistributedexampleorders} with a single worker (\cref{fig:sgd-epoch-time}), we find that 
\khlin-SGD strongly outperforms the standard practice of random reshuffling (RR) and the theoretically justified but overly conservative CD-GraB: SBW variant.
In addition, \khlin-SGD 
matches the state-of-the-art test accuracy of CD-GraB: Greedy and lags only slightly in terms of training convergence.

The accelerated convergence rate of LKH-SGD over the standard slow SGD rate of RR provides a direct verification of the \cref{thm:convergence} guarantee, and we further verify in \cref{fig:mainfigure} that the singular values of the gradient update matrices drop off steeply, resulting in relatively small $\epsilon_k$-ranks (see \cref{fig:mainfigure}). 
See 
\url{https://github.com/microsoft/khsgd}
for PyTorch code replicating this experiment %
and \cref{app:sgd_details} for supplementary experiment details.

%% file: sections/testing_new.tex
\section{Cheap Two-Sample Testing}\label{sec:ctt}
\newcommand{\mout}{m_{\textup{out}}}
\newcommand{\dembd}{d_{\mrm{embd}}}
\newcommand{\eventbin}[1][b]{\mc E_{#1}}
\newcommand{\distX}{\P}
\newcommand{\distY}{\Q}%
\newcommand{\inflation}{\mbb R}
\newcommand{\distXin}{\distX_{\textup{in}}}
\newcommand{\distYin}{\distY_{\textup{in}}}
\newcommand{\xs}[1][]{\mc{X}_{#1}}
\newcommand{\ys}[1][]{\mc{Y}_{#1}}
\newcommand{\zs}[1][]{\mc{Z}_{#1}}
Our final application is two-sample testing, determining 
whether two datasets are drawn from the same underlying distribution.  We observe independent samples $\xs\defeq(\x_i)_{i=1}^m$ and $\ys\defeq(\y_j)_{j=1}^n$ from the unknown distributions $\distX$ and $\distY$ respectively, and we seek to accept or reject the null hypothesis that $\distX = \distY$.  
Standard kernel MMD tests 
tackle this task 
by computing the empirical  MMD
\begin{talign}
\mmd_{\k}(\distXin, \distYin)
    \stext{for}
\distXin,\distYin
    \defeq 
\frac{1}{m}\sum_{\x\in \xs}\!\dirac_{\x},
\frac{1}{n}\sum_{\y\in \ys}\!\dirac_{\y}
\end{talign}
for an appropriate kernel $\k$ 
and rejecting the null hypothesis whenever $\mmd_{\k}(\distXin, \distYin)$ is sufficiently large \citep{gretton2012kernel}.
Such tests are prized both for their broad applicability and for their high discriminating \emph{power}, that is, their probability of rejecting the null when $\P\neq\Q$. 
A standard way to summarize the power properties of a test is through its \emph{detectable separation rate}.
\begin{definition}[\tbf{Detectable separation rate}]\label{def:separation-rate}
We say a two-sample test has \emph{detectable separation rate} $\eps_{\kernel,m,n}$ if, for any detection probability $1-\beta\in (0,1)$, there exists a constant $c_{\k,\beta}>0$ such that the test has power at least $1-\beta$ of rejecting the null whenever $\mmd_\k(\distX,\distY) \geq c_{\k,\beta} \cdot \eps_{\kernel,m,n}$.
\end{definition}
Standard MMD tests can detect distributional differences on the order of 
$\eps_{\kernel,m,n} = \frac{1}{\sqrt{\min(m,n)}}$ 
\citep[Cor.~9]{gretton2012kernel},  
and this detectable separation rate is known to be the best possible for MMD tests \citep[Prop.~2]{domingoenrich2023compresstestpowerfulkernel} 
and minimax optimal for translation invariant kernels \citep[Thm.~8]{kim2023differentially}.
However, standard MMD tests also suffer from the $\Theta((m+n)^2)$ time burden of computing the empirical MMD.
Recently, \citet{domingoenrich2023compresstestpowerfulkernel} showed that one can improve scalability while preserving power by compressing $\distXin$ and $\distYin$ using a high-quality thinning algorithm.  However, their analysis applies only to a restricted class of distributions and kernels and exhibits a pessimistic dimension dependence on $\reals^d$.  Here, we offer a new analysis of their Compress Then Test approach that applies to any bounded kernel on any domain and, as an application, develop the first non-asymptotic power guarantees for testing with learned deep neural network kernels.

\subsection{Low-rank analysis of Compress Then Test}
\begin{algorithm2e}[h]
\caption{\cttname (\ctt)}%
\label{algo:ctt}
\SetAlgoLined
  \DontPrintSemicolon
\small{
  \KwIn{Samples ($\xs$, $\ys$),  
  \# coresets $\sblock$,
  \osname $\ossymb$,
  kernel $\k$,
  failure probability~$\delta$,  \# replicates $\numperm$, level $\alpha$} 
  \BlankLine
  Partition $\xs$ into $\sblock_m =  \frac{\sblock m}{m+n}$ equal-sized bins 
  $ ( \xs^{(i)} )_{i=1}^{\sblock_m}$ \\
  Partition $\ys$ into $\sblock_n =  \frac{\sblock n}{m+n}$ equal-sized bins 
  $ ( \ys^{(i)} )_{i=1}^{\sblock_n}$ \\
  
  \BlankLine
  // Identify coreset of size
  $\nout=2^\ossymb\sqrt{\frac{m+n}{\sblock}}$
  for each bin\\
  \lFor{$i=1, \dots, \sblock_m$}
  {$\distX^{(i)}_{\tout} \leftarrow \ktcompressd( \xs^{(i)}, \ossymb, \k)$}
  \lFor{$i=1, \dots, \sblock_n$}{$\distY^{(i)}_{\tout} \leftarrow \ktcompressd( \ys^{(i)}, \ossymb, \k)$}

    \BlankLine
    // Compute \tmmd test statistic \\

\makebox[\linewidth]{$M_{\numperm+1} \gets 
\mmd_{\k}(\frac{1}{s_m}\sum_{i=1}^{s_m}\distX^{(i)}_{\tout},\frac{1}{s_n}\sum_{i=1}^{s_n}\distY^{(i)}_{\tout})$ 
\hfill\refstepcounter{equation}\llap{(\theequation)} \label{tmmd}} \\[.2\baselineskip]
    \BlankLine
    // Simulate null by randomly permuting the $\sblock$ coresets $\numperm$ times \\ %
    \For{$b=1,\dots,\numperm$}
    {
    $(\distX^{(i)}_{\tout,b})_{i=1}^{s_m}, (\distY^{(i)}_{\tout,b})_{i=1}^{s_n}\gets \textsc{Permute}((\distX^{(i)}_{\tout})_{i=1}^{s_m}, (\distY^{(i)}_{\tout})_{i=1}^{s_n})$
    
    $M_b \!\gets\! 
    \mmd_{\k}(\frac{1}{s_m}\sum_{i=1}^{s_m}\distX^{(i)}_{\tout,b},\frac{1}{s_n}\sum_{i=1}^{s_n}\distY^{(i)}_{\tout,b})$
    }
    
    \BlankLine
    // Threshold test statistic\\
    $R \gets$ position of $M_{\numperm+1}$ in an increasing ordering of $(M_b)_{b=1}^{\numperm+1}$ with ties broken uniformly at random

    \textbf{return Reject} with prob.\  
        $\min(1, \max(0,R-(1-\alpha)(\numperm+1)))$}%
\end{algorithm2e}

\cref{algo:ctt} details the Compress Then Test (\ctt) approach of \citet[Alg.~1]{domingoenrich2023compresstestpowerfulkernel}. %
Given a coreset count $\sblock\geq 2$, a \osname $\ossymb\ge0$, and a nominal level $\alpha\in(0,1)$, \ctt divides $\xs$ and $\ys$ into datapoint bins of size $\nin\defeq \frac{m+n}{\sblock}$, 
thins each bin down to size $\nout \defeq 2^\ossymb \sqrt{\nin}$ using \ktcompressd (a refinement of \khcompressd detailed in \cref{app:ktcompress}), and uses the thinned coresets to cheaply approximate 
$\mmd_{\k}(\distXin, \distYin)$ and permuted versions thereof. 
\citet[(8)]{domingoenrich2023compresstestpowerfulkernel} showed that the total runtime of \ctt is dominated by 
\begin{talign}
    \bigO{4^{\ossymb} (m+n) (\sblock + \log_4\parenth{\frac{m+n}{\sblock} - \ossymb})}
\end{talign}
kernel evaluations, yielding a near-linear $\bigO{(m+n)\log^c(m+n)}$ time algorithm whenever  $\sblock=\bigO{\log_4(m+n)}$ and $\ossymb \leq c \log_4 \log(m+n)$.
Moreover, Prop.~1 of \citet{domingoenrich2023compresstestpowerfulkernel} ensures that \ctt has probability at most $\alpha$ of falsely rejecting the null hypothesis. 

Our next, complementary result shows that \ctt 
also matches the detectable separation rate of standard MMD tests up to an inflation factor $\error/2^\ossymb$ depending on the \osname $\ossymb$.

\begin{theorem}[\tbf{Low-rank analysis of \ctt power}]\label{thm:ctt_power}
Suppose the parameters of \ctt (\cref{algo:ctt}) satisfy $m\leq n$,
\begin{talign}
\sblock_m \geq \frac{32}{9} \log(\frac{2e}{\gamma}), 
\qtext{and}
\delta= \min(\frac{\wtil \beta}{6}, (\frac{\wtil \beta}{2})^{1/\floor{\alpha(\numperm+1)}} \frac{\alpha}{30 e \sblock})
\end{talign}
for $\wtil \beta \defeq \frac{\beta}{1+\beta/2}$
and 
$\gamma \defeq \frac{\alpha}{4e} (\frac{\Tilde \beta}{4})^{1/\floor{\alpha(\numperm+1)}}$. 
Then \ctt has detectable separation rate (\cref{def:separation-rate}) 
\begin{talign}
\eps_{\kernel,m,n} = (1 +\error/2^\ossymb)/ \sqrt m,
\end{talign}
where $\error^2$ denotes the $(1-\frac{\wtil\beta}{20\sblock_n})$-th quantile of 
\begin{align}\label{eq:error-inflation-factor-simplified}
    &\textstyle\errorhat^2 \defeq \log(\frac{m+n}{s}) \log(\frac{n}{\wtil\beta})\, \cdot \\
    & \min_{r\leq 2\nout}\textstyle \bigl\{ \infnorm{\k} r \log(\frac{n}{\wtil\beta}) + (\lambda_{r+1}(\K)+\lambda_{r+1}(\K')) \nout\bigr\}.
\end{align}
for
$\K\defeq(\k(\x_i,\x_j))_{i,j=1}^m$,  $\K'\defeq(\k(\y_i,\y_j))_{i,j=1}^n$, 
and 
$\infnorm{\k} \defeq \sup_{x,y\in \supp{\distX+\distY}}\abss{\k(x,y)}$. 

\end{theorem}

The proof in \cref{proof:ctt_power} combines the low-rank sub-Gaussian error bounds of \cref{thm:subg_low_rank_gen_kernel} with the generic compressed power analysis of \citet[App.~B.1]{domingoenrich2023compresstestpowerfulkernel} to yield power guarantees for  bounded kernels on any domain.
Notably, when $\rank{\K}$ and $\rank{\K'}$ are bounded or, more generally, $\polylog(n)$ one can choose the compression level $\ossymb=\Theta(\log_4\log(m+n))$ to exactly match the optimal quadratic-time detectable separation rates with a near-linear time \ctt test.
Moreover, the inflation factors remain well-controlled whenever the induced kernel matrices exhibit rapid eigenvalue decay. 

As a concrete example, consider the learned deep neural network kernel of 
\citet{liu2020learning},
\newcommand{\kdeep}{\k_{\textup{deep}}}
\begin{talign}\label{eq:deep_kernel}
    \kdeep(\x,\y) \defeq \brackets{(1-\epsilon) \kappa(\phi(\x),\phi(\y)) + \epsilon} q(\x,\y),
\end{talign}
where $\phi: \reals^d \to \reals^{\dembd}$ is a pretrained neural network, 
$q$ and $\kappa$ are $\textsc{Gauss}(\eta)$ 
kernels \cref{eq:gaussian_kernel} on $\Rd$ and $\reals^{\dembd}$ respectively, 
and $\eps\in(0,1)$.
This deep kernel generates full-rank kernel matrices \citep[Prop.~5]{liu2020learning} but
induces exponential eigenvalue decay due to its decomposition as a mixture of Gaussian kernels. 
Hence, as we show in \cref{proof:ctt_power_deep_kernel}, 
\ctt 
with $\kdeep$,  
$\ossymb=\Theta(\log_4\log(m+n))$, and sub-Gaussian inputs matches the detection quality of a quadratic-time MMD test in near-linear time.

\begin{corollary}[\tbf{Power of deep kernel \ctt}]\label{cor:ctt_power_deep_kernel}
Instantiate the assumptions of \cref{thm:ctt_power} with $\k=\kdeep$ \cref{eq:deep_kernel}.
If the inputs $(\phi(\x_1), \x_1,\phi(\y_1), \y_1)$ are \emph{sub-Gaussian}, that is, 
\begin{talign}\label{eq:subexp-dist}
\E[e^{c\statictwonorm{(\phi(\x_1), \x_1,\phi(\y_1), \y_1)}^2}]<\infty
\end{talign}
for some $c>0$, %
then \ctt satisfies the conclusions of \cref{thm:ctt_power} with $d'\defeq \dembd + d$ and 
\begin{talign}
\error[\kdeep]
=
O(\log^{\frac{d'}{2}+\frac{3}{2}}(\frac{n}{\wtil \beta})).
\end{talign}
\end{corollary}
Moreover, when the input and neural features lie on smooth compact manifolds \citep[as, e.g., in][]{zhu2018ldmnet}, the error inflation of \ctt adapts to the smaller intrinsic manifold dimension, enabling an improved trade-off between runtime and detection power. See \cref{proof:ctt_power_deep_kernel_manifold} for our proof.

\begin{corollary}[\tbf{Power of deep manifold kernel \ctt}]
\label{cor:ctt_power_deep_kernel_manifold}
Under the assumptions of \cref{cor:ctt_power_deep_kernel}, if $\x_1$, $\y_1$, $(\x_1,\phi(\x_1))$,  and $(\y_1,\phi(\y_1))$ belong to smooth compact manifolds (\cref{assum:manifold}) with dimension $d^\star <d'$ then 
\ctt satisfies the conclusions of \cref{thm:ctt_power} with 
\begin{talign}
\error[\kdeep]
=
O(\log^{\frac{5d^\star}{4} + \frac{3}{2}}(\frac{n}{\wtil \beta})).
\end{talign}
\end{corollary}
\cref{cor:ctt_power_deep_kernel,cor:ctt_power_deep_kernel_manifold} follow from explicitly bounding the eigenvalues of the generated deep kernel matrices as in \cref{eq:gsn_lambda_ball,eq:gsn_lambda_manifold}. 
By \citet[Thm.~8]{kim2023differentially}, the separation rate of \cref{thm:ctt_power} is minimax optimal up to the inflation factor $\error/2^\ossymb$ and, hence, those of \cref{cor:ctt_power_deep_kernel,cor:ctt_power_deep_kernel_manifold} are minimax optimal up to log factors.

One could alternatively bound the compression error of \ktcompressd using the covering number approach of \citet[Thm.~2, Prop.~3]{dwivedi2021generalized}. 
In the setting of \cref{cor:ctt_power_deep_kernel}, the argument of \cref{proof:ctt_power_deep_kernel} combined with this distinct analysis would yield an alternative error inflation factor $\terror[\kdeep]/2^{\ossymb}$ with worse dimension dependence,
\begin{talign}
\terror[\kdeep]
=
\Theta(\log^{\frac{3d'}{4}+2}(\frac{n}{\wtil \beta})),
\end{talign}
and without known adaptivity to an intrinsic manifold dimension. 

\subsection{Powerful deep kernel testing in near-linear time} %
\label{sub:ctt_experiments}

\begin{figure}[t!]%
    \begin{center}
        \includegraphics[width=\linewidth]{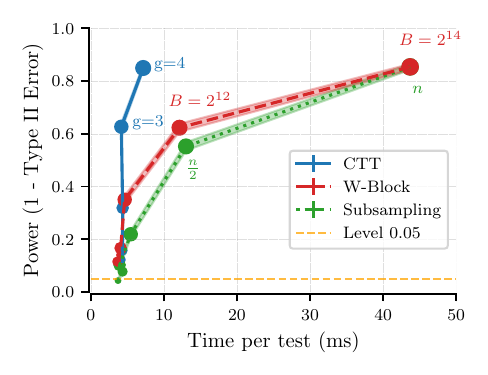}
    \end{center}
    
    \caption{\tbf{Time-power trade-off curves for detecting Higgs bosons with deep kernel MMD tests.} We plot mean values $\pm 1$ standard error across $1000$ independent trials with level $\alpha=0.05$ and  $\numperm=100$ permutations.}
    
    \label{fig:tst-power-runtime}
\end{figure}

To evaluate the practical utility of deep kernel \ctt, we follow the Higgs mixture experiment of \citet[Sec.~5]{domingoenrich2023compresstestpowerfulkernel} and use the deep kernel training procedure of \citet[Tab.~1]{liu2020learning}.
Here, the aim is to distinguish a Higgs boson signal process $\P$ from a background process $\Q$ given $m=n=16384$ observations,  $d=2$ particle-detector features, and a five-layer fully-connected neural network $\phi$ with softplus activations and embedding dimension $\dembd = 20$.

\cref{fig:tst-power-runtime} compares the time-power trade-off curves induced by three fast kernel testing approaches to this problem: \subsampling, a standard wild-bootstrap MMD test \citep{chwialkowski2014wild} that simply evaluates empirical $\mmd_{\kdeep}$ using $\nout=m_{\tout}$ uniformly subsampled points; \textsc{W-Block}, a wild-bootstrap test that averages $\frac{n}{B}$ subsampled squared $\mmd_{\kdeep}$  estimates based on $\nout=m_{\tout} =B$ points \citep{zaremba2013b}; and \ctt with $\sblock=32$ bins and varying $\ossymb$.
We find that the \ctt curve dominates that of the alternative methods and matches the power of an exact MMD test (\subsampling with $\nout=n$) in a fraction of the time. 
The improvements of CTT over the standard power-runtime trade-off of \subsampling provides a direct verification of the \cref{thm:ctt_power} guarantee, and we additionally verify in \cref{fig:Rk-growth} that the empirical inflation factor $\textstyle \errorhat[\kdeep] = O(\log^5(n))$ in this setting due to approximate low-rankness.
See \url{https://github.com/microsoft/deepctt} for PyTorch code replicating this experiment and \cref{app:testing_details} 
for supplementary experiment details.

%% file: sections/impact.tex
\section*{Impact Statement}

This work introduced a new analysis of thinning algorithms that adapts to low-rank structures.
We exploited this adaptivity to design fast algorithms with strong quality guarantees for  three key applications in machine learning: dot-product attention in Transformers, stochastic gradient training in optimization, and deep kernel testing for distinguishing distributions.
More broadly, our techniques provide a general framework for reducing computational resource use in machine learning. Such tools have the potential to reduce energy costs and environmental harms from model training, inference, and evaluation and to improve accessibility in resource-constrained settings, all while provably maintaining high quality. %

%% file: sections/acknowledgments.tex
\paragraph{Acknowledgments}

We thank Insu Han, A. Feder Cooper, and Wentao Guo for their assistance with their code bases and datasets. 

%% file: appendix.tex
\onecolumn
\numberwithin{lemma}{section} 
\numberwithin{proposition}{section} 
\numberwithin{definition}{section} 
\numberwithin{theorem}{section} 
\numberwithin{algorithm}{section} 
\numberwithin{algocf}{section}
\numberwithin{figure}{section} 
\numberwithin{table}{section} 
\numberwithin{corollary}{section} 
\numberwithin{example}{section}
\numberwithin{assumption}{section}

\section*{Appendix Contents}
\etocsettocstyle{}{}
    \etocdepthtag.toc{mtappendix}
    \etocsettagdepth{mtchapter}{none}
    \etocsettagdepth{mtappendix}{section}
    \etocsettagdepth{mtappendix}{subsection}
    {\tableofcontents}

\input{appendices/appendix_notation}

\input{appendices/subg_thinning_algorithms}

\input{appendices/proof_of_subg_low_rank_gen_kernel}
\input{appendices/proof_of_gaussian_kernel}
\input{appendices/proof_of_att_err}
\input{appendices/proof_of_l2_balancing}
\input{appendices/ktcompress}

\input{appendices/proof_of_ctt_power}
\input{appendices/proof_of_ctt_power_deep_kernel}
\input{appendices/additional_experiments}
\input{appendices/experiment_supplement}

%% file: appendices/appendix_notation.tex
\section{Appendix Notation and Definitions}
We often use the shorthand $(a)_+ \defeq \max(a,0)$ as well as the shorthand $\k(\xset,\xset)$ to represent the matrix $(\k(\x_i,\x_j))_{i,j=1}^n$. 
In addition, for each kernel $\k$, we let $\rkhs$ and $\knorm{\cdot}$ represent the associated reproducing kernel Hilbert space (RKHS) and RKHS norm, so that $\ball_{\kernel}=\{ f\in\rkhs : \knorm{f} \leq 1\}$ and define
\begin{talign}
(\Pin - \Qout)\k \defeq \frac{1}{\nin}\sum_{x\in\xin} \k(\x,\cdot) - \frac{1}{\nout}\sum_{x\in\xout} \k(\x,\cdot).
\end{talign}
We also relate our definition of a sub-Gaussian thinning algorithm (\cref{def:alg-subg}) to several useful notions of sub-Gaussianity.
\begin{definition}[\tbf{Sub-Gaussian vector}]\label{def:vector-subg}
We say that a random vector $\diff \in \R^n$ is \emph{$(\K,\subg)$-sub-Gaussian on an event $\event$} if $\K$ is SPSD and $\subg>0$ satisfies 
\begin{talign}\label{eq:vector-subg}
    \Esubarg{\event}{\exp(\bu^\top \K \diff)} \leq \exp(\frac{\subg^2}{2} \cdot \bu^\top \K \bu)
    \qtext{for all}
    \bu \in \reals^n.
\end{talign}
If, in addition, the event has probability $1$, we say that $\w$ is \emph{$(\K,\subg)$-sub-Gaussian}.
\end{definition}
Notably, a thinning algorithm is $(\K,\subg,\delta)$-sub-Gaussian if and only if its associated vector $\pin-\qout$ is $(\K,\subg)$-sub-Gaussian on an event $\event$ of probability at least $1-\delta/2$.

\begin{definition}[\tbf{Sub-Gaussian function}]\label{def:function-subg}
For a kernel $\kernel$, %
we say that a random function $\fsubg\in \rkhs$ is \emph{$(\kernel,\subg)$-sub-Gaussian on an event $\event$} if $\subg > 0$ satisfies
\begin{talign}\label{eq:function-subg}
    \Esubarg{\event}{\exp(\inner{f}{\fsubg}_{\kernel})} \leq \exp(\frac{\subg^2}{2}\cdot \knorm{f}^2)
    \qtext{for all}
    f \in\rkhs.
\end{talign}
If, in addition, the event has probability $1$, we say that $\fsubg$ is \emph{$(\kernel,\subg)$-sub-Gaussian}.
\end{definition}
Our next two lemmas show that for finitely-supported signed measures like $\Pin-\Qout$, this notion of functional sub-Gaussianity is equivalent to the prior notion of vector sub-Gaussianity, allowing us to use the two notions interchangeably. 
Hereafter, we say that $\k$ generates a SPSD matrix $\K$ if $\k(\xset,\xset) = \K$. 

\begin{lemma}[\tbf{Functional sub-Gaussianity implies vector sub-Gaussianity}]
\label{lem:funct_subg_vector_subg}
In the notation of \cref{def:alg-subg}, if $(\Pin - \Qout)\kernel$ is $(\kernel,\subg)$-sub-Gaussian on an event $\event$ and $\kernel$ generates $\K$, then the vector $\pin - \qout$ is $(\K,\subg)$-sub-Gaussian on $\event$.
\end{lemma}
\begin{proof}
Suppose $(\Pin - \Qout)\kernel$ is $(\kernel,\subg)$-sub-Gaussian on an event $\event$, fix a vector $\bu\in \reals^n$, and define the function
\begin{talign}
    f_{\bu} \defeq \sumn u_i \kernel(\cdot, x_i) \in \rkhs.
\end{talign}
By the reproducing property, 
\begin{talign}\label{eq:hnorm-of-fu}
    \bu^\top \K (\pin -\qout) = \inner{f_{\bu}}{(\Pin-\Qout)\kernel}_{\kernel} \qtext{and} \knorm{f_{\bu}}^2 = \bu^\top \K \bu.
\end{talign}
Invoking the representations \cref{eq:hnorm-of-fu} and the functional sub-Gaussianity condition \cref{eq:function-subg} we therefore obtain
\begin{talign}
    \Esubarg{\event}{\exp(\bu^\top \K(\pin-\qout)} &= \Esubarg{\event}{\exp(\inner{f_{\bu}}{(\Pin-\Qout)\kernel}_{\kernel})} 
    \leq \exp(\knorm{f_{\bu}}^2 \cdot \frac{\subg^2}{2}) 
    = \exp(\bu^\top \K \bu \cdot \frac{\subg^2}{2}),
\end{talign}
so that $\pin-\qout$ is $(\K,\subg)$-sub-Gaussian on the event $\event$ as claimed.
\end{proof}

\begin{lemma}[\tbf{Vector sub-Gaussianity implies functional sub-Gaussianity}]
\label{lem:vector_subg_funct_subg}
In the notation of \cref{def:alg-subg}, if $\pin - \qout$ is $(\K,\subg)$-sub-Gaussian on an event $\event$ and $\kernel$ generates $\K$, then $(\Pin - \Qout)\kernel$ is $(\kernel,\subg)$-sub-Gaussian on $\event$.
\end{lemma}
\begin{proof}
Suppose $\pin-\qout$ is $(\K,\subg)$-sub-Gaussian on an event $\event$, fix a function $f\in \rkhs$, and consider the set 
\begin{talign}
\Lset \defeq \braces{f_{\bu} \defeq \sum_{i=1}^n u_i \kernel(\cdot,x_i) : \bu \in \reals^n}.
\end{talign} 
Since $\Lset$ is a closed linear subspace of $\rkhs$, we can decompose $f$ as $f  = f_{\bu} + f_\perp$,
where $\bu\in\Rn$ and $f_\perp$ is orthogonal to $\Lset$ \citep[Theorem 12.4]{rudin1991functional},
so that 
\begin{talign}\label{eq:knorm-decomposition}
    \knorm{f}^2 = \knorm{f_{\bu}}^2 + \knorm{f_\perp}^2\qtext{and} \knorm{f_{\bu}}^2 = \bu^\top \K \bu.
\end{talign}
Invoking the orthogonality of $f_\perp$ and $(\Pin - \Qout)\kernel\in \Lset$, the reproducing property representations \cref{eq:hnorm-of-fu}, and the vector sub-Gaussianity condition \cref{eq:vector-subg}, we find that
\begin{talign}
    \Esubarg{\event}{\exp(\inner{f}{(\Pin-\Qout)\kernel}_{\kernel})} 
    &= \Esubarg{\event}{\exp(\inner{f_{\bu} + f_\perp}{(\Pin - \Qout) \kernel}_{\kernel})} 
    = \Esubarg{\event}{\exp(\bu^\top \K (\pin - \qout)})\\
    &\leq \exp(\bu^\top \K \bu \cdot \frac{\subg^2}{2}) 
    \sless{\cref{{eq:knorm-decomposition}}} \exp(\knorm{f}^2 \cdot \frac{\subg^2}{2}),
\end{talign}
so that $(\Pin-\Qout)\kernel$ is $(\kernel,\subg)$-sub-Gaussian on the event $\event$ as claimed.
\end{proof}

We end our discussion about the versions of sub-Gaussianity considered above by presenting the standard fact about the additivity of sub-Gaussianity parameters under summation of independent sub-Gaussian random vectors, adapted to our setting.

\begin{lemma}[\tbf{Vector sub-Gaussian additivity}]\label{lem:K_sub_gsn_additivity}
    Suppose that, for each $j\in [m]$, 
    $\Delta_j\in\reals^n$ is $(\mbf K,\subg_j)$ on an event $\event[j]$ given $\Delta_{1:(j-1)}\defeq (\Delta_1,\ldots,\Delta_{j-1})$ and $\event[\leq j-1]\defeq \bigcap_{i=1}^{j-1}\event[i]$. 
    Then $\sum_{j=1}^m \Delta_j$ is $(\mbf K, (\sum_{j=1}^m \subg_j^2)^{1/2})$-sub-Gaussian on $\event[\leq m]$.
    \end{lemma}
    \begin{proof}
    Let $\event[\leq s] = \bigcap_{j=1}^s\event[j]$ for each $s\in [m]$.
    We prove the result for $\mc Z_s = \sum_{i=1}^s \Delta_j$ by induction on $s\in [m]$. 
    The result holds for the base case of $s=1$ by assumption. For the inductive case, suppose the result holds for $s\in [m-1]$. Fixing $\bu\in \R^n$, we may apply the tower property, our conditional sub-Gaussianity assumption, and our inductive hypothesis in turn to conclude
    \begin{talign}
        \Earg{\exp(\inner{\bu}{\K \sum_{j=1}^{s+1} \Delta_j})\indic{\event[\leq s+1]}} &= \Earg{\exp(\inner{\bu}{\K \sum_{j=1}^{s} \Delta_j})\indic{\event[\leq s]} \Earg{\exp(\inner{\bu}{\Delta_{s+1}})\indic{\event[s+1]} \mid \Delta_{1:s},\event[\leq s]} } \\
        &\leq \Earg{\exp(\inner{\bu}{\K \sum_{j=1}^{s} \Delta_j})\indic{\event[\leq s]}} \exp\parenth{\frac{\subg_{s+1}^2}{2}\cdot \bu^\top \K \bu}
        \leq \exp\big( \frac{\sum_{j=1}^{s+1} \subg_j^2}{2} \cdot \bu^\top \K \bu\big).
    \end{talign}
    Hence, $\mc Z_{s+1}$ is $(\K,(\sum_{j=1}^{s+1} \subg_j^2)^{1/2})$-sub-Gaussian on $\event[\leq s+1]$, and the proof is complete.
    \end{proof}

%% file: appendices/subg_thinning_algorithms.tex
\section{Proof of \cref{tab:subg_thinning_algorithms}: Sub-Gaussian Thinning Examples}\label{app:subg_thinning_algorithms}
\newcommand{\ininfnorm}[1]{\norm{#1}_{\infty,\mrm{in}}}

This section provides supplementary details for each of the sub-Gaussian thinning algorithms of \cref{tab:subg_thinning_algorithms}.
\subsection{\textsc{Subsampling}}
\label{sub:subsampling}
\subsubsection{\pcref{prop:uniform-subsampling}}
\label{proof:uniform-subsampling}
We begin by computing the first and second moments of $\qout$:
$\E[\qout] = \pin$ and
\begin{talign}
\E[\qout\qout^\top] 
    =
\frac{1}{\nout}\diag(\pin) + \frac{\nin(\nout-1)}{\nout(\nin-1)}(\pin\pin^\top - \frac{1}{\nin}\diag(\pin))
    =
\frac{1}{\nout}(\frac{\nin-\nout}{\nin-1})\diag(\pin)
    +
\frac{\nin(\nout-1)}{\nout(\nin-1)} \pin\pin^\top.
\end{talign}
Hence,
\begin{talign}
\E[\mmd^2_{\K}(\pin, \qout)]
    &= 
\pin\tp\K\pin
    - 
2 \pin\tp \K \E[\qout]
    + 
\E[\qout\tp\K\qout]
    =
\tr(\K\E[\qout\qout^\top])
    - 
\pin\tp \K \pin \\
    &=
\frac{1}{\nout}(\frac{\nin-\nout}{\nin-1})(\tr(\K \diag(\pin)) -\pin\tp\K\pin)
    =
\frac{1}{\nout}(\frac{\nin-\nout}{\nin-1}) C_{\K}.
\label{eq:uniform-mmd}
\end{talign}

To derive the second advertised result, we note that
\begin{talign}
\E[\indnorm^2]
    &\geq
\max_{i\in\ind}
\E[(\e_i^\top\K(\pin-\qout))^2]
    =
\max_{i\in\ind} \E[\mmd^2_{\K\e_i\e_i^\top\K}(\pin,\qout)]
\end{talign}
and invoke the initial result \cref{eq:uniform-mmd} to conclude.

\subsubsection{Sub-Gaussianity of subsampling}

\begin{proposition}[\tbf{Sub-Gaussianity of uniform subsampling}]\label{prop:uniformsubg} 
For any SPSD $\K \in \reals^{n\times n}$, uniform subsampling (without replacement) is a $(\K,\subg,0)$-sub-Gaussian thinning algorithm with 
\begin{talign}\label{eq:uniformsubg}
    \subg \defeq \frac{\sqrt{\Kmax}}{\sqrt{\nout}}. %
\end{talign}
\end{proposition}

\begin{proof}
\newcommand{\xstar}{x^\star}
Fix any vector $\bu\in\reals^n$, and  
let $J_1, \dots, J_{\nout}$ 
be the random indices in $[n]$ selected by uniform subsampling.
Since $\bu\tp\K(\pin - \qout) = \frac{1}{\nout}\sum_{i=1}^{\nout} \bu\tp\K(\pin - \e_{J_i})$ is an average of mean-centered scalars drawn without replacement and satisfying
\begin{talign}
|\bu\tp\K\e_{J_i}|
\leq \sqrt{\bu\tp\K\bu} \sqrt{\e_{J_i}\tp\K\e_{J_i}}
\leq \sqrt{\Kmax} \sqrt{\bu\tp\K\bu}
\qtext{with probability $1$}
\end{talign}
by Cauchy-Schwarz, Thm.~4 and equations (1.8) and (4.16) of \citet{Hoeffding1963} imply that 
\begin{talign}
\E[\exp(\bu\tp\K(\pin - \qout))]
    \leq
\exp(\frac{\Kmax}{2\nout} \bu\tp\K\bu).
\end{talign}
\end{proof}

\input{appendices/rkhd}

\input{appendices/khcompressd}

\subsection{\gsthin}
\label{sub:gs_thin}
The section introduces and analyzes the Gram-Schmidt Thinning algorithm (\gsthin, \cref{algo:gs_thin}). 
\gsthin repeatedly divides an input sequence in half using, \gshalve (\cref{algo:gs_halve}), a symmetrized and kernelized version of the Gram-Schmidt (GS) Walk of \citet{bansal2018gram}.
We will present two different implementations of \gshalve: a quartic-time implementation (\cref{algo:gs_halve}) based on the GS Walk description of \citet{bansal2018gram} and a cubic-time implementation based on local updates to the matrix inverse (\cref{algo:gs_halve_cubic}).
While both the algorithms lead to the same output given the same source of randomness, we present the original implementation\footnote{\label{footnote:alg_equiv} Towards making this equivalence clear, \cref{algo:gs_halve} has been expressed with the same variables that \cref{algo:gs_halve_cubic} uses. \cref{algo:gs_halve} can be slightly simplified if it were to be considered independently.} for conceptual clarity and the optimized implementation for improved runtime.
Throughout, for a matrix $\mbf Q$ and vector $\bu$, we use the notation $\mbf Q_{\ind\times\mc{J}}$ 
and $\bu_{\ind}$
to represent the submatrix $(\mbf Q_{ij})_{i\in\ind,j\in\mc{J}}$ and subvector $(\bu_{i})_{i\in\ind}$.

\begin{algorithm2e}[]
\caption{\gsthin: Gram-Schmidt Thinning}%
\label{algo:gs_thin}
\small{
    \KwIn{point sequence  $\xin=(\x_i)_{i = 1}^{\nin}$, kernel $\kernel$, output size $\nout \in \nin / 2^\naturals$, $\halve\in \braces{\gshalve, \gshalvecubic}$}
  \BlankLine
  // Repeatedly divide coreset size in half \\
  $m \gets \log_2(\nin/\nout)$ \\ 
  \lFor{$\ell=1, 2, \ldots, m$}
    {%
    $\xin \gets \halve(\xin, \kernel)$
  }
  \KwRet{\textup{$\xout\defeq\xin$, coreset of size} $\nout=\nin/2^m$}{} 
  } 
\end{algorithm2e}

\SetKwFunction{proctwo}{\texttt{kernel\_gs\_walk}}

\begin{algorithm2e}[]
\caption{\gshalve: Gram-Schmidt Halving}%
\label{algo:gs_halve}
\small{
    \KwIn{point sequence $\xin = \parenth{x_i}_{i=1}^{\nin}$ with even $\nin$, kernel $\kernel$} 
    \BlankLine

    $\xout\gets \braces{}$ %
    \quad // Initialize empty coreset \\[1mm]
    // Select one point to keep from each consecutive pair using  kernelized GS Walk  \\
    $\bz \gets $ \proctwo{$\xin$} \\
    \For{$i=1,\dots,\nin/2$}
    {
        \eIf{$\bz_i = 1$}
        {
            $\xout.\texttt{append}(x_{2i-1})$
        }
        {
            $\xout.\texttt{append}(x_{2i})$
        }
    }
    \KwRet{$\xout$\textup{, coreset of size $\nin/2$}}
    
    \hrulefill\\
    \SetKwProg{myproc}{function}{}{}
     \myproc{\proctwo{$(\x_i)_{i=1}^{\nin}$}:}{
        $t \gets 1$; \quad $\bz_t \gets (0,0,\ldots,0) \in \R^{\nin/2} $ \quad\quad\quad\quad// Initialize fractional assignment vector \\
            $\mc A \gets [\nin/2] $  \quad\quad\quad\quad // Initialize  set of active coordinates \\
    $p \sim \mc A $ \quad\quad\quad\quad // Select a pivot uniformly at random \\

    \While{$\bz_t \notin \braces{\pm 1}^{\nin/2}$}
    {   
        $\mc A'  \gets \mc A  \,\backslash\, \big\{ \min \big( \braces{i\in [\nin/2]: \abss{\bz_{ti}} = 1} \,\backslash\, ( [\nin/2] \,\backslash\, \mc  A ) \big)  \big\}  $ \\   // Update set of active coordinates by removing smallest index set to $\pm 1 $\\
        \eIf{$p\notin \mc A'$}
        {
            $p' \sim \Unif(\mc A')$ \quad // Select a new pivot from $\mc A'$ uniformly at random
        }{$p' \gets p$}
        // Compute step direction in which to update fractional assignment vector \\
        $\bu_{t} \gets \argmin_{\bu \in \R^{\nin/2}} \bu^\top \mbf Q \bu$ subject to $\bu_{p'}=1$ and $\bu_i=0$ for all $i\notin \mc A'$, \\ 
        \quad where $\mbf Q\in \R^{(\nin/2)\times (\nin/2)}$ has entries $\mbf Q_{ij} \defeq \kernel(x_{2i-1},x_{2j-1}) + \kernel(x_{2i},x_{2j} ) - \kernel(x_{2i-1},x_{2j}) - \kernel(x_{2i},x_{2j-1}) $ \\ 
        $\delta^{+} \gets \abss{\max \Delta}$ and $\delta^- \gets \abss{\min \Delta}$, where $\Delta = \braces{\delta\in \R: \bz_{t} + \delta\bu_{t} \in [-1,+1]^{\nin/2}}$ \qquad// Select candidate step sizes\\
        $\delta_t \gets \delta^+$ with probability $\delta^-/(\delta^+ + \delta^-)$; otherwise $\delta_t \gets -\delta^-$ \qquad// Choose step size and sign at random \\
        $\bz_{t+1} \gets \bz_{t} + \delta_t \bu_{t}$ \qquad // Update fractional assignments \\
        $t\gets t+1$; \quad
        $\mc A \gets \mc A'$;\quad
        $ p \gets p'$  \\  
    }
     }
     \KwRet{$\bz_t$\textup{, sign vector in $\{\pm1\}^{\nin/2}$}}
} 
\end{algorithm2e}

\SetKwFunction{proccubic}{\texttt{kernel\_gs\_walk\_cubic}}

\begin{algorithm2e}[]
\caption{\gshalvecubic: Gram-Schmidt Halving with cubic runtime}
\label{algo:gs_halve_cubic}
\small{
    \KwIn{point sequence $\xin = \parenth{x_i}_{i=1}^{\nin}$ with even $\nin$, kernel $\kernel$ with positive definite  $\k(\xin,\xin)$} 
    \BlankLine
    $\xout\gets \braces{}$ %
    \quad // Initialize empty coreset \\[1mm]
    // Select one point to keep from each consecutive pair using  kernelized GS Walk  \\
    $\bz \gets $ \proccubic{$\xin$} \\
    \For{$i=1,\dots,\nin/2$}
    {
        \eIf{$\bz_i = 1$}
        {
            $\xout.\texttt{append}(x_{2i-1})$
        }
        {
            $\xout.\texttt{append}(x_{2i})$
        }
    }
    \KwRet{$\xout$\textup{, coreset of size $\nin/2$}}

    \hrulefill\\
    \SetKwProg{myproc}{function}{}{}
     \myproc{\proccubic{$\parenth{x_i}_{i=1}^{\nin}$}:}{
        $t \gets 1$; \quad $\bz_t \gets (0,0,\ldots,0) \in \R^{\nin/2} $ \quad\quad\quad\quad// Initialize fractional assignment vector \\
       $\mc A \gets [\nin/2] $  \quad\quad\quad\quad // Initialize  set of active coordinates \\
    $p \sim \mc A $ \quad\quad\quad\quad // Select pivot uniformly at random \\

    $\mbf Q \gets (\kernel(x_{2i-1},x_{2j-1}) + \kernel(x_{2i},x_{2j} ) - \kernel(x_{2i-1},x_{2j}) - \kernel(x_{2i},x_{2j-1}))_{i,j=1}^{\nin/2}$ %
    \quad // Form paired difference kernel matrix\\
    $  \mbf C \gets (\mbf Q_{ \mc A \backslash \{p\} \times \mc A \backslash \{p\}   } )^{-1}$ \\
    \While{$\bz_t \notin \braces{\pm 1}^{\nin/2}$}
    {
        $\mc A'  \gets \mc A  \,\backslash\, \big\{ \min \big( \braces{i\in [\nin/2]: \abss{\bz_{ti}} = 1} \,\backslash\, ( [\nin/2] \,\backslash\, \mc  A ) \big)  \big\}  $ \\   // Update set of active coordinates by removing smallest index set to $\pm 1 $\\
        \eIf{$p\notin \mc A'$}
        {
            $p' \sim \Unif(\mc A')$ \quad // Select a new pivot from $\mc A'$ uniformly at random
        }{$p' \gets p$}
         $\mc A_1 \gets \mc A \,\backslash\, \{p\}  $ \\ $\mc A_2 \gets \mc A' \,\backslash\, \{ p' \} $. \\
        $i \gets \mc A_1 \backslash \mc A_2  $  // Choose $i$ as the (unique) index that was removed from the active coordinates \\[1mm] 
         // Compute $(\mbf{Q}_{\mc A_2   \times \mc A_2})^{-1} $ using block matrix inversion and the Sherman-Morrison formula  \\[1mm]   
        $\mbf D \gets \mbf C_{ \mc A_2 \times A_2  } $ \\[1mm] %
        $ \mbf C \gets \mbf D - \frac{\mbf D \mbf Q_{ \mc A_2 \times \{ i\} } \mbf Q_{\{i\} \times \mc A_2 } \mbf D}{ \mbf Q_{ ii } + \mbf Q_{\{i\} \times \mc A_2 } \mbf D \mbf Q_{ \mc A_2 \times \{ i\} }}$  \quad\quad\quad\quad \\[1mm] %
        // Compute step direction in which to update fractional assignment vector \\
       Compute $\bu_t $ as $  (\bu_t)_{\mc A_2 }  = - \mbf C  \mbf Q_{ \mc A_2 \times \{p'\}  }  $ , $\bu_{tp'} = 1$, and $\bu_{ti} = 0 $ for $i \notin \mc A'$    \\[1mm] 

        $\delta^{+} \gets \abss{\max \Delta}$ and $\delta^- \gets \abss{\min \Delta}$, where $\Delta = \braces{\delta\in \R: \bz_{t} + \delta\bu_{t} \in [-1,+1]^{\nin/2}}$ \quad// Select candidate step sizes \\ 
        $\delta_t \gets \delta^+$ with probability $\delta^-/(\delta^+ + \delta^-)$; otherwise $\delta_t \gets -\delta^-$ \quad// Choose step size and sign at random \\
        $\bz_{t+1} \gets \bz_{t} + \delta_t \bu_{t}$ \quad // Update fractional assignments \\
        $t\gets t+1$; \quad
        $\mc A \gets \mc A'$; \quad  
        $ p \gets p'$  \\  
    }
     }
       \KwRet{$\bz_t$\textup{, sign vector in $\{\pm1\}^{\nin/2}$}}
} 
\end{algorithm2e}

Our first result, proved in \cref{proof:gs_thin}, shows that \gsthin is a sub-Gaussian thinning algorithm.
\begin{proposition}[\tbf{\gsthin sub-Gaussianity}]\label{prop:gs_thin}
For $\K$ generated by $\k$, 
\gsthin (\cref{algo:gs_thin}) is a $(\mkernel,\subg,0)$-sub-Gaussian thinning algorithm with parameter
\begin{talign}\label{eq:gs_thin_parameter}
    \subg \defeq \frac{2}{\sqrt 3} \frac{\sqrt{\maxnorm{\K}}}{\nout}.
\end{talign}
\end{proposition}
Our second result, proved in \cref{proof:quartic-GS}, shows that \gsthin with the \gshalve implementation has $O(\nin^4)$ runtime. 

\begin{proposition}[\tbf{Runtime of \gsthin with \gshalve}]
\label{prop:gs_thin_runtime_1}
The runtime of \gsthin with implementation \gshalve (\cref{algo:gs_halve}) is $\bigO{\nin^4}$. 
\end{proposition}

Our third result, proved in \cref{proof:gs_halve_agreement}, establishes the equivalence between \gshalve and \gshalvecubic.
More precisely, we show that the sequence of partial assignment vectors generated by \proctwo{$\cdot$} of \cref{algo:gs_halve} and \proccubic{$\cdot$} of \cref{algo:gs_halve_cubic} are identical given identical inputs, an invertible induced kernel matrix, and an identical source of randomness.

    \begin{proposition}[\textbf{Agreement of \gshalve and \gshalvecubic}]\label{lem:gs_halve_agreement}
        Let $\bz_1,\bz_2,\dots$ be the fractional assignment sequence generated by \proctwo{$(\x_i)_{i=1}^{\nin}$} in \cref{algo:gs_halve} and $\bz_1',\bz_2',\dots$ be the fractional assignment sequence generated by \proccubic{$(\x_i)_{i=1}^{\nin}$} in \cref{algo:gs_halve_cubic} with an identical source of randomness. 
        If the pairwise difference matrix
        \begin{talign}
        \mbf Q \defeq (\kernel(x_{2i-1},x_{2j-1}) + \kernel(x_{2i},x_{2j}) - \kernel(x_{2i-1},x_{2j}) - \kernel(x_{2i},x_{2j-1}))_{i,j\in [\nin/2]} 
    \end{talign}
    is positive definite, then $\bz_t = \bz_t'$ for all $t$.
    \end{proposition}

Our fourth result, proved in \cref{proof:cubic-GS}, shows that \gsthin with the \gshalvecubic implementation has $O(\nin^3)$ runtime.

\begin{proposition}[\tbf{Runtime of \gsthin with \gshalvecubic}] \label{prop:gs_thin_runtime_2}
    The runtime of \gsthin with implementation \gshalvecubic (\cref{algo:gs_halve_cubic}) is $\bigO{\nin^3}$.
\end{proposition}

\subsubsection{\pcref{prop:gs_thin}}
\label{proof:gs_thin}

Our first lemma bounds the sub-Gaussian constant of \gshalve (\cref{algo:gs_halve}).

\begin{lemma}[\tbf{\gshalve sub-Gaussianity}]\label{lem:gs_halve}
In the notation of \cref{def:thinning_algo}, consider the input and output vectors $\pin,\qout\in\reals^n$ of \gshalve (\cref{algo:gs_halve}) for $\xset\supseteq \xin$ with $|\xset| = n \geq \nin$.
If $\K = \k(\xset,\xset)$, 
then $\pin - \qout$ is $(\mbf K, \subg)$-sub-Gaussian with
\begin{talign}\label{eq:gs_halve_nu}
    \subg \defeq \frac{2\maxnorm{\K}^{1/2}}{\nin} = \frac{\maxnorm{\K}^{1/2}}{\nout}.
\end{talign}
\end{lemma}
\begin{proof}
Since $\K$ is SPSD, there exists a matrix $\bPhi\in\reals^{n\times d}$ such that $\K = \bPhi \bPhi^\top$. Let $\mbf B\in \R^{d \times (\nin/2)}$ be the matrix with entries
\begin{talign}
    \mbf B_{j,i} \defeq \bPhi_{2i-1,j} - \bPhi_{2i,j} \qtext{for}~i\in[\nin/2]
    \qtext{and} j \in [d].
\end{talign}
Note that, for each $i\in[\nin/2]$, 
\begin{talign} \label{eq:Bi_two_norm}
    \sum_{j\in[d]}\mbf B_{j,i}^2 = \mbf{K}_{2i-1,2i-1} + \mbf{K}_{2i,2i} - \mbf{K}_{2i-1,2i} - \mbf{K}_{2i,2i-1} \leq 4\maxnorm{\mbf K}.
\end{talign}
Hence, by \citet[Thm.~6.6]{harshaw2024balancing}, $\frac{1}{\nin}\mbf B \z$ is $(\ident,\subg)$-sub-Gaussian where $\ident$ is the identity matrix in $\reals^{d\times d}$.
Now fix any $\bu \in \Rd$. 
Since $\frac{1}{\nin}\mbf B \z = -\bPhi^\top(\pin-\qout)$ by construction,  %
\begin{talign}\label{eq:gs_halve-1}
\Earg{ \exp\parenth{\bu^\top\K(\pin-\qout) }}
    \leq
\Earg{ \exp\parenth{-\inner{\bPhi^\top \bu}{\frac{1}{\nin}\mbf B \z}}} 
    \leq 
\exp\parenth{\frac{ \subg^2}{2} \cdot \twonorm{\bPhi^\top \bu}^2}  
    = 
\exp\parenth{\frac{\subg^2}{2} \cdot \bu^\top \mbf K \bu}.
\end{talign}
\end{proof}

Now, for $\ell\in[m]$, let $\mbi p_\ell \in \R^n$ denote the output probability vector produced by the $\ell$-th call to \gshalve. %
Defining $\mbi p_0 \defeq \pin$ and $\qout \defeq \mbi p_m$, we have
\begin{talign}
    \pin - \qout = \sum_{i=1}^m \Delta_i,\qtext{for}~ \Delta_i\defeq \mbi p_{i-1} - \mbi p_{i}
    \qtext{for} i\in[m].
\end{talign}

By \cref{lem:gs_halve}, each $\mbi p_{i-1} - \mbi p_{i}$ is $(\mbf K,\frac{2\maxnorm{\mbf K}^{1/2} }{\nin/2^{i-1}})$-sub-Gaussian conditional on $(\Delta_1, \ldots,\Delta_{i-1})$. 
Applying \cref{lem:K_sub_gsn_additivity} to the sequence $\parenth{\Delta_j}_{j=1}^m$, we find that $\pin - \qout$ is $(\K,\subg)$-sub-Gaussian with parameter
\begin{talign}
    \subg = \parenth{\sum_{j=1}^m \frac{4\maxnorm{\K}}{(\nin/2^{j-1})^2} }^{1/2} = \frac{2\maxnorm{\K}^{1/2}}{\nin} \parenth{\sum_{j=1}^m 4^j}^{1/2} \leq \frac{\maxnorm{\K}^{1/2}}{\nin} \sqrt{\frac{4}{3} 4^m}. 
\end{talign}
Simplifying the above using the fact that $\nout = \nin/2^m$ yields our desired result \cref{eq:gs_thin_parameter}.

\subsubsection{\pcref{prop:gs_thin_runtime_1}} \label{proof:quartic-GS}

    We essentially reproduce the argument from \citet{bansal2018gram} for the runtime of the \gshalve algorithm in our kernelized context.

    The main computational cost of \gshalve is the execution of the \proctwo{$\cdot$} subroutine in \cref{algo:gs_halve}.
    The number of iterations in while loop for $\bz_t$ is at most $\nin/2$. 
    This is due to the fact that in each iteration, at least one new variable is set to $ \left\{ \pm 1 \right\} $.
    Further, in each iteration, the main computational cost is the computation of 
    \begin{talign}
        \bu_{t} \gets \argmin_{\bu \in \R^{\nin/2}} \bu^\top \mbf Q \bu
    \end{talign}
    under the constraints  that $\bu_p=1$ and $\bu_i=0$ for all $i\notin \mc A$. 
    Since this can be implemented in $\bigO{\nin^3}$ time  using standard convex optimization techniques, 
\gshalve has total runtime  
\begin{talign}\label{eq:gs_halve_runtime}
    r_{\mrm H}(\ell) \leq C\ell^4
\end{talign}
for an input sequence of size $\ell$ and a constant $C$ independent of $\ell$.
Now, note that \gsthin calls \gshalve iteratively on inputs of size $\nin 2^{-i}$ for $i=0,1,\ldots, m-1 $ where $m = \log_2(\nin/\nout)$. 
Thus, \gsthin has runtime
\begin{talign}
    \sum_{i=0}^{m-1} r_{\mrm{H}} (\nin/2^i) \leq \sum_{i=0}^{m-1} C(\nin/2^i)^4 = 
    \bigO{\nin^4}.
\end{talign}

\subsubsection{\pcref{lem:gs_halve_agreement}}\label{proof:gs_halve_agreement}
    We want to reason that any round of partial coloring leads to the same output across the two algorithms.
    Fix any fractional assignment update round. 
    Recall that $\mc A_1 = \mc A \backslash \{p\}  $ and $\mc A_2 = \mc A' \,\backslash\, \{ p' \} $. 
    These represent the active set coordinates without the pivot before and after the update respectively.
     
    The main difference between \cref{algo:gs_halve_cubic,algo:gs_halve} is in the computation of the step direction $\bu_t$, which is the solution of the program
    \begin{talign}
        \bu_{t} \gets \argmin_{\bu \in \R^n} \bu^\top \mbf Q \bu
        \qtext{subject to} \bu_{p'}=1 
        \qtext{and}
        \bu_i=0
        \qtext{for all}
        i\notin \mc A'.
    \end{talign}
    $\bu_t$ has a closed form with entries
    \begin{talign}
        (\bu_{t} )_{ \mc A_2 } = - (\mbf{Q}_{\mc A_2 \times \mc A_2})^{-1} \cdot \mbf{Q}_{\mc A_2 \times  \{p'\}  }.
    \end{talign}
    Note that the invertibility of $\mbf{Q}_{\mc A_2 \times \mc A_2 }$ follows from the positive-definiteness of $\bQ$, as, for any $\w\in\R^{|\mc A_2|}$,
    \begin{align}
        \w^{\top}\mbf{Q}_{\mc A_2 \times \mc A_2 } \w = \tilde{\w}^{\top}\mbf{Q} \tilde{\w} > 0 
    \end{align}
    for a second vector $\tilde{\w}$ with $\tilde{\w}_{\mc A_2} = \w$ and all other entries equal to zero.
    Therefore, to compute $\bu_t$, it suffices to keep track of the inverse of $\mbf Q_{\mc A_2 \times \mc A_2}$ as $ \mc A' $ across iterations.

    Let $i$ be the unique element in $\mc A_1  \backslash \mc A_2  $. 
    Writing $\bQ_{\mc A_1 \times \mc A_1 }$ in block form, we have
    \begin{talign}
        \mbf Q_{\mc A_1 \times \mc A_1 } = \begin{bmatrix}
            \mbf Q_{\mc A_2  \times \mc A_2  } & \mbf Q_{\mc A_2 \times \{i\}   } \\
            \mbf Q_{ \{i\} \times \mc A_2  } & \mbf Q_{ ii }
        \end{bmatrix}.
    \end{talign}    
    By block matrix inversion  \citep[see, e.g.,][Thm. 2]{block_matrix}, the leading size $|\Aset_2|\times|\Aset_2|$ principal submatrix of $(\mbf Q_{\mc A_1 \times \mc A_1 })^{-1}$ equals
    \begin{talign} \label{eq:block_inv}
    \D \defeq
    \left(\mbf Q_{\mc A_2  \times \mc A_2  } - \frac{ \mbf Q_{ \mc A_2 \times \{ i\} } \mbf Q_{\{i\} \times \mc A_2 }    }{ \mbf Q_{ ii  }  }  \right)^{-1}.
    \end{talign}
    Thus, by the Sherman-Morrison formula \citep{sherman_morrison},   
    \begin{talign} 
        (\mbf Q_{ \mc A_2 \times A_2 })^{-1} = \left( \mbf D^{-1} + \frac{ \mbf Q_{ \mc A_2 \times \{ i\} } \mbf Q_{\{i\} \times \mc A_2 }    }{ \mbf Q_{ ii   }  }  \right)^{-1}
        =
        \mbf D - \frac{\mbf D \mbf Q_{ \mc A_2 \times \{ i\} } \mbf Q_{\{i\} \times \mc A_2 } \mbf D}{ \mbf Q_{ ii } + \mbf Q_{\{i\} \times \mc A_2 } \mbf D \mbf Q_{ \mc A_2 \times \{ i\} }}.
        \label{eq:block_inv_3}
    \end{talign}

    Hence, if we already have access to a matrix $\bC = (\mbf Q_{ \mc A_1 \times \mc A_1 })^{-1}$, we can compute $\mbf D$ by dropping the row and column of $\bC$ corresponding to $i$ and then compute $(\mbf Q_{ \mc A_2 \times \mc A_2 })^{-1}$ using \cref{eq:block_inv_3}.   
    Since in \cref{algo:gs_halve_cubic} we begin by  explicitly computing the inverse of $\mbf Q_{\mc A' \times \mc A'}$,
    the update step in \cref{algo:gs_halve_cubic} maintains the required inverse
    and thus its partial assignment updates match those of  \cref{algo:gs_halve}.
\subsubsection{\pcref{prop:gs_thin_runtime_2} } \label{proof:cubic-GS}
We begin by establishing the runtime of \proccubic{$\cdot$}.

    \begin{lemma}[\textbf{Running time of \proccubic{$\cdot$} }]\label{lem:gs_halve_cubic}
        The routine \proccubic{$\cdot$} runs in $\bigO{ \ell^3}$ time given a point sequence of size $\ell$.
    \end{lemma}
    \begin{proof}
    First, the initialization of $\bC$ costs $O(\ell^3)$ time using standard matrix inversion algorithms. 
        Second, the number of iterations in the while loop is at most $\ell/2$ since, in each iteration, at least one new variable is assigned a permanent sign in $ \left\{ \pm 1 \right\} $.
        In each while loop iteration, the main computational costs are the update of $ \mbf C $ and the computation of the step direction $\bu_t$, both of which cost $O(\ell^2)$ time using standard matrix-vector multiplication.  
        Hence, together, all while loop iterations cost $O(\ell^3)$ time.  
    \end{proof}

    Given the above lemma, we have that \gshalvecubic, on input of size $\ell$, has a running time 
\begin{talign}
    r_{\mrm{H}}(\ell) \leq C\ell^3
\end{talign}
for some $C$ independent of $\ell$.
When used in \gsthin this yields the runtime
\begin{talign}
    \sum_{i=0}^{m-1} r_{\mrm{H}} (\nin/2^i) \leq \sum_{i=0}^{m-1} C(\nin/2^i)^3 = 
    \bigO{\nin^3}.
\end{talign}

\subsection{\gscompress} 
\label{sub:gs_compress}
This section introduces and analyzes the new \gscompress algorithm (\cref{algo:gs_compress}) which combines the \compress meta-algorithm of \citet{shetty2022distributioncompressionnearlineartime} with the \gshalvecubic halving algorithm (\cref{algo:gs_halve_cubic}). The following result bounds the sub-Gaussian constant and runtime of \gscompress.

\begin{algorithm2e}[t]
    \caption{\gscompress: Compress with \gshalvecubic halving}
    \label{algo:gs_compress}
    \SetAlgoLined
      \DontPrintSemicolon
\small{
  \KwIn{point sequence $\xin=(\x_i)_{i=1}^{\nin}$, kernel $\kernel$, $\nout\in\sqrt{\nin}\cdot 2^{\naturals}$}
\BlankLine
$\ossymb \gets \log_2(\nout/\sqrt{\nin})$ 
\qquad\qquad\qquad\qquad\qquad\qquad\qquad\qquad\qquad// identify \osname \\
\BlankLine
\SetKwProg{myproc}{function}{}{}
\myproc{\proccompress{$\cset$}:}{
    \lIf{  $\abss{\cset} = 4^{\ossymb}$ }{
        \KwRet{ $\cset $ }{}
    }  
        Partition $\cset$ into four arbitrary subsequences $ \{ \cset_i \}_{i=1}^4 $ each of size $\abss{\cset}/4$ \\ 
        \For{$i=1, 2, 3, 4$}
        {$ \wtil{ \cset_i } \leftarrow$ \proccompress{$\cset_i$} \ \qquad\qquad\qquad\qquad\qquad\qquad\quad\ \  // return coresets of size $2^{\ossymb} \cdot \sqrt{\frac{\abss{\cset}}{4}}$
        }
        $ \wtil{\cset} \gets\textsc{Concatenate}( \wtil{\cset}_1, \wtil{\cset}_2,\wtil{\cset}_3, \wtil{\cset}_4)$;
        \quad $\ell \gets 2 \cdot 2^{\ossymb} \cdot \sqrt{\abss{\cset}}$ 
        \quad // coreset of size   $\ell$ \\
        \KwRet{ \textup{  $\gshalvecubic(\wtil{\cset},  \kernel)$ 
        \qquad\qquad\qquad\qquad\qquad\quad \!// coreset of size $ 2^{\ossymb} \sqrt{\abss{\cset}}$} }       
}
\BlankLine
    \KwRet{ \proccompress{$\xin$} \qquad\qquad\qquad\ \qquad\qquad\qquad\qquad\qquad\!\textup{// coreset of size $\nout = 2^{\ossymb} \sqrt{\nin}$}}{}
}
\end{algorithm2e}

\begin{proposition}[\tbf{\gscompress sub-Gaussianity and runtime}]\label{prop:gs_thin_compress}
If $\K$ is generated by $\k$, then 
\gscompress is $(\K,\subg,0)$-sub-Gaussian with
\begin{talign}
    \subg \defeq \frac{1}{\nout} \sqrt{\log_2(\nout) \maxnorm{\K}}.
\end{talign}
Moreover, \gscompress has an $\bigO{\nout^3}$ runtime.
\end{proposition}
\begin{proof}

By \cref{lem:gs_halve} and \cref{lem:gs_halve_agreement}, \gshalvecubic is $(\K,\subg_{\mrm H}(\ell))$-sub-Gaussian for an input point sequence of size $\ell$ and $\subg_{\mrm H}(\ell) = 2\sqrt{\maxnorm{K}}/\ell$.
Hence, by \cref{lem:vector_subg_funct_subg}, \gshalvecubic is also $\subg_{\mrm H}(\ell)$ $f$-sub-Gaussian in the sense of \citet[Def.~2]{shetty2022distributioncompressionnearlineartime} for each $f\in\rkhs$.  
By \citet[Rmk.~2]{shetty2022distributioncompressionnearlineartime}, \gscompress is therefore $f$-sub-Gaussian with parameter
\begin{talign}
    \subg &\leq \sqrt{\log_2(\nin/\nout)} \subg_{\mrm H}(2\nout) 
    \leq \sqrt{\log_2(\nout)} \frac{\maxnorm{\K}^{1/2}}{\nout} 
\end{talign}
for each $f\in\rkhs$.
Hence, \cref{lem:funct_subg_vector_subg} implies that \gscompress is a $(\K,\subg,0)$-sub-Gaussian thinning algorithm.

Furthermore, \citet[Thm.~1]{shetty2022distributioncompressionnearlineartime} implies that \gscompress has a runtime of 
\begin{talign}
    \sum_{i=0}^{\log_2( \nin / (2 \nout)) } 4^i \cdot r_{\mrm H}(2\nout 2^{-i}).
\end{talign}
where the \gshalvecubic runtime $r_{\mrm H}(\ell) \leq C\ell^3$ for $C$ independent of the input size $\ell$ by  \cref{lem:gs_halve_cubic}.
Therefore, the \gscompress runtime is bounded by 
\begin{talign}
    \sum_{i=0}^{\log_2( \nin / (2 \nout)) } 4^i \cdot (2\nout)^3 2^{-3i} = O(  \nout^3). 
\end{talign}
\end{proof}
\begin{remark}[\tbf{\compress with \gshalve}]
If the \gshalve implementation were used in place of \gshalvecubic, parallel reasoning would yield an  $\bigO{\nout^4}$ runtime for \gscompress.
\end{remark}

%% file: appendices/rkhd.tex
\subsection{\khd} 
\label{sub:khd}

\begin{algorithm2e}[ht!]
\caption{\khd: Kernel Halving with simplified swapping thresholds and failure probability $\delta/2$}
\label{algo:khd}
\small{
  \KwIn{point sequence $\xin=(\x_i)_{i = 1}^{\nin}$ with even $\nin$, kernel $\kernel$}
  \BlankLine
  {$\coreset[1], \coreset[2] \gets \braces{}$};\quad $\wtpsi_0\gets \boldzero \in \rkhs$\quad // Initialize empty coresets: $\coreset[1],\coreset[2]$ have size $i$ after round $i$ \\ 
  {$\multiplier_{\max,i} \gets 0$}\qquad\qquad\qquad\qquad\qquad\  // Max function norm so far \\
  \For{$i=1, 2, \ldots, \nin/2$}
    {%
    // Construct kernel difference function using next two points \\
    $(\x, \x') \gets (\x_{2i-1}, \x_{2i})$;\quad
    $f_i \gets \kernel(\x_{2i-1}, \cdot)-\kernel(\x_{2i}, \cdot)$; \quad $\eta_i \gets -1$ \\
	 \BlankLine
     // Compute swapping threshold $\thresh_i$ \\ %
     $\multiplier_i^2 \!=\! \norm{f_i}_{\kernel}^2 
      \!=\! \kernel(\x,\x)\!+\!\kernel(\x',\x')\!-\!2\kernel(x,x')$;\quad $\multiplier_{\max,i} = \max(\multiplier_i, \multiplier_{\max,i-1})$ \\
     $ \thresh_i \gets \multiplier_i \multiplier_{\max,i}(\half + \log(2\nin/\delta))$
     \ \ 
   \BlankLine
    // Compute RKHS inner product $\angles{\wtpsi_{i-1}, f_i}_{\kernel}$, which has a simple form \\
    $\alpha_i\gets  \sum_{j=1}^{2i-2}(\kernel
	 (\x_j, \x)-\kernel(\x_j,\x')) 
	 - 2\sum_{\z\in\coreset[1]}(\kernel(\z, \x)-\kernel(\z,\x'))$ \\
    \BlankLine
			 // Assign one point to each coreset after probabilistic swapping \\[2pt]
		     $(x, x') \gets (x', x)$ \text{ and } $\eta_i \gets 1$ \qtext{\textit{with probability}} $\min(1, \half (1-\frac{\alpha_i}{\thresh_i})_+)$ \\ 
           $\coreset[1]\texttt{.append}(\x); 
		        \ \ \  \coreset[2]\texttt{.append}(\x'); \ \ \ 
		     \wtpsi_i\gets \wtpsi_{i-1} + \eta_i f_i $ \quad // $ \wtpsi_i=\sum_{x'\in\coreset[2]}\!\kernel(x', \cdot)\!-\!\sum_{x\in\coreset[1]}\!\kernel(x, \cdot)$
  }
  \KwRet{\textup{$\xout\defeq\coreset[1]$, coreset of size} $\nout=\nin/2$}{} 
  } 
\end{algorithm2e}

In this section, we analyze \khd (\cref{algo:khd}), a variant of the Kernel Halving algorithm \citep[Alg.~2]{dwivedi2024kernel} with simplified swapping thresholds.
\cref{khd-sub-gaussian}, proved in \cref{proof:khd-sub-gaussian}, establishes the sub-Gaussianity of \khd and its intermediate iterates.

\begin{proposition}[\tbf{Sub-Gaussianity of \khd}]\label{khd-sub-gaussian} 
Suppose $\nin \geq 2$. 
In the notation of \cref{algo:khd}, on a common event $\event$ of probability at least $1-\delta/2$, 
for all $i\in[\nin/2]$, 
$\frac{1}{2i}\wtpsi_i$ is $(\kernel, \subg_i)$-sub-Gaussian with 
\begin{talign}\label{eq:khd-subg}
\subg_i 
    &=
\multiplier_{\max,i}\frac{\sqrt{\log(2\nin/\delta)}}{2i}
    = 
\frac{\sqrt{\log(2\nin/\delta)}}{2i}
\max_{j\in[i]}\mmd_{\k}(\dirac_{\x_{2j-1}},\dirac_{\x_{2j}})
    \leq
\frac{\sqrt{\log(2\nin/\delta)}}{2i}
\max_{j\in[i]}\mmd_{\k}(\dirac_{\x_{2j-1}},\dirac_{\x_{2j}}) \\
    &\leq
\frac{\sqrt{\log(2\nin/\delta)}}{2i}2\min(\max_{\x\in\xin}\sqrt{\k(\x,\x)},
\max_{\x\in\xin}\mmd_{\k}(\dirac_{\x},\Pin))
.
\end{talign}
\end{proposition}
\cref{khd-sub-gaussian} and the triangle inequality imply that $(\Pin -\Qout)\k=\frac{1}{\nin}\wtpsi_{\nin/2}$ is $(\k,\subg)$-sub-Gaussian on $\event$ with
\begin{talign}\label{eq:khd-subg}
\subg 
    &=
\multiplier_{\max,\nin/2}\frac{\sqrt{\log(2\nin/\delta)}}{\nin}
    \leq
\frac{\sqrt{\log(2\nin/\delta)}}{\nin}2\min(\max_{\x\in\xin}\sqrt{\k(\x,\x)},
\max_{\x\in\xin}\mmd_{\k}(\dirac_{\x},\Pin))
.
\end{talign}

By \cref{lem:funct_subg_vector_subg}, we thus have that the \khd output $\pin - \qout$ is $(\K,\subg)$-sub-Gaussian on  $\event$ for $\K$ generated by $\k$ and that $\khd \in \ksubge$.

\subsubsection{\pcref{khd-sub-gaussian}}\label{proof:khd-sub-gaussian}
We begin by studying the sub-Gaussian properties of a related algorithm, the self-balancing Hilbert walk (SBHW) of \citet[Alg.~3]{dwivedi2024kernel}. 
By \citet[Thm.~3(i)]{dwivedi2024kernel}, when the SBHW is run on the RKHS $\rkhs$ with the same $f_i$ and $\thresh_i$ sequences employed in \khd, the output $\psi_i$ of each round is $(\kernel, \sigma_i)$-sub-Gaussian for 
\begin{talign} \label{eq:sbhw-subg}
\sig_0^2
    \defeq 0
\qtext{and}
\sig_i^2
    \defeq \sig_{i-1}^2 + \knorm{f_{i}}^2\big(1 + \frac{\sig_{i-1}^2}{\thresh_i^2}(\knorm{f_{i}}^2 - 2\thresh_i)\big)_+
\quad \forall i \geq 1.
\end{talign}
The following lemma bounds the growth of the sub-Gaussian constants $\sig_i$ in terms of the swapping thresholds $\thresh_i$.
\begin{lemma}[Growth of SBHW sub-Gaussian constants]\label{sbhw-subg-growth}
For each $i$, the SBHW sub-Gaussian constants \cref{eq:sbhw-subg} satisfy
\begin{talign}
\sig_i^2 
    \leq 
c_i
\qtext{for}
c_i \defeq \max_{j \in [i]} \max(\multiplier_j^2, r_j)
\qtext{and}
r_i 
    \defeq 
\frac{\thresh_i^2}{2\thresh_i - \multiplier_i^2}
    \leq
\frac{\thresh_i^2}{2\thresh_i - \multiplier_i\multiplier_{\max,i}}.
\end{talign}
\end{lemma}
\begin{proof}
We will prove the result by induction on $i$.
\paragraph{Base case.} $\sig_1^2 = \multiplier_1^2 \leq c_1$ as desired.
\paragraph{Inductive case.}
Suppose $\sig_{i-1}^2 \leq c_{i-1}$.
Then $\sig_i^2 = g(\sig_{i-1}^2)$ for $g(x) = x + \multiplier_i^2 ( 1-x/r_i)_+$.
Note that the slope of $g$ is $1 - \multiplier_i^2/r_i$ for $x < r_i$ and $1$ for $x > r_i$. 
If $1 - \multiplier_i^2/r_i \geq 0$, then $g$ is increasing and its maximum value over $[0, c_{i}]$ is at $c_{i}$.
If, on the other hand, $1 - \multiplier_i^2/r_i < 0$, then $g$ first decreases and then increases so its maximum value over $[0, c_{i}]$ is either at $0$ or at $c_i$. 
Since $c_i \geq \max(r_i, c_{i-1})$, $\sig_i^2 \leq \max(g(0), g(c_i)) = \max(\multiplier_i^2, c_i) = c_i$.
The proof is complete.
\end{proof}
Invoking \cref{sbhw-subg-growth}, the assumption $\nin \geq 2$, and the fact that $\delta\mapsto\frac{\half+\log(2/\delta)}{\log(2/\delta)}$ is increasing on $(0,1]$, we find that
\begin{talign}\label{eq:sbhw-subg-bound}
\sigma_i^2
    \leq
\multiplier_{\max,i}^2 
\log(2\nin/\delta)\frac{(\half + \log(2\nin/\delta))^2}{2(\log(2\nin/\delta))^2}
    \leq
\multiplier_{\max,i}^2
\log(2\nin/\delta)\frac{(\half + \log(4))^2}{2(\log(4))^2}
    \leq 
\multiplier_{\max,i}^2
\log(2\nin/\delta).
\end{talign}
The first inequality in \cref{eq:sbhw-subg-bound} and the definition \cref{eq:sbhw-subg} further imply that
\begin{talign}
\thresh_i 
    =
\multiplier_i \multiplier_{\max,i} (\half + \log(2\nin/\delta))
    \geq 
\sig_i \multiplier_i \sqrt{2\log(2\nin/\delta)}
    \geq
\sig_{i-1} \multiplier_i \sqrt{2\log(2\nin/\delta)}.
\end{talign}
Hence, by \citet[Thm.~3(iii)]{dwivedi2024kernel}, for each $i \in [\nin/2]$, the vector $\wtpsi_i$ of \khd coincides with the vector $\psi_i$ of SBHW on a common event $\event$ of probability at least $1-\delta/2$. 
Therefore, each $\frac{1}{2i}\wtpsi_i$ is $(\kernel,\frac{1}{2i}\sig_i)$-sub-Gaussian on $\event$, implying the result.

\subsection{\khlind}
\label{sub:khlind}
\SetKwFunction{procgetswap}{\texttt{get\_swap\_params}}

\begin{algorithm2e}[ht!]

    \caption{\khlind: Kernel Halving with linear kernel and failure probability $\delta/2$}
    \label{algo:khlin}
    \small{
      \KwIn{point sequence $\xin=(\x_i)_{i = 1}^{\nin}$  with even $\nin$ and $\x_i \in \mathbb{R}^d$} 
      \BlankLine
      {$\coreset[1], \coreset[2] \gets \braces{}$};\quad $\psi_0\gets \boldzero \in \mathbb{R}^d $\quad // Initialize empty coresets: $\coreset[1],\coreset[2]$ have size $i$ after round $i$ \\ 
      $\sigma_0 \gets 0 $  
      \qquad\qquad\qquad\qquad\qquad\quad\ \  // Keep track of sub-Gaussian constant \\
      \For{$i=1, 2, \ldots, \nin/2$}
        {%
        // Consider two points \\
        $(\x, \x') \gets (\x_{2i-1}, \x_{2i})$;
        \quad $\eta_i \gets -1$ \\
         \BlankLine
         // Compute swapping threshold $\thresh_i$ \\ %
          $\multiplier_i^2 = \angles{ \x - \x' , \x - \x' }   $; \quad $\delta_i = \frac{\delta}{2i (\log (\nin/2 ) + 1)}$ \\  
          $(\thresh_i, \sigma_i ) \gets$ \procgetswap{$\sigma_{i-1}, \multiplier_i , \delta_i$}\\  
       \BlankLine
        // Compute inner product \\
        $\alpha_i\gets \angles{\psi_{i-1}, \x - \x' }  $ \\
        \BlankLine
                 // Assign one point to each coreset after probabilistic swapping \\[2pt]
                 $(\x, \x') \gets (\x', \x)$ \text{ and } $\eta_i \gets 1$ \qtext{\textit{with probability}} $\min(1, \half (1-\frac{\alpha_i}{\thresh_i})_+)$ \\ 
               $\coreset[1]\texttt{.append}(\x); 
                    \ \ \  \coreset[2]\texttt{.append}(\x'); \ \ \ 
                 \wtpsi_i\gets \wtpsi_{i-1} + \eta_i f_i $ 
      }
      \KwRet{\textup{$\xout\defeq\coreset[1]$, coreset of size} $\nout=\nin/2$}{} 
      } 

      \hrulefill\\

      \SetKwProg{myproc}{function}{}{}
     \myproc{\procgetswap{$\sigma, \vmax[], \delta$}:}{
     $
            \cnew[] 
                \gets \max(\vmax[] \sigma\sqrt{\smash[b]{2\log(2/\delta)}}, \vmax[]^2)$ \\
     $\sigma^2 \gets \sigma^2
            \!+\! \vmax[]^2(1 \!+\! ({\vmax[]^2}{}\! - \!2\cnew[]){\sigma^2}{/\cnew[]^2})_+$\\
     }
     \KwRet{$(\cnew[], \sigma)$}\;

    \end{algorithm2e}

In this section, we analyze \khlind (\cref{algo:khlin}), the Kernel Halving algorithm of \citep[Alg.~2]{dwivedi2024kernel} 
with a linear kernel, $\k(\x,\y) = \inner{\x}{\y}$, on $\reals^d$ and failure probability $\delta/2$.
Notably, \cref{algo:khlin} can be carried out in only $O(nd)$ time thanks to the linear kernel structure.
\cref{khlind-sub-gaussian}, proved in \cref{proof:khlind-sub-gaussian}, establishes the sub-Gaussianity of \khlind and its intermediate iterates.

\begin{proposition}[\tbf{Sub-Gaussianity of \khlind}]\label{khlind-sub-gaussian} 
Suppose $\nin \geq 2$. 
In the notation of \cref{algo:khlin}, on a common event $\event$ of probability at least $1-\delta/2$, 
for all $i\in[\nin/2]$, 
$\frac{1}{2i}\wtpsi_i$ is $(\kernel, \subg_i)$-sub-Gaussian with $\k(\x,\y)=\inner{\x}{\y}$ and 
\begin{talign}\label{eq:khlind-subg}
\subg_i 
    &=
\frac{\sqrt{\log(2\nin(\log(\nin/2)+1)/\delta)}}{2i}
\max_{j\in[i]}\twonorm{\x_{2j-1}-\x_{2j}} \\
    &\leq
\frac{\sqrt{\log(2\nin(\log(\nin/2)+1)/\delta)}}{2i}2\min(\max_{\x\in\xin}\sqrt{\twonorm{\x}},
\max_{\x\in\xin}\twonorm{\x-\xbar})
    \qtext{for}
\xbar=\frac{1}{\nin}\sum_{\x\in\xin}\dirac_{\x}.
\end{talign}
\end{proposition}

\cref{khlind-sub-gaussian} and the triangle inequality imply that $(\Pin -\Qout)\k=\frac{1}{\nin}\psi_{\nin/2}$ is $(\k,\subg)$-sub-Gaussian on $\event$ with
\begin{talign}\label{eq:khlind-subg}
\subg 
    &=
\frac{\sqrt{\log(2\nin(\log(\nin/2)+1)/\delta)}}{\nin}
\max_{j\in[\nin/2]}\twonorm{\x_{2j-1}-\x_{2j}} \\
    &\leq
\frac{\sqrt{\log(2\nin(\log(\nin/2)+1)/\delta)}}{\nin}2\min(\max_{\x\in\xin}\sqrt{\twonorm{\x}},
\max_{\x\in\xin}\twonorm{\x-\xbar})
    \qtext{for}
\xbar=\frac{1}{\nin}\sum_{\x\in\xin}\dirac_{\x}.
\end{talign}

By \cref{lem:funct_subg_vector_subg}, we thus have that the \khlind output $\pin - \qout$ is $(\K,\subg)$-sub-Gaussian on  $\event$ for $\K$ generated by $\k$ and that $\khlind \in \ksubge$.
\subsubsection{\pcref{khlind-sub-gaussian}}\label{proof:khlind-sub-gaussian}
We begin by studying the sub-Gaussian properties of a related algorithm, the self-balancing Hilbert walk (SBHW) of \citet[Alg.~3]{dwivedi2024kernel}. 
By \citet[Thm.~3(i)]{dwivedi2024kernel}, when the SBHW is run on the RKHS $\rkhs$ with the same $f_i$ and $\thresh_i$ sequences employed in \khlind, the output $\psi_i$ of each round is $(\kernel, \sigma_i)$-sub-Gaussian. 
Moreover, since  
\begin{talign}
\thresh_i 
    \geq
\sig_{i-1} \multiplier_i \sqrt{2\log(2/\delta_i)}
    \qtext{for each}
i\in[\nin/2], 
\end{talign}
\citet[Thm.~3(iii)]{dwivedi2024kernel} implies that, for each $i \in [\nin/2]$, the vector $\wtpsi_i$ of \khlind coincides with the vector $\psi_i$ of SBHW on a common event $\event$ of probability at least $1-\delta/2$. 
Therefore, each $\frac{1}{2i}\wtpsi_i$ is $(\kernel,\frac{1}{2i}\sig_i)$-sub-Gaussian on $\event$.
Finally, \citet[(46)]{dwivedi2024kernel} shows that  $\sig_i \leq \subg_i$ for each $i \in [\nin/2]$, yielding the result.

\subsection{\rkhd} \label{sec:rkhd}
\begin{algorithm2e}[ht!]
\caption{\rkhd: Repeated \khd}
\label{algo:rkhd}
\small{
  \KwIn{point sequence $\xin=(\x_i)_{i = 1}^{\nin}$, kernel $\kernel$, output size $\nout \in \nin / 2^\naturals$}
  \BlankLine
  // Repeatedly divide coreset size in half \\
  $m \gets \log_2(\nin/\nout)$ \\ 
  \lFor{$\ell=1, 2, \ldots, m$}
    {%
    $\xin \gets \kh(\delta/m)(\xin, \kernel)$
  }
  \KwRet{\textup{$\xout\defeq\xin$, coreset of size} $\nout=\nin/2^m$}{} 
  } 
\end{algorithm2e}
In this section, we analyze repeated \khd (\rkhd, \cref{algo:rkhd}), a variant of the \ktsplit algorithm \citep[Alg.~1a]{dwivedi2024kernel} with simplified swapping thresholds.
Our next result, proved in \cref{proof:rkhd-sub-gaussian}, establishes the sub-Gaussianity of \rkhd.
\begin{proposition}[\tbf{Sub-Gaussianity of \rkhd}]\label{rkhd-sub-gaussian} 
If $\nout \in \nin/2^\naturals$ then \rkhd (\cref{algo:rkhd}) is $(\k,\subg)$-sub-Gaussian with 
\begin{talign}
\subg 
    =
\frac{2}{\nout\sqrt{3}}\sqrt{\log(\frac{6\nout \log_2(\nin/\nout)}{\delta})}
 \min(\max_{\x\in\xin}\sqrt{\k(\x,\x)},
\max_{\x\in\xin}\mmd_{\k}(\dirac_{\x},\Pin))
\end{talign}
on an event $\event$ of probability at least $1-\delta/2$.
\end{proposition}
By \cref{lem:funct_subg_vector_subg}, we thus have that the \rkhd output $\pin - \qout$ is $(\K,\subg)$-sub-Gaussian on  $\event$ for $\K$ generated by $\k$ and that $\rkhd \in \ksubge$.
Finally, $\subg = O(\frac{\sqrt{\log(\nout/\delta)}}{\nout})$ when $\nout \geq \sqrt{\nin}$.

\subsubsection{\pcref{rkhd-sub-gaussian}}\label{proof:rkhd-sub-gaussian}
Let $c = 2\min(\max_{\x\in\xin}\sqrt{\k(\x,\x)},
\max_{\x\in\xin}\mmd_{\k}(\dirac_{\x},\Pin))$, and, for each $\ell\in[m]$, let $\wtpsi^{(\ell)}$ represent the vector $\wtpsi_{\nin/2^\ell}$ produced at the end of the $\ell$-th call to \khd.
By the proof of \cref{khd-sub-gaussian} and the union bound, on an event $\event$ of probability at least $1-\delta/2$, $(\wtpsi^{(\ell)})_{\ell\in[m]} = (\psi^{(\ell)})_{\ell\in[m]}$, where each $\frac{2^{\ell-1}}{\nin}\psi^{(\ell)}$ is $(\k,\nu^{(\ell)})$-sub-Gaussian given $(\psi^{(j)})_{j\in[\ell-1]}$ for
\begin{talign}
\nu^{(\ell)}
    =
c \frac{\sqrt{\log(2\nin m/(2^{\ell-1}\delta))}}{\nin/2^{\ell-1}}.
\end{talign}
Hence, on $\event$, the weighted sum
\begin{talign}
(\Pin-\Qout)\k
    =
\sum_{\ell\in[m]} \frac{2^{\ell-1}}{\nin}\wtpsi^{(\ell)}
    =
\sum_{\ell\in[m]} \frac{2^{\ell-1}}{\nin}\psi^{(\ell)}
\end{talign}
is $(\k,\sqrt{\sum_{\ell\in[m]}(\subg^{(\ell)})^2})$-sub-Gaussian by \citet[Lem.~14]{dwivedi2024kernel}.
Finally, by \citet[Eq.~(63)]{dwivedi2024kernel}, $\sqrt{\sum_{\ell\in[m]}(\subg^{(\ell)})^2}\leq \subg$.

%% file: appendices/khcompressd.tex
\subsection{\khcompressd}
\label{sub:khcompress}
\SetKwFunction{proccompress}{\texttt{compress}}

\begin{algorithm2e}[t]
\caption{\khcompressd: Compress with \kh halving and failure probability $\delta$}
\label{algo:khcompressd}
\SetAlgoLined
  \DontPrintSemicolon
\small{
  \KwIn{point sequence $\xin=(\x_i)_{i=1}^{\nin}$, kernel $\kernel$, $\nout\in\sqrt{\nin}\cdot 2^{\naturals}$} %
\BlankLine
$\ossymb \gets \log_2(\nout/\sqrt{\nin})$ 
\qquad\qquad\qquad\qquad\qquad\qquad\qquad\qquad\qquad// identify \osname \\
\BlankLine
\SetKwProg{myproc}{function}{}{}
\myproc{\proccompress{$\cset$}:}{
    \lIf{  $\abss{\cset} = 4^{\ossymb}$ }{
        \KwRet{ $\cset $ }{}
    }  
        Partition $\cset$ into four arbitrary subsequences $ \{ \cset_i \}_{i=1}^4 $ each of size $\abss{\cset}/4$ \\ 
        \For{$i=1, 2, 3, 4$}
        {$ \wtil{ \cset_i } \leftarrow$ \proccompress{$\cset_i$} \ \qquad\qquad\qquad\qquad\qquad\qquad\quad\ \  // return coresets of size $2^{\ossymb} \cdot \sqrt{\frac{\abss{\cset}}{4}}$
        }
        $ \wtil{\cset} \gets\textsc{Concatenate}( \wtil{\cset}_1, \wtil{\cset}_2,\wtil{\cset}_3, \wtil{\cset}_4)$;
        \quad $\ell \gets 2 \cdot 2^{\ossymb} \cdot \sqrt{\abss{\cset}}$ 
        \quad // coreset of size   $\ell$ \\
        \KwRet{ \textup{  $\kh\Big(\frac{ \ell^2}{\nin 4^{\ossymb+1}(\log_4 \nin-\ossymb)}\delta\Big)(\wtil{\cset},  \kernel)$ 
        \,\,\qquad\qquad\qquad\qquad \!// coreset of size $ 2^{\ossymb} \sqrt{\abss{\cset}}$} }       
}
\BlankLine
    \KwRet{ \proccompress{$\xin$} \qquad\qquad\qquad\ \qquad\qquad\qquad\qquad\qquad\!\textup{// coreset of size $\nout = 2^{\ossymb} \sqrt{\nin}$}}{}
}
\end{algorithm2e}

In this section, we analyze \khcompressd  (\cref{algo:khcompressd}), a variant of the \textsc{KT-Split-Compress} algorithm \citep[Ex.~3]{shetty2022distributioncompressionnearlineartime} with simplified swapping thresholds.

\begin{proposition}[\tbf{Sub-Gaussianity of \khcompressd}]\label{khcompressd-sub-gaussian} 
If $\nout \in \sqrt{\nin} \,2^\naturals$ then \khcompressd (\cref{algo:khcompressd}) is $(\k,\subg)$-sub-Gaussian with 
\begin{talign}\label{khcompressd-subg}
\subg 
    =
\frac{1}{\nout}\sqrt{\log_2(\nout)\log(\frac{4\nout \log_2(\nin/\nout)}{\delta})}
 \max_{\x\in\xin}\sqrt{\k(\x,\x)}
\end{talign}
on an event $\event$ of probability at least $1-\delta/2$.
\end{proposition}
\begin{proof}
Since the original Kernel Halving algorithm of \citet[Alg.~2]{dwivedi2024kernel} is equal to the \ktsplit algorithm of \citet[Alg.~1a]{dwivedi2024kernel} with $m=1$ halving round, 
\khcompressd is simply the \textsc{KT-Split-Compress} algorithm of \citep[Ex.~3]{shetty2022distributioncompressionnearlineartime} with $\khd$ of \cref{algo:khd} substituted for $\ktsplit(\delta,m=1)$.  
The result now follows immediately from the \khd sub-Gaussian constant of \cref{khd-sub-gaussian} and the argument of \citet[Rem.~2, Ex.~3]{shetty2022distributioncompressionnearlineartime}.
\end{proof}

\cref{khcompressd-sub-gaussian} and \cref{lem:funct_subg_vector_subg} imply that 
$\khcompress(\delta)  \in \ksubge$ for any kernel $\kernel$ that generates $\K$. 
In addition, $\subg= \bigO{\frac{ \sqrt{\log(\nout) \log(\nout/\delta)}}{\nout}}$ when $\nout \geq \sqrt{\nin}$. 
Furthermore, \citet[Rem.~1]{shetty2022distributioncompressionnearlineartime} implies that $\khcompress(\delta)$ has a runtime less than $4^{\ossymb +1} \nin\parenth{\log_4(\nin) - \ossymb} = 4 \nout^2 \log_2 (\nin/\nout) = O(\nout^2\log\nout)$ when $\nout \geq \sqrt{\nin}$.

%% file: appendices/proof_of_subg_low_rank_gen_kernel.tex
\section{\pcref{thm:subg_low_rank_gen_kernel}}\label{proof:subg_low_rank_gen_kernel}

We establish the MMD bound \cref{eq:mmd_bound} in \cref{proof:mmd_bound}, the first kernel max seminorm bound \cref{eq:ind_bound} in \cref{proof:ind_bound}, and the Lipschitz kernel max seminorm bound \cref{eq:ind_bound_lipschitz} in \cref{proof:ind_bound_lipschitz}.
Throughout, we use the notation
$\Pevent(\event') \defeq \P(\event,\event')$ for events $(\event,\event')$.
\subsection{Proof of MMD bound~\cref{eq:mmd_bound}} \label{proof:mmd_bound}
Without loss of generality, we suppose that $r\leq \rank{\K}$. Let $\V\bLam\V^\top$ be an eigendecomposition of $\K$ with orthonormal $\V\in\reals^{n\times n}$ and diagonal $\bLam = \diag(\lam_1, \cdots, \lam_n) \in \real^{n\times n}$. Let $\V_r$ represent the first $r$ columns of $\V$,
and let $\V_{-r}$ represent the last $n-r$ columns of $\V$.
Introduce the shorthand
\begin{talign}
    \diff \defeq \pin -\qout \in\real^n
    \qtext{and}
    \bPhi\defeq\V\bLam^{1/2}\V^\top \in \real^{n\times n}.
    \label{eq:shorthands}
\end{talign}
We can directly verify that
\begin{talign}
    \V\V\tp = \V\tp\V = \ident,
    \quad
    \V\V\tp = \V_r\V_r\tp + \V_{-r}\V_{-r}\tp,
    \quad
    \qtext{and}
    \K = \bPhi\bPhi\tp.
    \label{eq:phi_v_eq}
\end{talign}
Using the above equalities, 
we decompose the squared MMD into two components,
\begin{talign}
\mmd^2_{\mkernel}(\pin, \qout)  
    =
\diff^\top \K \diff
    =
\diff^\top \bPhi\bPhi^\top\diff
    = 
\diff^\top \bPhi\V\V^\top\bPhi^\top\diff 
    &=
\diff^\top \bPhi\V_r\V_r^\top\bPhi^\top\diff
    +
\diff^\top \bPhi\V_{-r}\V_{-r}^\top\bPhi^\top\diff \\
    &=
\twonorm{\V_r^\top\bPhi^\top\diff}^2
    +
\twonorm{\V_{-r}^\top\bPhi^\top\diff}^2.
\label{eq:decompose_mmd}
\end{talign}
In \cref{sub:proof_of_top_r_err,sub:proof_of_residual} respectively, we will establish the bounds 
\begin{talign}
    &\P(\twonorm{\V_{r}^\top\bPhi^\top\diff}^2\leq 
    e\subg^2 (er+\log(1/\delta'))
    \geq 1-\delta/2-\delta'
    \qtext{and}
    \label{eq:top_r_err}\\
    &\P(\twonorm{\V_{-r}^\top\bPhi^\top\diff}^2
    \leq  \lam_{r+1}(\frac{1}{\nout}-\frac{1}{n})) = 1,
    \label{eq:residual} 
\end{talign}
which when combined with \cref{eq:decompose_mmd} yield the advertised claim~\cref{eq:mmd_bound} on the squared MMD. 
\subsubsection{Proof of \cref{eq:top_r_err}: Bounding $\twonorm{\V_r^\top\bPhi^\top\diff}^2$}
\label{sub:proof_of_top_r_err}
Our first lemma bounds the Euclidean norm of a vector in terms of a finite number of inner products.
\begin{lemma}[\tbf{Euclidean norm cover}]\label{norm-cover}
For any $\bv \in\reals^{r}$ and $\vareps\in (0,1)$, 
\begin{talign}
\twonorm{\bv} \leq \frac{1}{1-\vareps}\max_{\bu \in \cover_{\vareps,r}}\inner{\bu}{\bv}\label{eq:ltwo_norm_eps}
\end{talign} 
for a set 
$\cover_{\vareps,r}$ contained in the ball $\ball^r$ with $|\cover_{\vareps,r}| \leq (1+2/\vareps)^r$.
\end{lemma}
\begin{proof}
Fix any $\vareps \in (0,1)$, 
and let $\cover_{\vareps,r}$ be a set of minimum cardinality satisfying 
\begin{talign}
\cover_{\vareps,r} \subset \ball^r
\qtext{and}
\sup_{\bu\in\ball^r}\min_{\bu'\in\cover_{\vareps,r}} \twonorm{\bu - \bu'}\leq \vareps.
\end{talign}
By \citet[Lem.~5.2]{wainwright2019high}, $|\cover_{\vareps,r}| \leq (1+2/\vareps)^r$.
Now we invoke the variational representation of $\twonorm{\cdot}$ and the Cauchy-Schwarz inequality to conclude that
\begin{talign}
\twonorm{\bv}
    =
\sup_{\bu\in\ball^r}\inner{\bu}{\bv}
    &=
\sup_{\bu\in\ball^r}\min_{\bu'\in\cover_{\vareps,r}}
\brackets{\inner{\bu-\bu'}{\bv}+\inner{\bu'}{\bv}}\\
    &\leq
\sup_{\bu\in\ball^r}\min_{\bu'\in\cover_{\vareps,r}}
\twonorm{\bu-\bu'}\twonorm{\bv}+\max_{\bu'\in\cover_{\vareps,r}}\inner{\bu'}{\bv}\\
    &\leq
\vareps \twonorm{\bv}+\max_{\bu'\in\cover_{\vareps,r}}\inner{\bu'}{\bv}.
\end{talign}
Rearranging terms yields the claimed bound~\cref{eq:ltwo_norm_eps}.
\end{proof}

Our next lemma uses this covering estimate to bound the exponential moments of $\twonorm{\V_r^\top\bPhi^\top\diff}$.
\begin{lemma}[\tbf{Norm sub-Gaussianity}] \label{norm-subg} For any $\vareps>0$ and any $t > 0$,
\begin{talign}
\Eevent[\exp(t\twonorm{\V_r^\top\bPhi^\top\diff})] 
    \leq   
(1+\frac{2}{\vareps})^r \exp(\frac{\subg^2t^2}{2(1-\eps)^2}).
\end{talign}
\end{lemma}
\begin{proof}
Fix any $t>0$. 
Since $x\mapsto \exp(tx)$ is increasing, \cref{norm-cover} implies that 
\begin{talign}
\Eevent[\exp(t\twonorm{\V_r^\top\bPhi^\top\diff})]
    &\leq
\Eevent[\exp(t\cdot \frac{1}{1-\vareps}\max_{\bu \in \cover_{\vareps,r}}\inner{\bu}{\V_r^\top\bPhi^\top\diff})] \\
    &=
\Eevent[\max_{\bu \in \cover_{\vareps,r}}\exp(\frac{t}{1-\vareps}\inner{\V_r\bu}{\bPhi^\top\diff})] \\ 
    &\leq
\sum_{\bu \in \cover_{\vareps,r}}\Eevent[\exp(\frac{t}{1-\vareps}\inner{\V_r\bu}{\bPhi^\top\diff})] 
\end{talign}
for a subset $\cover_{\vareps,r}$ with $|\cover_{\vareps,r}|\leq (1+\frac{2}{\vareps})^r$ and $\twonorm{\bu}\leq 1$ for each $\bu\in\cover_{\vareps,r}$.

Now fix any $\bu\in\cover_{\vareps,r}$ and 
let $\bLam_r=\diag(\lam_1, \dots, \lam_r)$. Using \cref{eq:shorthands,eq:phi_v_eq}, we have
\begin{talign}
\V_r &= \bPhi\tp \V_r \bLam_r^{-1/2}
	\qtext{and therefore} \\
\inner{\V_r\bu}{\bPhi^\top\diff}
    &=
\inner{\bPhi\tp\V_r \bLam_r^{-1/2}\bu}{\bPhi^\top\diff}
    =
\inner{\V_r \bLam_r^{-1/2}\bu}{\K\diff}.
\end{talign}
In addition, we have 
\begin{talign}
    (\V_r \bLam_r^{-1/2}\bu) \tp \K (\V_r \bLam_r^{-1/2}\bu)
    = \bu\tp  \bLam_r^{-1/2} \V_r\tp \V \bLam \V\tp \V_r \bLam_r^{-1/2}\bu  = \bu\tp\bu.
\end{talign}
Next, we can invoke our sub-Gaussianity assumption (\cref{def:alg-subg}) to conclude that
\begin{talign}
\Eevent[\exp(\frac{t}{1-\vareps}\inner{\V_r\bu}{\bPhi^\top\diff})] 
	=
\Eevent[\exp(\frac{t}{1-\vareps}\inner{\V_r \bLam_r^{-1/2}\bu}{\K\diff})] 
    &\leq
\exp(\frac{\subg^2t^2}{2(1-\vareps)^2} \inner{\V_r \bLam_r^{-1/2}\bu}{\K\V_r \bLam_r^{-1/2}\bu})\\
    &\leq
\exp(\frac{\subg^2t^2}{2(1-\vareps)^2} \twonorm{\bu}^2).
\end{talign}
Since $\twonorm{\bu}\leq 1$ and $|\cover_{\vareps,r}|\leq (1+\frac{2}{\vareps})^r$, the advertised result now follows.
\end{proof}
By Markov's inequality \citep{markov1884certain} and \cref{norm-subg}, for any $\alpha > 0$, 
\begin{talign}
\P(\twonorm{\V_r^\top\bPhi^\top\diff} > \alpha)
    &=
    \Pevent(\twonorm{\V_r^\top\bPhi^\top\diff} > \alpha) 
    +
    \P(\twonorm{\V_r^\top\bPhi^\top\diff} > \alpha, \event^c) 
    \\
    &\leq \Pevent(\twonorm{\V_r^\top\bPhi^\top\diff} > \alpha) + \P(\event^c) \\
    &\leq
\inf_{t > 0} \Eevent[\exp(t\twonorm{\V_r^\top\bPhi^\top\diff}) ]/\exp(t\alpha) 
+ \delta/2
\\
    &\leq
(1+\frac{2}{\vareps})^r \inf_{t > 0}\exp(\frac{\subg^2t^2}{2(1-\vareps)^2}-t\alpha)
+ \delta/2 \\
    &=
(1+\frac{2}{\vareps})^r \exp(\frac{-(1-\vareps)^2\alpha^2}{2\subg^2})
+ \delta/2.
\end{talign}
Next, we have
\begin{talign}
(1+\frac{2}{\vareps})^r \exp(\frac{-(1-\vareps)^2\alpha^2}{2\subg^2}) \leq \delta'
\qtext{if}
\alpha \geq \frac{\subg\sqrt{2}}{1-\vareps} \sqrt{\log(\frac{1}{\delta'}) + r\log(1+\frac{2}{\vareps})}.
\end{talign}
Since this bound holds for any $\vareps$, 
choosing $\vareps = 1-\sqrt{2/e}$, we find that 
\begin{talign}
\twonorm{\V_r^\top\bPhi^\top\diff}^2 
    \leq  
e\subg^2 \brackets{r\log(1+2/(1-\sqrt{2/e}))+\log(1/\delta')}
    \leq  
e\subg^2 \brackets{er+\log(1/\delta')}
\end{talign}
with probability at least $1-\delta/2-\delta'$ as claimed.
\subsubsection{Proof of \cref{eq:residual}: Bounding $\twonorm{\V_{-r}^\top\bPhi^\top\diff}^2$}
\label{sub:proof_of_residual}
Since
\begin{talign}
    \twonorm{\diff}^2 = \pin\tp\pin + \qout\tp\qout - 2\pin\tp\qout = \frac{\nin}{\nin^2} + \frac{\nout}{\nout^2} - \frac{2\nout}{\nin\nout}
    = \frac{1}{\nout} - \frac{1}{\nin},
    \label{eq:w_norm}
\end{talign}
we have, for $\bLam_{-r} \defeq \diag(\lam_{r+1},\cdots,\lam_{n})$ and $\maxeig$ the maximum eigenvalue of a SPSD matrix,
\begin{talign}\label{eq:Vminusrnorm}
\twonorm{\V_{-r}^\top\bPhi^\top\diff}^2
    =
\diff^\top\V_{-r}\bLam_{-r}\V_{-r}^\top\diff
    \leq
\maxeig(\V_{-r}\bLam_{-r}\V_{-r}^\top) \twonorm{\diff}^2
    \seq{\cref{eq:w_norm}}
\lam_{r+1}(\frac{1}{\nout}-\frac{1}{\nin}).
\end{talign}
\subsection{Proof of kernel max seminorm bound~\cref{eq:ind_bound}}\label{proof:ind_bound}

We begin by establishing a general bound on the maximum discrepancy between input and output expectations over a collection of test functions admitting a finite cover.
\begin{lemma}[\tbf{Discrepancy cover bound}]\label{discrepancy-cover-bound}
Fix any kernel $\k$, subset $\Fset \subset \rkhs$, and scalars $\vareps \geq 0$ and $\delta' \in (0,1)$.
Define 
\begin{talign}
a\defeq \sup_{f\in\Fset}\knorm{f}
\qtext{and}
\ball_{\Fset} \defeq \{ f \in \rkhs : \knorm{f}\leq a\},
\end{talign} 
and let 
$\cover_{\eps,\Fset}$ 
be a set of minimum cardinality satisfying 
\begin{talign}\label{eq:F-cover}
\cover_{\eps,\Fset} \subset \ball_{\Fset}
\qtext{and}
\sup_{f\in\Fset}\min_{f'\in\cover_{\eps,\Fset}}\max_{\x\in\xin} |f(\x) - f'(\x)|\leq \vareps.
\end{talign}
If $(\Pin - \Qout)\kernel$ is $(\kernel,\subg)$-sub-Gaussian on an event $\event$ (\cref{def:function-subg}), then, on $\event$, 
\begin{talign}
\fnorm 
	\defeq 
\sup_{f\in\Fset}(\Pin-\Qout)f
	\leq 2\eps + 
\subg a \sqrt{2 \log({|\cover_{\eps,\Fset}|}{/\delta'})}
\qtext{with probability at least $1-\delta'$.}
\end{talign}
\end{lemma}
\begin{proof}
The triangle inequality and the covering property \cref{eq:F-cover} together imply that, with probability $1$, 
\begin{talign}
(\Pin-\Qout)f
    &\leq
\min_{f'\in\cover_{\eps,\Fset}}
(\Pin-\Qout)f'
    + 
|(\Pin-\Qout)(f-f')| \\
    &\leq
\fnorm[\cover_{\eps,\Fset}]
    +
\min_{f'\in\cover_{\eps,\Fset}} 
|\Pin(f-f')|+|\Qout(f-f')| \\
    &\leq
\fnorm[\cover_{\eps,\Fset}]
    +
2\min_{f'\in\cover_{\eps,\Fset}} 
\max_{\x\in\xin} |f(\x) - f'(\x)| \\
    &\leq
\fnorm[\cover_{\eps,\Fset}]
    +
2\vareps \label{eq:prob-diff-f-cover}
\end{talign}
for each $f\in\Fset$.
Since $s\mapsto e^{ts}$ is increasing, the bound \cref{eq:prob-diff-f-cover}, the assumed sub-Gaussianity (\cref{def:function-subg}), and the fact that $\cover_{\eps,\Fset}$ belongs to $\ball_{\Fset}$ imply that
\begin{talign}
\Eevent[\exp(t\fnorm)] 
   &\leq 
e^{2t\vareps}\Eevent[\exp(t \fnorm[\cover_{\eps,\Fset}])] \\
	&\leq 
\sum_{f'\in\cover_{\eps,\Fset}} e^{2t\vareps}\Eevent[\exp(t (\Pin-\Qout)f')] \\ 
	&\leq 
\sum_{f'\in\cover_{\eps,\Fset}} \exp(\frac{t^2\subg^2 \knorm{f'}^2}{2} + 2t\eps) 
	\leq 
|\cover_{\eps,\Fset}| \exp(\frac{t^2\subg^2 a^2}{2} + 2t\eps).
\end{talign}
Now, by Markov's inequality \citep{markov1884certain}, for any $\alpha > 0$, 
\begin{talign}
    \Pevent(\sup_{f\in\Fset}(\Pin-\Qout)f >\alpha + 2\eps) 
    &\leq \inf_{t > 0} \Eevent[\exp(t\fnorm)] / \exp(t(\alpha+ 2\eps)) \\ 
    &\leq |\cover_{\eps,\Fset}| \inf_{t>0} \exp(\frac{t^2\subg^2 a^2}{2} -t\alpha) 
    = |\cover_{\eps,\Fset}|\exp(\frac{-\alpha^2}{2\subg^2a^2}).
\end{talign}
Finally, choosing $\alpha = \subg a \sqrt{2 \log({|\cover_{\eps,\Fset}|}{/\delta'})}$ yields the desired claim.
\end{proof}

Now fix any $\eps \geq 0$, $\delta' \in (0,1)$, and kernel $\k$ that generates $\K$, 
and consider the subset $\Fset = \{\pm\k(\x_i,\cdot) : i\in\ind\}$.
Since $\indnorm = \fnorm$ and $\sup_{f\in\Fset}\knorm{f} = D_{\ind}$, \cref{discrepancy-cover-bound} implies that, on the event $\event$, 
\begin{talign}\label{eq:znorm-eps-cover}
\indnorm
	\leq 
 2\eps + 
\subg D_{\ind}  \sqrt{2 \log({|\cover_{\eps,\Fset}|}{/\delta'})}
\qtext{with probability at least $1-\delta'$.}
\end{talign}
Since $\P(\event^c) \leq \delta/2$ and $|\Fset| \leq 2 |\Zset|$, we use the estimate $|\cover_{0,\Fset}|\leq 2|\ind|$ with $\eps=0$ to obtain the  advertised bound \cref{eq:ind_bound}.

\subsection{Proof of Lipschitz kernel max seminorm bound~\cref{eq:ind_bound_lipschitz}}\label{proof:ind_bound_lipschitz}
Introduce the query point set $\Zset\defeq\{ \x_i : i \in \ind\}$, fix any $\delta'\in(0,1)$ and $\z_0\in\Zset$, and define the symmetrized seminorm 
\begin{talign}
\zznorm 
    \defeq 
\sup_{\z,\z'\in\Zset}|(\Pin-\Qout)\kernel(\z) - (\Pin-\Qout)\kernel(\z')|.
\end{talign}
By the triangle inequality and the derivation of \cref{proof:ind_bound}, we have, on the event $\event$, 
\begin{talign}
\indnorm
    &\leq 
\zznorm 
    +
|(\Pin -\Qout)\kernel(\z_0)|
\\
    &\leq 
\zznorm
    +
\subg \sqrt{\k(\z_0,\z_0)} \sqrt{2 \log({4}{/\delta'})}
    \qtext{with probability at least $1-\delta'/2$.} \label{max_seminorm_lipschitz_bound_part_1}
\end{talign}
Since $\P(\event^c)\leq \delta/2$, it only remains to upper bound $\zznorm$ on $\event$ with probability at least $1-\delta'/2$.

To this end, we first establish that $((\Pin -\Qout)\kernel(\z))_{\z\in\Zset}$ is a sub-Gaussian process on $\event$ with respect to a particular bounded-\Holder metric $\rho$.
\begin{definition}[\tbf{Sub-Gaussian process on an event}]
We say an indexed collection of random variables $(X_\theta)_{\theta\in\Theta}$ is a \emph{sub-Gaussian process with respect to $\rho$ on an event $\event$} if $\rho$ is a metric on $\Theta$ and 
\begin{talign}
\Eevent\Big[\exp\big(\frac{(X_\theta - X_\theta')^2}{\rho(\theta,\theta')^2}\big)\Big]\leq 2
    \qtext{for all}
\theta,\theta'\in\Theta.
\end{talign}
\end{definition}
\begin{lemma}[\tbf{Bounded-\Holder sub-Gaussian process}]\label{bounded_holder_subg}
Consider a kernel $\k$ on $\Xset=\reals^d$ satisfying $|\k(\z,\x) - \k(\z',\x)| \leq L_\k \twonorm{\z-\z'}$ for all $\z,\z'\in\Zset \subset \Xset$ and $\x\in\xin$.
If $(\Pin - \Qout)\kernel$ is $(\kernel,\subg)$-sub-Gaussian on an event $\event$ (\cref{def:function-subg}), then $((\Pin -\Qout)\kernel(\z))_{\z\in\Zset}$ is a sub-Gaussian process on $\event$ with respect to the metric 
\begin{talign}\label{eq:metric}
\rho(\z,\z') 
    \defeq
\subg \sqrt{8/3} \min(2 \sup_{\z\in\Zset} \sqrt{\k(\z,\z)}, \sqrt{2 L_{\k} \twonorm{\z-\z'}} ).
\end{talign}
\end{lemma}
The proof of \cref{bounded_holder_subg} can be found in \cref{proof:bounded_holder_subg}.
Our next lemma, a slight modification of \citet[Thm.~5.36]{wainwright2019high}, bounds the suprema of symmetrized sub-Gaussian processes on an event in terms of covering numbers.

\begin{lemma}[\tbf{Sub-Gaussian process tails}]\label{subgauss_process_tails}
Suppose $(X_{\theta})_{\theta\in\Theta}$
is a sub-Gaussian process with respect to $\rho$ on an event $\event$, and define the diameter 
$\diam(\Theta, \rho) \defeq \sup_{\theta,\theta'\in\Theta}\rho(\theta,\theta')$,
the covering number
\begin{talign}
\coveringnumber(u; \Theta, \rho)\defeq \min\{|\cover_u| :  \cover_u\subseteq\Theta,  \max_{\theta\in\Theta}\min_{\theta'\in\cover_u}\rho(\theta,\theta')\leq u\}
    \qtext{for all $u>0$,}
\end{talign}
and the entropy integral 
$\mcJ(\Theta, \rho) \defeq \int_0^{\diam(\Theta,\rho)}\sqrt{\log(1+\coveringnumber(u; \Theta, \rho))}\,du$.
Then, 
\begin{talign}
\Pevent(\sup_{\theta,\theta'\in\Theta} |X_{\theta}-X_{\theta'}| 
    \geq 
8 (\mcJ(\Theta, \rho)+t))
    \leq
2\exp(-t^2/\diam(\Theta,\rho)^2)
    \qtext{for all}
t > 0.
\end{talign}
\end{lemma}
\begin{proof}
Since $\sqrt{\log(1+xy)} \leq \sqrt{\log((1+x)(1+y))} \leq \sqrt{\log(1+x)} + \sqrt{\log(1+y)}$ for all $x,y>0$, the proof is identical to that of \citet[Thm.~5.36]{wainwright2019high} with $c_1 = 8$ and $(\Eevent, \Pevent)$ substituted for $(\E,\P)$.
\end{proof}

Our final lemma bounds the diameter, covering numbers, and entropy integral of $\Zset$ %
using the metric $\rho$.

\begin{lemma}[\tbf{Covering properties of bounded-\Holder metric}]\label{bounded_holder_covering}
Consider the bounded-\Holder metric $\rho$ \cref{eq:metric} for a kernel $\k$ on $\xset=\reals^d$ and a finite set $\Zset\subset\xset$. 
If $\Z$ is a matrix with one row corresponding to each element of $\Zset$, $r=\rank{\Z}$, and $R = \max_{\z\in\Zset} \twonorm{\z}$, then, in the notation of \cref{subgauss_process_tails}, 
\begin{talign}
\coveringnumber(u;\Zset,\rho)
    &\leq
(1+c^2/{u^2})^r
    \qtext{for}
c \defeq \subg\sqrt{\frac{32}{3}R L_{\k}}     \qtext{and all} 
u > 0, \label{eq:bounded-holder-covering-number}\\
\diam(\Zset,\rho) 
    &\leq 
D \defeq \min(c, \subg \sqrt{\frac{32}{3}} \max_{\z\in\Zset} \sqrt{\k(\z,\z)}),
    \qtext{and} \label{eq:bounded-holder-diameter}\\
\mcJ(\Zset,\rho)
    &\leq
D\sqrt{2r\log(\sqrt{3}ec/D)}.
\end{talign}
\end{lemma}
\begin{proof}
The diameter bound \cref{eq:bounded-holder-diameter} follows directly from the definition of $\rho$ \cref{eq:metric} and the fact $\max_{\z,\z'\in\Zset}\twonorm{\z-\z'}\leq 2R$.

To establish the covering number bound \cref{eq:bounded-holder-covering-number}, we let $\U \Sig \V^\top$ be a compact singular value decomposition of $\Z$ so that 
\begin{talign}
\V\in\reals^{d\times r},
    \quad
\Z=\Z\V\V^\top,
    \qtext{and}
\max_{\z\in\Zset}\twonorm{\V^\top\z}=\max_{\z\in\Zset}\twonorm{\z} = R.
\end{talign}
Fix any $\eps > 0$, and let $\cover$ and $\extcover$
be a sets of minimum cardinality satisfying 
\begin{talign}
\cover 
    &\subset 
\ball^r(R),
    \qquad\quad \ \ \ 
\max_{\bv\in \ball^r(R)}\min_{\bv'\in\cover}\twonorm{\bv' - \bv} \leq \eps^2/2, \\
\extcover 
    &\subset
\ball^d(R),
    \qtext{and}
\max_{\z\in\Zset}\min_{\bz'\in\extcover}\twonorm{\z' - \z} \leq \eps^2/2. \label{eq:extcover-criteria}
\end{talign}
Since $\V^\top \z \in \ball^r(R)$ for each $\z\in\Zset$ and $\V\bv' \in\ball^d$ for each $\bv'\in\ball^r$, we have
\begin{talign}
\max_{\z\in\Zset}
\min_{\bv'\in\cover}
\twonorm{\V\bv' - \z}
    &=
\max_{\z\in\Zset}
\min_{\bv'\in\cover}
\twonorm{\V(\bv' - \V^\top\z)}\\
    &=
\max_{\z\in\Zset}
\min_{\bv'\in\cover}
\twonorm{\bv' - \V^\top\z}
    \leq
\eps^2/2,
\end{talign}
so that $\V\cover$ satisfies the criteria of \cref{eq:extcover-criteria}.
Since $|\V\cover|\leq |\cover| \leq (1+4R/\eps^2)^r$ by  \citet[Lem.~5.2]{wainwright2019high}, we must also have 
$|\extcover| 
    \leq
(1+4R/\eps^2)^r$. 

Now, since $\extcover$ has minimum cardinality amongst sets satisfying \cref{eq:extcover-criteria}, for each $\z'\in\extcover$, there is some $\z\in\Zset$ satisfying $\twonorm{\z' - \z}\leq \eps^2/2$ (or else $\z'$ would be superfluous).
Hence, 
there exists a set $\intcover\subseteq\Zset$ satisfying 
\begin{talign}
|\intcover|
    \leq
|\extcover|
    \leq 
(1+4R/\eps^2)^r
    \qtext{and}
\max_{\z\in\Zset}\min_{\bz'\in\intcover}\twonorm{\z' - \z} \leq \eps^2. 
\end{talign}
Moreover, by our metric definition \cref{eq:metric}, 
\begin{talign}
\max_{\z\in\Zset}\min_{\z'\in\intcover}\rho(\z,\z')
    \leq
\frac{c}{2\sqrt{R}}\max_{\z\in\Zset}\min_{\z'\in\intcover}\sqrt{\twonorm{\z-\z'}}
    \leq
\frac{c\eps}{2\sqrt{R}}.
\end{talign}
Hence, for $u=\frac{c\eps}{2\sqrt{R}}$, $\coveringnumber(u;\Zset,\rho)\leq |\intcover| \leq (1+c^2/u^2)^r$.
Since $\eps > 0$ was arbitrary, we have established \cref{eq:bounded-holder-covering-number}.

Finally, we bound the entropy integral using the inequality
$1\leq c^2/u^2$ for $u\in[0,D]$, the concavity of the square-root function, and 
Jensen's inequality:
\begin{talign}
\mcJ(\Zset, \rho)
    &\leq
\int_0^D \sqrt{\log(1+(1+c^2/u^2)^r)} \,du
    \leq
\int_0^D \sqrt{\log((3c^2/u^2)^r)} \,du 
    =
\int_0^D \sqrt{2r\log(\sqrt{3}c/u)} \,du \\
    &\leq
D\sqrt{\frac{1}{D}\int_0^D 2r\log(\sqrt{3}c/u) \,du} 
    =
D\sqrt{2r\log(\sqrt{3}ec/D)}.
\end{talign}
\end{proof}
Together, \cref{bounded_holder_subg,,subgauss_process_tails,,bounded_holder_covering} imply that, in the notation of \cref{bounded_holder_covering},
\begin{talign}
\zznorm  
    \leq
8D \sqrt{2r\log(\sqrt{3}ec/D)} + 8 D \sqrt{\log(4/\delta')}
\end{talign}
on $\event$ 
with probability at least $1-\delta'/2$.
Combining this bound with the inequality \cref{max_seminorm_lipschitz_bound_part_1} yields the result.

\subsection{\pcref{bounded_holder_subg}}\label{proof:bounded_holder_subg}
Define $X_{\z} = (\Pin -\Qout)\kernel(\z)$ for each $\z\in\Zset$, and fix any $\z,\z'\in\Zset$.
Our sub-Gaussianity assumption implies 
\begin{talign}
\Eevent[\exp(\lam (X_{\z}-X_{\z'})]
    \leq
\exp(\frac{\subg^2\lam^2}{2}\knorm{\k(\z,\cdot)-\k(\z',\cdot)}^2)
    \qtext{for all}
\lam\in\reals.
\end{talign}
Moreover, by our Lipschitz assumption, 
\begin{talign}
\knorm{\k(\z,\cdot)-\k(\z',\cdot)}^2
    =
\k(\z,\z) - \k(\z,\z') + \k(\z',\z') - \k(\z',\z)
    \leq
\min(4 \max_{\z\in\Zset}\k(\z,\z),
    2L_{\k}\twonorm{\z-\z'}).
\end{talign}
Finally, 
\cref{sqd_exp_moment_bound} shows that
$\Eevent[\exp(\frac{(X_{\z}-X_{\z'})^2}{\rho(\z,\z')^2}]\le2$ so that $(X_{\z})_{\z\in\Zset}$ is a sub-Gaussian process on $\event$ with respect to $\rho$.

\begin{lemma}[\tbf{Squared exponential moment bound}]\label{sqd_exp_moment_bound}
If $\Eevent[\exp(\lam X)] \leq \exp(\frac{\subg^2\lam^2}{2})$ for all $\lam\in\reals$, then $\Eevent[\exp(\frac{3X^2}{8\nu^2})] \leq 2$.
\end{lemma}
\begin{proof}
The proof is identical to that in \citet[Sec.~2.4]{wainwright2019high} with $\Eevent$ substituted for $\E$.
\end{proof}

%% file: appendices/proof_of_gaussian_kernel.tex
\section{\pcref{cor:gaussian_mmd}}\label{proof:gaussian_mmd}
\newcommand{\bconstant}{b} %

\cref{cor:gaussian_mmd} follows immediately from the following explicit, non-asymptotic bound. 

\begin{corollary}[\tbf{Detailed Gaussian MMD of \kh}]\label{cor:gaussian_mmd_detailed}
If $\xin \subset \ball^d (R)$ for $R>0$, then $\kh(\delta)$ with $\kernel=\textsc{Gauss}(\eta)$, $n=\nin \geq (2e)^d$, and $\bconstant \defeq \frac{1}{2}$ 
delivers
\begin{talign}\label{eq:gaussian_mmd_ball_detailed}
\mmd_{\mkernel}^2(\pin,\qout)  
    &\leq 
\frac{1}{\nout^2} \log(\frac{4\nout}{\delta}) \brackets{ e^2 \max\braces{\brackets{ \frac{2e}{d} \log\parenth{ \nin \nout \bconstant}}^d, (\frac{R^2\eta e^3 4}{d})^{d}} + e\log(\frac{1}{\delta'})} + \frac{1}{\nout\bconstant}(\frac{1}{\nout}-\frac{1}{\nin})
\end{talign}
with probability at least $1-\delta/2-\delta'$.
\end{corollary}
\begin{proof}
Consider the approximate rank parameter 
\begin{talign}\label{eq:b}
    r^\star &\defeq \max\braces{\brackets{ \frac{2e}{d} \log\parenth{ \nin \nout \bconstant}}^d, (R^2\eta e^3 4 /d)^{d}}.
\end{talign}
The assumption $\nin \geq (2e)^d$ and the fact that 
$b \geq 1/(2^d \nout)$ 
ensure that $\log(\nin\nout b) \geq d + \log(\nout b/2^d) \geq d$ and therefore that $r^\star\geq (2e)^d$.  
Hence, by \citet[Thm.~3]{altschuler2019massivelyscalablesinkhorndistances}, 
the $(r^\star+1)$-th eigenvalue of $\K$ satisfies
\begin{talign}
    \lambda_{r^\star+1} &\leq \nin \exp\braces{-\frac{d}{2e} \max\braces{\frac{2e}{d}\log(\nin\nout b),(R^2 \eta e^3 4 /d)}\log\parenth{\frac{d \max\braces{\frac{2e}{d}\log(\nin\nout b),(R^2 \eta e^3 4 /d)}}{4 e^2 \eta R^2}}} \\
    &\leq \nin \exp\braces{ - \log(\nin\nout b) \log(e) }
    \leq \nin \parenth{ \frac{1}{\nin \nout \bconstant}} = \frac{1}{\nout \bconstant}.
\end{talign}
Since $\maxnorm{\K} =1$ and $\khd \in \ksubg$ with $\subg$ defined in \cref{khd-sub-gaussian}, the result now follows  from  
\cref{thm:subg_low_rank_gen_kernel}.
\end{proof}

\section{\pcref{cor:gaussian_mmd_manifold}}
\label{proof:gaussian_mmd_manifold}

\begin{assumption}[\tbf{$d^\star$-manifold with $Q$-smooth atlas {\citep[Assum.~1]{altschuler2019massivelyscalablesinkhorndistances}}}]\label{assum:manifold}
Let $\Omega \subset \R^d$ be a smooth compact manifold without boundary of dimension $d^\star < d$. Let $(\mbf \Psi_j,U_j)_{j\in [T]}$ for $T\in \naturals$ be an atlas for $\Omega$, where $(U_j)_j$ are open sets covering $\Omega$ and $\mbf \Psi_j: U_j \mapsto \ball^{d^\star}(r_j)$ are smooth maps with smooth inverses, mapping $U_j$ bijectively to $\ball^{d^\star}(r_j)$. Assume that there exists $Q>0$ such that $\sup_{u\in \ball^{d^\star}(r_j)} \norm{D^{\boldsymbol \alpha} \mbf \Psi_j\inv (u)} \leq Q^{\abss{\boldsymbol \alpha}}$ for all $\boldsymbol \alpha\in \naturals^{d^\star}$ and $j\in [T]$, where $\abss{\boldsymbol \alpha} \defeq \sum_{j=1}^{d^\star} \alpha_j$ and $D^{\boldsymbol \alpha} = \frac{\partial^{\abss{\boldsymbol \alpha}}}{\partial u_1^{\alpha_1}\ldots \partial u_{d^\star}^{\alpha_{d^\star}}}$ for $\boldsymbol \alpha\in \naturals^{d^\star}$.
\end{assumption}

\cref{cor:gaussian_mmd_manifold} follows immediately from the following more detailed result.

\begin{corollary}[\tbf{Detailed Intrinsic Gaussian MMD of \kh}]\label{cor:gaussian_mmd_manifold_detailed}
Suppose $\xin$ lies on a manifold $ \Omega \subset \ball^d$ satisfying \cref{assum:manifold}. 
Then $\khd$ with $\kernel=\textsc{Gauss}(\eta)$ and $n=\nin$ delivers
\begin{talign}\label{eq:gaussian_mmd_manifold_detailed}
\mmd_{\mkernel}^2(\pin,\qout)
    \leq 
\frac{1}{\nout^2} \log(\frac{4\nout}{\delta})\big({\frac{e^2}{ c^{5 d^\star/2}} \log^{\frac{5 d^\star}{2}} \parenth{ \nin \nout }+e\log(\frac{1}{\delta'})}\big) + \frac{1}{\nout}(\frac{1}{\nout}-\frac{1}{\nin})
\end{talign}
with probability at least $1-
\frac{\delta}{2}-\delta'$ for $c$ independent of $\xin$.
\end{corollary}
\begin{proof}
\citet[Thm.~4]{altschuler2019massivelyscalablesinkhorndistances} showed that the $(r+1)$-th eigenvalue of $\K$ satisfies \cref{eq:gsn_lambda_manifold} 
for a constant $c$ independent of $\xset=\xin$. 
Since $\maxnorm{\K} =1$ and $\khd \in \ksubg$ with $\subg$ defined in \cref{khd-sub-gaussian}, the result now follows  from  
\cref{thm:subg_low_rank_gen_kernel} with 
$r 
    = 
\parenth{  \log\parenth{\nin \nout}/c}^{5 d^\star / 2}.$
\end{proof}

%% file: appendices/proof_of_att_err.tex
\section{\pcref{att-err}}\label{proof:att-err}
Throughout we will make use of the convenient representation
\begin{talign}
\That
    =
\hDinv\ha\V
    \sstext{for}
\indout 
    &\defeq
\{ i \in [n] : (\augkey_i,\augval_i)\in\xout\},
    \sstext{}
\ha 
    \defeq 
\frac{n}{\nout}(\exp(\frac{\inner{\query_i}{\key_j}}{\sqrt{d}})\indic{j\in\indout})_{i,j=1}^n,
    \sstext{and}
\hD
    \defeq
\ha\boldone_n. 
\label{eq:ahat}
\end{talign}

Our proof makes use of three lemmas.
The first, proved in \cref{proof:att-err-decomposition}, bounds the approximation error for the attention matrix $\T$ in terms of the approximation error for $\A\V$ and $\A\boldone_n$.

\begin{lemma}[\tbf{Decomposing attention approximation error}]\label{att-err-decomposition}
In the notation of \cref{algo:thinformer,eq:ahat},
\begin{talign}
\maxnorm{\hdav-\dav}
    \leq
\min\big(\maxnorm{(\frac{1}{n}\D)^{-1}}, \maxnorm{(\frac{1}{n}\hD)^{-1}}\big)
(\frac{1}{n}\maxnorm{\hav - \av}
    +
\frac{1}{n}\infnorm{\A\boldone_n-\ha\boldone_n}
\maxnorm{\V}).
\end{talign}
\end{lemma}

The second, proved in \cref{proof:kms-bounds-att}, bounds the approximation error for $\A\V$ and $\A\boldone_n$ in terms of the KMS \cref{eq:kms} for a specific choice of attention kernel matrix.

\begin{lemma}[\tbf{KMS bound on attention approximation error}]\label{kms-bounds-att}
Instantiate the notation of \cref{algo:thinformer,eq:ahat} and define the query set
\begin{talign}
\xp 
    \defeq 
\{\x_{i+nj}\defeq (\augquery_i, \e_{j}^{d+1}) : i \in [n], j \in [d+1]\}
    \qtext{where}
\augquery_i 
    \defeq
\query_i/d^{\quarter}
\end{talign} 
and $\e_{j}^{d+1}$ is the $j$-th standard basis vector in $\reals^{d+1}$.
If $\Katt \defeq \katt(\xset,\xset)$ for $\xset \defeq \xp\cup\xin$, then
\begin{talign}
\max\big(\frac{1}{n} \maxnorm{(\hat{\mbf A}- \mbf A)\V},
\frac{1}{n} \infnorm{(\hat{\mbf A}- \mbf A)\boldone_n}\maxnorm{\V}\big)
    =
\attindnorm
    \qtext{for}
\ind \defeq [n(d+1)].
\end{talign}
\end{lemma}

Our third lemma, proved in \cref{proof:att-parameters}, bounds the size of key parameters of the thinned attention problem.
\begin{lemma}[\tbf{Thinned attention problem parameters}]\label{att-parameters}
Instantiate the notation of \cref{kms-bounds-att}, and define 
$R\defeq \max_{i\in[n]}\max(\twonorm{\query_i},\twonorm{\key_i})$.
Then, for all $i,j\in\ind$ and $l\in\supp{\pin}$, 
\begin{talign}
&\maxnorm{(\frac{1}{n}\D)^{-1}}
    \leq 
\exp(\frac{R^2}{\sqrt{d}}),
    \quad 
\max_{\x\in\xin}\sqrt{\katt(\x,\x)}
    \leq
\exp(\frac{R^2}{2\sqrt{d}})\sqrt{\rownorm{\V}^2+\maxnorm{\V}^2},
\\
&R_\ind 
    \defeq
\max_{i\in\ind}\twonorm{\x_i}
    \leq 
\sqrt{\frac{R^2}{\sqrt{d}}+1},
    \quad
D_{\ind} 
    \defeq
\max_{i\in\ind}\sqrt{\Katt[,ii]}
    \leq
\exp(\frac{R^2}{2\sqrt{d}}),
\\
&\rank{\X_\ind} 
    \leq
d+1
    \qtext{for}
\X_\ind
    \defeq
[\x_i]_{i\in\ind}^\top,
    \qtext{and}\\
&|\Katt[,il] -\Katt[,jl]| \leq L_{\Katt} \twonorm{\x_i-\x_j} 
    \qtext{for} 
L_{\Katt}
    \defeq
\exp(\frac{R^2}{\sqrt{d}})\sqrt{\frac{R^2}{\sqrt{d}} + 2}\maxnorm{\V}.
\end{talign}
\end{lemma}

Now instantiate the notation of \cref{kms-bounds-att}, and define the coefficient 
\begin{talign}
c
    \defeq
2\sqrt{2}\left(32\sqrt{\frac{2}{3}\,(d+1)\log(3e^2(\frac{R^2}{\sqrt{d}} + 2)\maxnorm{\V})}
    +
\sqrt{2\log(8)}(1+\frac{32}{\sqrt{3}})\right).
\end{talign}
Together,  
\cref{att-parameters}, the KMS quality bound of \cref{thm:subg_low_rank_gen_kernel}, and the \khcompresshalf sub-Gaussian constant $\subg$ of \cref{khcompressd-sub-gaussian} imply that, with probability at least $\half$, 
\begin{talign}
&\attindnorm
    \leq 
\frac{c}{2\sqrt{2}}\exp(\frac{R^2}{\sqrt{d}})\sqrt{\rownorm{\V}^2+\maxnorm{\V}^2}\frac{\sqrt{\log_2(\nout)\log({8\nout \log_2\frac{\nin}{\nout}})}}{\nout}.
\end{talign}
Hence, by \cref{att-err-decomposition,kms-bounds-att}, with probability at least $\half$,
\begin{talign}
\maxnorm{\hdav-\dav}
    &\leq 
\frac{c}{\sqrt{2}}\exp(\frac{2R^2}{\sqrt{d}})\sqrt{\rownorm{\V}^2+\maxnorm{\V}^2}\frac{\sqrt{\log_2(\nout)\log({8\nout \log_2\frac{\nin}{\nout}})}}{\nout} \\
    &\leq 
c\exp(\frac{2R^2}{\sqrt{d}})\rownorm{\V}\frac{\sqrt{\log_2(\nout)\log({8\nout \log_2\frac{\nin}{\nout}})}}{\nout}.
\end{talign}

\subsection{\pcref{att-err-decomposition}}\label{proof:att-err-decomposition}
By the triangle inequality, we have
\begin{talign}
\maxnorm{\hdav-\dav}
    \leq
\maxnorm{\hdav-\hddav} 
    +
\maxnorm{\hddav-\dav}.
\end{talign}
We bound the first term on the right-hand side using the submultiplicativity of the max norm under diagonal rescaling:
\begin{talign}
\maxnorm{\hdav-\hddav} 
    \leq
\maxnorm{\hD^{-1}}\maxnorm{\hav - \av}
    =
\maxnorm{(\frac{1}{n}\hD)^{-1}}\frac{1}{n}\maxnorm{\hav - \av}.
\end{talign}
To bound the second term we use the same submultiplicativity property and the fact that each entry of $\dav$ is the average of values in $\V$:
\begin{talign}
\maxnorm{\hddav-\dav}
    &=
\maxnorm{\hDinv(\D-\hD)\dav}
    \leq
\maxnorm{\hDinv}
\maxnorm{\D-\hD}
\maxnorm{\dav} \\
    &\leq
\maxnorm{(\frac{1}{n}\hD)^{-1}}
\frac{1}{n}\infnorm{\A\boldone_n-\ha\boldone_n}
\maxnorm{\V}.
\end{talign}
An identical argument reversing the roles of $(\D,\A)$ and $(\hD,\ha)$ yields the second bound.
\subsection{\pcref{kms-bounds-att}}\label{proof:kms-bounds-att}
Define the augmented value matrix $\augV=[\V,\maxnorm{\V}\boldone_n]\in\reals^{d+1}$.
By the definition of $\Katt$ and $\ha$,
\begin{talign}
\attindnorm
    =
\max_{i\in[n],j\in[d+1]}
|\sum_{\ell\in[n]}
\A_{i\ell} \augV_{\ell j}
(\pin-\qout)_\ell|
    =
\frac{1}{n}\infnorm{(\A-\ha)\augV\e_j^d}
    =
\frac{1}{n}\maxnorm{(\A-\ha)\augV}.
\end{talign}

\subsection{\pcref{att-parameters}}\label{proof:att-parameters}
First, by the Cauchy-Schwarz inequality and the nonnegativity of $\D=\A\boldone_n$ we have
\begin{talign}
\maxnorm{(\frac{1}{n}\D)^{-1}}
    =
\frac{1}{\min_{i\in[n]}\frac{1}{n}\sum_{j\in[n]}\A_{ij}}
    \leq
\frac{1}{\min_{i\in[n],j\in[n]}\exp(\frac{\inner{\query_i}{\key_j}}{\sqrt{d}})}
    \leq
\frac{1}{\min_{i\in[n],j\in[n]}\exp(\frac{-\twonorm{\query_i}\twonorm{\key_j}}{\sqrt{d}})}
    \leq
\exp(\frac{R^2}{\sqrt{d}}).
\end{talign}
Second, the  $\max_{\x\in\xin}\sqrt{\katt(\x,\x)}$ inequality follows as
\begin{talign}
\katt((\augkey_i,\augval_i),(\augkey_i,\augval_i))
    =
\exp(\frac{\twonorm{\key_i}^2}{\sqrt{d}})(\twonorm{\val_i}^2+\maxnorm{\V}^2)
    \leq
\exp(\frac{R^2}{\sqrt{d}})(\rownorm{\V}^2+\maxnorm{\V}^2).
\end{talign}
Third, the $R_\ind$ inequality follows as 
\begin{talign}
\twonorm{(\augquery_i,\e_j^{d+1})}
    =
\sqrt{\twonorm{\augquery_i}^2+1}
    \leq
\sqrt{\frac{R^2}{\sqrt{d}}+1}
    \qtext{for all}
i\in[n],j\in[d+1].
\end{talign}
Fourth, the $D_\ind$ inequality follows as 
\begin{talign}
\max_{i\in\ind}\Katt[,ii]
    =
\max_{i\in[n]}\exp(\frac{\twonorm{\query_i}^2}{\sqrt{d}})
    \leq
\exp(\frac{R^2}{\sqrt{d}}).
\end{talign}
Fifth, the rank inequality follows as $\x_i\in\reals^{d+1}$ for $i\in\ind$.
Finally, the Lipschitz inequality follows as, for any $i,k,l\in[n]$ and $j,m\in[d+1]$,
\begin{talign}
&|\exp(\frac{\inner{\query_i}{\key_l}}{\sqrt{d}})\inner{\e_j^{d+1}}{\augval_l} -  \exp(\frac{\inner{\query_k}{\key_l}}{\sqrt{d}})\inner{\e_m^{d+1}}{\augval_l}| \\
    &\leq
\exp(\frac{\inner{\query_i}{\key_l}}{\sqrt{d}})|\augval_{lj}-\augval_{lm}|
    +
|\exp(\frac{\inner{\query_i}{\key_l}}{\sqrt{d}})-\exp(\frac{\inner{\query_k}{\key_l}}{\sqrt{d}})||\augval_{lm}| \\
    &\leq 
\exp(\frac{\twonorm{\query_i}\twonorm{\key_l}}{\sqrt{d}})\twonorm{\e_j^{d+1}-\e_m^{d+1}}\frac{|\augval_{lj}-\augval_{lm}|}{\sqrt{2}}
    +
\exp(\frac{\max(\twonorm{\query_i},\twonorm{\query_k})\twonorm{\key_l}}{\sqrt{d}})|\frac{\inner{\query_i-\query_k}{\key_l}}{\sqrt{d}}||\augval_{lm}| \\
    &\leq
\exp(\frac{R^2}{\sqrt{d}})\twonorm{\e_j^{d+1}-\e_m^{d+1}}\frac{|\augval_{lj}-\augval_{lm}|}{\sqrt{2}}
    +
\exp(\frac{R^2}{\sqrt{d}})\frac{\twonorm{\query_i-\query_k} R}{\sqrt{d}}|\augval_{lm}| \\
    &\leq
\exp(\frac{R^2}{\sqrt{d}})\twonorm{\e_j^{d+1}-\e_m^{d+1}}\sqrt{2}\maxnorm{\V}
    +
\exp(\frac{R^2}{\sqrt{d}})\frac{\twonorm{\query_i-\query_k} R}{\sqrt{d}}\maxnorm{\V} \\
    &\leq
\exp(\frac{R^2}{\sqrt{d}})\sqrt{\frac{R^2}{\sqrt{d}} + 2}\maxnorm{\V}\twonorm{(\augquery_i,e_j^{d+1})-(\augquery_k,\e_m^{d+1})} 
\end{talign}
by the triangle inequality, multiple applications of Cauchy-Schwarz, and the mean-value theorem applied to $x\mapsto e^x$.

%% file: appendices/proof_of_l2_balancing.tex
\section{\pcref{thm:convergence}}\label{proof:convergence}

Our proof makes use of three intermediate results.
The first, inspired by \citet[Thm.~10]{harvey2014near} and \citet[Lem.~1]{cooper2023coordinatingdistributedexampleorders}, relates the quality of the ordering produced by \cref{algo:reordering} to the quality of the thinning. 
\begin{lemma}[\tbf{Quality of thinned reordering}]
\label{lem:thinned-reordering-quality}
The output of thinned reordering (\cref{algo:reordering}) satisfies
\begin{talign}
\max_{j\in[n]}\twonorm{\sum_{i=1}^j \x_{\perm_{k+1}(\perm_{k}^{-1}(i))}^k}
    \leq
\half \max_{j\in[n]}\twonorm{\sum_{i=1}^j \x_{i}^k}
    +
\half \max_{j\in[n]}\twonorm{\sum_{i=1}^j \eta_i^k\x_{i}^k}
    +
\twonorm{\sum_{i=1}^n \x_i^k}
\end{talign}
where $\perm_k^{-1}$ is the inverse permutation of $\perm_k$ 
and $\eta_i^k \defeq 2(\indic{\x_i^k \in \xout^k}-1)$.
\end{lemma}
\begin{proof}
Fix any $\jstar \in \argmax_{j\in[n]}\twonorm{\sum_{i=1}^j \x_{\perm_{k+1}(\perm_{k}^{-1}(i))}^k}$. 
If $\jstar \leq n/2$, then 
\begin{talign}
2\twonorm{\sum_{i=1}^{\jstar} \x_{\perm_{k+1}(\perm_{k}^{-1}(i))}^k}
    \leq 
2\max_{j\in[n]}\twonorm{\sum_{i=1}^j \indic{\eta_i^k = 1} \x_{i}^k}
    \leq
\max_{j\in[n]}\twonorm{\sum_{i=1}^j \x_i^k}
    +
\max_{j\in[n]}\twonorm{\sum_{i=1}^j \eta_i^k\x_i^k}
\end{talign}
by the triangle inequality.
Similarly, if $\jstar > n/2$, then,
\begin{talign}
2(\twonorm{\sum_{i=1}^{\jstar} \x_{\perm_{k+1}(\perm_{k}^{-1}(i))}^k}
    -
\twonorm{\sum_{i=1}^n \x_i^k})
    &\leq
2\twonorm{\sum_{i > \jstar} \x_{\perm_{k+1}(\perm_{k}^{-1}(i))}^k}
    \leq 
2\max_{j\in[n]}\twonorm{\sum_{i=1}^j \indic{\eta_i^k = -1} \x_{i}^k} \\
    &\leq
\max_{j\in[n]}\twonorm{\sum_{i=1}^j \x_i^k}
    +
\max_{j\in[n]}\twonorm{-\sum_{i=1}^j \eta_i^k\x_i^k}.
\end{talign}
\end{proof}

The second, a mild adaptation of \citet[Thms.~2 and 3]{cooper2023coordinatingdistributedexampleorders}, bounds the convergence rate of SGD with thinned reordering in terms of the thinning quality.

\begin{theorem}[\tbf{Convergence of SGD with thinned reordering}]
\label{thm:onlinedm}
Suppose that, for all $i\in[n]$ and $\w,\v\in\reals^d$, 
\begin{talign}
&\twonorm{\grad f_i(\w ) - \grad f(\w)}^2    
    \le 
\sigma^2 
    \qtext{and} 
\twonorm{\grad f_i(\w ) - \grad f_i(\mbi v ) } 
    \leq 
L \twonorm{\w  - \v}
\end{talign}
and that SGD \cref{eq:sgd} with thinned reordering (\cref{algo:reordering})  satisfies the \emph{prefix discrepancy bound }
\begin{talign}\label{eq:prefix-discrepancy-bound}
\max_{j\in[n]}
\twonorm{\sum_{i=1}^j \eta_i^k\x_i^k} \le 2 \tilde{A} \max_{i\in[n]}\twonorm{\x_i^k - \xbar^k} 
\qtext{for}
\eta_i^k \defeq 2(\indic{\x_i^k \in \xout^k}-1),
\quad
\bar{\x}^k \defeq \frac{1}{n}\sum_{i=1}^n \x_i^k, 
\end{talign}
and each epoch $k\in[K]$. 
Then the step size setting
\begin{talign}
\alpha 
	= 
\min\left\{\frac{1}{16 L (2n + \tilde{A})}, \left(\frac{4 F_1}{42 L^2 \sigma^2\tilde{A}^2 n K + 18L^2  n^3 \sigma^2}\right)^{1/3}\right\} 
    \qtext{with}
F_1
	\defeq 
f(\w_1) - \fstar
    \qtext{and}
\fstar\defeq\inf_{\v\in\Rd} f(\v)
\end{talign}
yields the convergence bound
\begin{talign}
\frac{1}{K}\sum_{k=1}^{K}\norm{ \nabla f(\w_k) }^2  
	&\le 
\frac{9 (F_1 L\sigma \tilde{A})^{2/3}}{(n K)^{2/3}} + \frac{(72 F_1 L\sigma)^{2/3} + 64F_1 L (2 + \tilde{A}/(n))}{K}.
\end{talign}
If, in addition, $f$ satisfies the $\mu$-Polyak-Łojasiewicz (PL) condition, 
\begin{talign}
&\mu(f(\w)-\fstar)
    \leq 
\frac{1}{2} \twonorm{\grad f(\w)}^2
\qtext{for all}
\w \in\Rd, 
\end{talign}
and the number of epochs satisfies
\begin{talign}
K \ge 10 + \frac{1}{\mu}32 L(2+\tilde{A}/n)\tilde{W}%
	\qtext{for}
\tilde{W} 
	\defeq 
W_0(K^2n^2C_3)
	\qtext{and}
C_3 
	\defeq
\frac{(F_1+\sigma^2/L)\mu^2}{224L^2\sigma^2\tilde{A}^2},
\end{talign}
where $W_0$ denotes the Lambert W function, 
then the step size setting
$\alpha 
	= 
\frac{2\tilde{W}}{Kn\mu}$ 
yields the convergence bound 
\begin{talign}
f(\w_{K}) - \fstar 
    \leq
\frac{1}{(n K)^2}\left(\frac{(F_1 + L^2\sigma^2)\tilde{W}}{C_3} + \frac{112L^2\sigma^2\tilde{A}^2{\tilde{W}}^2}{\mu^3}\right).
\end{talign}
\end{theorem}
\begin{proof}
The proof is identical to that of \citet[Thms.~2 and 3]{cooper2023coordinatingdistributedexampleorders} with $m=1$ worker 
once each instance of $\infnorm{\cdot}$ is replaced with $\twonorm{\cdot}$, 
each instance of $L_{2,\infty}$ is replaced with $L$, 
each instance of $T$ is replaced with $K$, 
and 
\cref{lem:thinned-reordering-quality} is substituted for \citet[Lem.~1]{cooper2023coordinatingdistributedexampleorders}. 
\end{proof}

The final result uses \cref{thm:subg_low_rank_gen_kernel} to bound the prefix discrepancy of \khlind. 
\begin{lemma}[\tbf{\khlind prefix discrepancy}]\label{khlind-prefix-discrepancy}
Fix any epoch $k\in[K]$.  
With probability at least $1-\frac{\delta}{2}-\delta'$, 
thinned reordering (\cref{algo:reordering}) with \khlind satisfies 
the \emph{prefix discrepancy bound} \cref{eq:prefix-discrepancy-bound} with 
\begin{talign}
\tilde{A} 
    =
\sqrt{\log (\frac{2n(\log(n/2)+1)}{\delta})
\brackets{e^2\,\epsrank[\eps_k](\X^k)+10e\log(\frac{n}{\delta'})}}
\end{talign}
for
$\X^k \defeq [\x_1^k,\dots,\x_n^k]^\top$,
$\eps_k \defeq \max_{i\in[n]}\ \textstyle\frac{\sqrt{9 e \log(2n\log(e n/2) /\delta ) \log(n/\delta') }\twonorm{\x_i^k-\xbar^k}}{\sqrt{n}},$
and
$\xbar^k \defeq \frac{1}{n}\sum_{i=1}^n\x_i^k.$
\end{lemma}
\begin{proof}
Define 
$\xset^k = \{\x_1^k,\dots,\x_n^k\}$,
$c = 2\max_{i\in[n]}\twonorm{\x_i^k - \xbar^k}$,
and $r = \epsrank[\eps_k](\X^k)$.
For any $j\in[n]$, we can write
\begin{talign}
\twonorm{\sum_{i=1}^j \eta_i^k\x_i^k}
	=
\twonorm{\sum_{i=1}^j \x_i^k - \sum_{i=1}^j \indic{\x_i^k\in\xoutj^k}\x_i^k}
	=
2j\twonorm{(\X^k)^\top(\pin^j - \qout^j)}
	=
2j\mmd_{\X^k(\X^k)^\top}(\pin^j, \qout^j)  
\end{talign}
where $\pin^j$ and $\qout^j$ are the empirical distributions over $\xinj^k = (\x_i^k)_{i=1}^j$ and $\xoutj^k = \{ \x_i^k \in \xout^k : i \in [j]\}$.

Since \khlind is an online algorithm that assigns signs $(\eta_i^k,\eta_{i+1}^k=1-\eta_i^k)$ to the points $(\x_i^k,\x_{i+1}^k)$ sequentially, we can view $\xoutj^k$ as the output of \khlind applied to $\xinj^k$ with $\nout= \frac{j}{2}$ and the linear kernel $\kernel(\x,\y)=\inner{\x}{\y}$ for each $j\in[n]$.  Therefore, we may invoke the established \khlind sub-Gaussian constants $\subg_j$ of \cref{khlind-sub-gaussian}, \cref{thm:subg_low_rank_gen_kernel}, the union bound, and the definition of $\eps$-rank (\cref{def:eps-rank}) to deduce that 
\begin{talign}
\max_{j\in[n]} \twonorm{\sum_{i=1}^j \eta_i^k\x_i^k}^2
    &\leq 
\max_{j\in[n]} 4j^2\subg_j^2 \brackets{e^2r+e\log(\frac{n}{\delta'})}+ \sig_{r+1}(\X^k)^2\frac{4j^2}{j} 
    \leq 
c^2 \tilde{A}^2
\end{talign}
with probability at least $1-\frac{\delta}{2}-\delta'$.
\end{proof}

\cref{thm:convergence} now follows directly from \cref{thm:onlinedm} and \cref{khlind-prefix-discrepancy} applied to \khsgd with $\delta'=\frac{1}{4K}$ and a union bound over epochs.

%% file: appendices/ktcompress.tex
\section{\ktcompressd}
\label{app:ktcompress}

We describe the thinning algorithm \ktcompressd used in \cref{algo:ctt}. We use \khd for every halving round except for the last round, which thins a point sequence of size $2 \nout$ to $\nout$. For this final halving round we use  $\textsc{KH-Refine}(\delta)$ (\cref{algo:khrefine}) derived from the \ktswap algorithm of \citet[Alg.~1b]{dwivedi2024kernel}. 
The refinement stage of 
\cref{algo:khrefine} greedily improves the MMD of the initial \khd output. 
Hence, $\mmd_\k(\xin, \xout) \leq \mmd_\k(\xin,\coreset[1])$ with probability 1.

\begin{algorithm2e}[ht!]
\caption{$\textsc{KH-Refine}(\delta)$: \khd with greedy refinement \citep[Alg.~1a]{dwivedi2024kernel}} 
  \label{algo:khrefine}
 \SetAlgoLined\DontPrintSemicolon
\small
{
    \KwIn{point sequence $\xin = (\x_i)_{i = 1}^{\nin}$, kernel $\k$, input size $\nin \in 2 \naturals$}
        \BlankLine
        $\cset \gets \khd (\xin, \k)$; \quad $\nout \defeq \nin / 2$ \\
        \BlankLine
        // Swap out each point in $\xout$ for the best alternative in $\xin$ \\[1pt]
        $\xout \gets \cset.\texttt{copy()}$\\
       \For{$\x\in\cset$}{
            \BlankLine
            $\xout \gets \xout \,\backslash\, \{\x\}\cup\{\argmin_{\x'\in\xin}\mmd_{\kernel}(\Pin, \Qout+\frac{1}{\nout}(\dirac_{\x'}-\dirac_{\x}))$\}
       }
\KwRet{
 $\xout$\textup{, refined coreset of size $\nin / 2$}
}
}
\end{algorithm2e}

%% file: appendices/proof_of_ctt_power.tex
\section{\pcref{thm:ctt_power}}
\label{proof:ctt_power}

\cref{thm:ctt_power} follows from the following more detailed statement, proved in \cref{proof:ctt_power_detailed}, as
\begin{talign}\label{eq:errorhat} \error[\K]^2(\nin,\frac{\wtil\beta}{20 \sblock_n},\ossymb) + \error[\K']^2(\nin,\frac{\wtil\beta}{20 \sblock_n},\ossymb)
=
O(\errorhat^2).
\end{talign}

\begin{theorem}[\tbf{Low-rank analysis of \ctt power, detailed}]\label{thm:ctt_power_detailed}
Under the assumptions of \cref{thm:ctt_power} with $\nin \defeq \frac{m+n}{\sblock}$, 
\ctt (\cref{algo:ctt}) rejects with probability at least $1-\beta$ whenever $c' \mmd_{\kernel}(\distX, \distY)/\sqrt{\log(1/\gamma)}$ exceeds 
\begin{talign}\label{eq:mmd_threshold}
    2 c_{\wtil \beta/(20 \sblock)} \frac{\infnorm{\kernel}^{\half}}{\sqrt m} + \frac{\error(\distX,\nin, \frac{\wtil\beta}{20 \sblock_m},\ossymb) + \error(\distY, \nin,\frac{\wtil \beta}{20 \sblock_n},\ossymb)}{2^{\ossymb} \sqrt m}.
\end{talign}
Here, 
$c'>0$ is a universal constant,  $c_{\delta} \defeq 2 + \sqrt{2\log(\frac{2}{\delta})}$, and $\error(\distX,\nin, \delta, \ossymb)$ and $\error(\distY,\nin, \delta, \ossymb)$ respectively denote the $(1-\frac{\delta}{2})$-th quantiles of $\error[\K](\nin,\delta,\ossymb)$ and $\error[\K'](\nin,\delta,\ossymb)$, where
\begin{talign}\label{eq:Rkxin}
\error[\Ktilde]^2(\nin,\delta,\ossymb) 
    &\defeq  
256 (\log_4 \nin - \ossymb - 1) (\sqrt{\log(\nin+1)} + \sqrt{\log(2/\delta)})^2 \\
    &\qquad \cdot \biggl( \frac{2 \sqrt{\maxnorm{\Ktilde}}}{\sqrt 3} \brackets{\sqrt{e \log(\frac{6\cdot 2^\ossymb \sqrt{\nin} (\log_4\nin-\ossymb)}{\delta})} + \sqrt{\log(\frac{3\nin(\log_4 \nin - \ossymb - 1)}{\delta})}} \\
    &\qquad \qquad + \min_{r\leq 2^{\ossymb+1}\sqrt \nin} \braces{ \frac{2 \sqrt{\maxnorm{\Ktilde}}}{\sqrt 3}\sqrt{e^2 r \log\parenth{\frac{6\cdot 2^\ossymb \sqrt \nin (\log_4 \nin-\ossymb)}{\delta} }} + \sqrt{\lambda_{r+1}(\Ktilde) \cdot 2^{\ossymb-1}\sqrt \nin}} \biggr)^2.
\end{talign}
\end{theorem}

\subsection{\pcref{thm:ctt_power_detailed}}\label{proof:ctt_power_detailed}
Recall the following definition from \citet[Def.~3]{shetty2022distributioncompressionnearlineartime}.
\begin{definition}[\tbf{$\k$-sub-Gaussian thinning algorithm}]
\label{def:k_sub_gsn_thinning_algo}
We say a thinning algorithm \alg (satisfying \cref{def:thinning_algo}) is $\k$-sub-Gaussian on an event $\event$ with shift $a$ and parameter $v$ if 
\begin{talign}\label{eq:k-sub-gsn-thinning-algo}
    \Pevent(\mmd_{\k}(\Pin, \Qout) \geq a + v \sqrt t \mid \xin) \leq e^{-t} \qtext{for all} t\geq 0.
\end{talign}

\end{definition}

Fix $\Ktilde\in\{\K,\K'\}$. To conclude our power result, it suffices,  by \citet[Rmk.~2,~App.~B.1]{domingoenrich2023compresstestpowerfulkernel} and the failure probability setting of \citet[Lem.~11]{domingoenrich2023compresstestpowerfulkernel}, to establish that
\begin{talign}
    \error[\Ktilde]^2(\nin,\delta,\ossymb) &= 256(\log_4 \nin - \ossymb -1) (\ckk(\delta,2^{\ossymb+1}\sqrt \nin) + \mkk(\delta,2^{\ossymb +1}\sqrt \nin) \sqrt{\log(\frac{3 \nin (\log_4 \nin - \ossymb-1)}{\delta})})^2 \\
    &\qquad \cdot(\sqrt{\log(\nin+1)} + \sqrt{\log(2/\delta)})^2,\label{eq:error-inflation-factor}
\end{talign}
for any scalars $\ckk(\delta,2^{\ossymb+1}\sqrt \nin)$ and $\mkk(\delta,2^{\ossymb+1}\sqrt \nin)$ satisfying the property that, on an event of probability at least $1-\delta/2$, every call to $\halve \defeq \kh(\frac{\ell^2}{\nin 4^{\ossymb + 1}(\log_4 \nin - \ossymb)} \delta)$ 
with input size $\ell$ and output size $\ell/2$ is $\k$-sub-Gaussian (\cref{def:k_sub_gsn_thinning_algo}) with shift $a_{\ell,\nin,\Ktilde}$ and parameter $v_{\ell,\nin,\Ktilde}$ satisfying
\begin{talign}\label{eq:parameter-shift}
    a_{\ell,\nin,\Ktilde} = \frac{\ckk(\delta,\ell)}{\ell/2}\qtext{and}
    v_{\ell,\nin,\Ktilde} = \frac{\mkk(\delta,\ell)}{\ell/2} \sqrt{ \log(\frac{12 \nin 4^\ossymb (\log_4 \nin - \ossymb)}{\ell \delta}) }.
\end{talign}
Substituting 
$\mkk( \delta,2^{\ossymb+1}\sqrt \nin) = (2^{\ossymb}\sqrt \nin) v_{2^{\ossymb+1}\sqrt \nin, \nin, \Ktilde} \brackets{\log(\frac{12 \nin 4^\ossymb (\log_4 \nin - \ossymb)}{2^{\ossymb+1}\sqrt \nin \delta})}^{-\half}$ 
and
$\ckk(\delta,2^{\ossymb+1}\sqrt \nin) = (2^{\ossymb}\sqrt \nin) a_{2^{\ossymb+1}\sqrt \nin, \nin, \Ktilde}$ 
into \cref{eq:error-inflation-factor}, we obtain the sufficient condition 
\begin{talign}
\error[\Ktilde]^2(\nin,\delta,\ossymb) 
    &= 
256(\log_4 \nin - \ossymb -1) 
\cdot (2^\ossymb \sqrt \nin)^2 \cdot(\sqrt{\log(\nin+1)} + \sqrt{\log(2/\delta)})^2 \\
&\cdot\parenth{a_{2^{\ossymb+1}\sqrt \nin, \nin, \Ktilde} + v_{2^{\ossymb+1}\sqrt \nin, \nin, \Ktilde} \brackets{\log(\frac{12 \nin 4^\ossymb (\log_4 \nin - \ossymb)}{2^{\ossymb+1}\sqrt \nin \delta})}^{-\half} \sqrt{\log(\frac{3 \nin (\log_4 \nin - \ossymb-1)}{\delta})}}^2. \label{eq:error-inflation-factor-a-v}
\end{talign}

We now identify suitable $a_{\ell,\nin,\Ktilde}$ and $v_{\ell,\nin,\Ktilde}$ with the aid of the following lemma, proved in \cref{proof:K-subg-implies-k-subg}.
\begin{lemma}[\tbf{$(\K,\subg,\delta)$-sub-Gaussian thinning algorithms are $\k$-sub-Gaussian}]
\label{lem:K-subg-implies-k-subg}
Suppose $\alg$ is a $(\K,\subg,\delta)$-sub-Gaussian thinning algorithm, satisfying \cref{def:alg-subg} with an event $\event$ of probability at least $1-\delta/2$. Then $\alg$ is $\k$-sub-Gaussian (\cref{def:k_sub_gsn_thinning_algo}) on $\event$ with shift $a_{\nout,\nin,\K}$ and parameter $v_{\nout,\nin,\K}$ defined as
\begin{talign}
    a_{\nout,\nin,\K} \defeq \subg \sqrt{e} + \min_{r\leq \nin} \braces{\subg \sqrt{e^2 r} + \sqrt{\lambda_{r+1}(\K) (\frac{1}{\nout} - \frac{1}{\nin})}}
    \qtext{and} 
    v_{\nout,\nin,\K} \defeq \subg \sqrt{e}
    .
\end{talign}
\end{lemma}
By \cref{rkhd-sub-gaussian,lem:funct_subg_vector_subg}, $\kh(\frac{\ell^2}{\nin 4^{\ossymb + 1}(\log_4 \nin - \ossymb)} \delta)$ with input size $\ell$ and output size $\ell/2$ is a $(\Ktilde,\subg,\frac{\ell^2}{\nin 4^{\ossymb + 1}(\log_4 \nin - \ossymb)} \delta)$-sub-Gaussian thinning algorithm with 
\begin{talign}
\subg
    \leq
    \frac{2}{(\ell/2) \sqrt 3} \sqrt{\log\parenth{\frac{6 (\ell/2)\log_2(\ell/(\ell/2))}{\delta}\cdot \frac{\nin 4^{\ossymb + 1}(\log_4 \nin - \ossymb)}{\ell^2}} \maxnorm{\Ktilde}} = \frac{2}{(\ell/2) \sqrt 3} \sqrt{\log\parenth{\frac{12 \nin 4^\ossymb (\log_4\nin -\ossymb) }{\ell \delta}} \maxnorm{\Ktilde}}.
\end{talign}
By \cref{lem:K-subg-implies-k-subg}, on an event of probability at least $1-\frac{\ell^2}{2\nin 4^{\ossymb + 1}(\log_4 \nin - \ossymb)} \delta$, $\kh(\frac{\ell^2}{\nin 4^{\ossymb + 1}(\log_4 \nin - \ossymb)} \delta)$ with input size $\ell$ and output size $\ell/2$ is a $\k$-sub-Gaussian thinning algorithm with shift $a_{\ell,\nin,\Ktilde}$ and parameter $v_{\ell,\nin,\Ktilde}$ defined as
\begin{talign}
    a_{\ell,\nin,\Ktilde} &= \frac{2}{(\ell/2) \sqrt 3} \sqrt{\log\parenth{\frac{12 \nin 4^\ossymb (\log_4\nin -\ossymb) }{\ell \delta}} \maxnorm{\Ktilde}} \sqrt{e\log 2} \\
    &\qquad + \min_{r\leq \ell}\braces{\frac{2}{(\ell/2) \sqrt 3} \sqrt{\log\parenth{\frac{12 \nin 4^\ossymb (\log_4\nin -\ossymb) }{\ell \delta}} \maxnorm{\Ktilde}} \sqrt{e^2 r} + \sqrt{\lambda_{r+1}(\Ktilde)(\frac{1}{\ell/2} - \frac{1}{\ell})}} \qtext{and} \label{eq:kt-split-shift} \\
    v_{\ell,\nin,\Ktilde} &= \frac{2}{(\ell/2) \sqrt 3} \sqrt{\log\parenth{\frac{12 \nin 4^\ossymb (\log_4\nin -\ossymb) }{\ell \delta}} \maxnorm{\Ktilde}} \sqrt{e}. \label{eq:kt-split-parameter}
\end{talign}
Moreover, by the union bound, as detailed in \citet[App.~F.1]{shetty2022distributioncompressionnearlineartime}, every \halve call made by \ktcompress is simultaneously $\k$-sub-Gaussian with these input-size-dependent parameters on a common event of probability at least $1-\frac{\delta}{2}$. 
Substituting \cref{eq:kt-split-shift,eq:kt-split-parameter} with $\ell = 2^{\ossymb+1} \sqrt \nin$ into \cref{eq:error-inflation-factor-a-v}, we obtain our error inflation factor expression \cref{eq:Rkxin}, completing the proof.

\subsection{\pcref{lem:K-subg-implies-k-subg}}
\label{proof:K-subg-implies-k-subg}
Fix any $t\geq 0$, and let $\delta'=e^{-t}$. 
By our sub-Gaussian assumption, 
\cref{thm:mmd-kernel-compression} implies that, as advertised, 
\begin{talign}
e^{-t}
    &\geq
\Psubarg{\event}{\mmd_\K^2(\pin, \qout) \geq \min_{r\leq \nin}\subg^2 \brackets{e^2 r + e t} + \lambda_{r+1}(\K)(\frac{1}{\nout} - \frac{1}{\nin})} \\
    &=
\Psubarg{\event}{\mmd_\K(\pin, \qout) \geq \min_{r\leq \nin}\sqrt{\subg^2 \brackets{e^2 r + e t} + \lambda_{r+1}(\K)(\frac{1}{\nout} - \frac{1}{\nin})}} \\
    &\geq
\Psubarg{\event}{\mmd_\K(\pin, \qout) \geq \subg \sqrt{e}\sqrt{t} + \min_{r\leq \nin}\subg\sqrt{ e^2 r} + \sqrt{\lambda_{r+1}(\K)(\frac{1}{\nout} - \frac{1}{\nin})}}.
\end{talign}

%% file: appendices/proof_of_ctt_power_deep_kernel.tex
\newcommand{\dcomb}{d'}
\newcommand{\Rcomb}{R'}
\newcommand{\Kdeep}{\K_{\textup{deep}}}

\section{\pcref{cor:ctt_power_deep_kernel}}\label{proof:ctt_power_deep_kernel}
Define the radius
\begin{talign}\label{eq:R-R-phi}
    \Rcomb \defeq \max_{\y\in \ys\cup\xs} \statictwonorm{(\phi(\y) , \y)},
\end{talign}
the augmented vectors $\ys' \defeq \braces{(\phi(\y),\y)}_{\y\in \ys}$, and the augmented kernel
\begin{talign}
q'((\phi(\x),\x),(\phi(\y),\y))
    \defeq
\kappa(\phi(\x),\phi(\y))
q(\x,\y)
    =
\exp(-\eta\twonorm{(\phi(\x),\x)-(\phi(\y),\y)}^2).
\end{talign}
Since the deep kernel \cref{eq:deep_kernel} takes the form
\begin{talign}
\kdeep(\x,\y) 
    & = (1-\epsilon) q'((\phi(\x),\x),(\phi(\y),\y)) + \epsilon q(\x,\y) 
\end{talign}
we also have
\begin{talign}
\Kdeep \defeq \kdeep(\ys,\ys) = (1-\eps)\bQ' + \eps \bQ
    \qtext{for}
\bQ' \defeq q'(\ys',\ys')
    \qtext{and}
\bQ \defeq q(\ys,\ys).
\end{talign}
Hence, by Weyl's inequality \citep[Thm.~4.3.1]{horn1985matrix} and the Gaussian kernel matrix eigenvalue bound \cref{eq:gsn_lambda_ball},
\begin{talign}
\lambda_{2r+1}(\Kdeep) 
    \leq 
(1-\eps)\lambda_{r+1}(\bQ')+\eps\lambda_{r+1}(\bQ) 
    \leq
n e^{-\frac{d'}{2e} r^{1/d'} \log\parenth{\frac{d' r^{1/d'}}{4 e^2 \eta \Rcomb^2}}}
    \qtext{for} 
(2e)^{d'} \leq r < n.
\end{talign}
Parallel reasoning and the assumption $m\leq n$ yield the same bound for $\lambda_{2r+1}(\kdeep(\xs,\xs))$ and $(2e)^{d'}\leq r < m$.
Now consider the approximate rank parameter 
\begin{talign}\label{eq:approx-rank-deepctt}
    r^\star &\defeq \max\braces{\brackets{ \frac{2e}{d'} \log\parenth{ n \nout \bconstant}}^{d'}, (\Rcomb^2\eta e^3 4 /d')^{d'}}
    \qtext{for}
    b\defeq \half.
\end{talign}
Then, for $n \geq (2e)^{d'}$, we have, 
exactly as in \cref{proof:gaussian_mmd}, 
\begin{talign}
    \lambda_{2r^\star+1}(\Kdeep) + \lambda_{2r^\star+1}(\kdeep(\xs,\xs))
    &\leq \frac{2}{\nout \bconstant}
\end{talign}
and therefore
\begin{talign}
\errorhat 
    = 
O\left(\sqrt{\log(\frac{n}{s})} \log(\frac{n}{\wtil\beta})\max\braces{\brackets{ \frac{2e}{d'} \log\parenth{ n \nout \bconstant}}^{d'/2}, (\Rcomb^2\eta e^3 4 /d')^{d'/2}}\right).
\end{talign}
Our final step is to bound the quantile of the sole remaining data-dependent term, $\Rcomb$.  Since the inputs are $c$-sub-Gaussian \cref{eq:subexp-dist}, Lem.~1 of \citet{dwivedi2024kernel} with $\psi^{-1}(r) = \frac{\sqrt{\log r}}{\sqrt c}$ implies that the $1-\frac{\wtil \beta}{20 \sblock_n}$ quantile of $\Rcomb$ is $O\big(\sqrt{\log(\frac{n}{\wtil \beta})}\big)$, yielding the result.

\section{\pcref{cor:ctt_power_deep_kernel_manifold}}
\label{proof:ctt_power_deep_kernel_manifold}

Our reasoning is identical to that in \cref{proof:ctt_power_deep_kernel} with the manifold Gaussian kernel matrix eigenvalue bound \cref{eq:gsn_lambda_manifold} now substituted for the Euclidean ball bound \cref{eq:gsn_lambda_ball} and the approximate rank setting $r^\star=(\log(n\nout)/c)^{5d^\star/2}$ substituted for \cref{eq:approx-rank-deepctt}.

%% file: appendices/experiment_supplement.tex
\section{Supplementary Experiment Details}
\label{sec:experiment_supplement}

\subsection{Approximating attention experiments}
\label{app:attention_details}

The experiment of \cref{tab:imagenet-acc-time} was carried out using 
Python 3.12.9,  
PyTorch 2.8.0.dev20250407+cu128  
\citep{paszke2019pytorch}, and an Ubuntu 22.04.5 LTS server with an AMD EPYC 7V13 64-Core Processor, 220 GB RAM, and a single NVIDIA A100 GPU (80 GB memory, CUDA 12.8, driver version 570.124.04). 
For reference, attention layer 1 has $(n,d)=(3136,64)$ and attention layer 2 has $(n,d)=(784,64)$.
For each of the first $50$ ImageNet 2012 validation set batches of size $64$, we measured the time required to complete a forward pass through each the approximate softmax matrix \cref{def:att} layer using CUDA events following $10$ warm-up batches to initialize the GPU.
\cref{tab:imagenet-configs} provides the hyperparameter settings for each attention approximation in \cref{tab:imagenet-acc-time}. 

The experiment of \cref{tab:biggan} was carried out using 
Python 3.12.9, PyTorch 2.6.0, and an Ubuntu 22.04.5 LTS server with an Intel(R) Xeon(R) Gold 5218 CPU Processor, 100 GB RAM, and a single NVIDIA A6000 GPU (48 GB memory, CUDA 12.1, driver version 530.30.02). 
For reference, the BigGAN model contains a single attention layer with $4096$ queries in $\reals^{64}$, $1024$ keys in $\reals^{64}$, and $1024$ values in $\reals^{256}$.
We measured the time required to complete a forward pass through  the approximate softmax matrix \cref{def:att} layer using CUDA events following $10$ warm-up batches to initialize the GPU.
\cref{tab:imagenet-configs-biggan} provides the hyperparameter settings for each attention approximation in \cref{tab:biggan}. 

The settings and implementations for all methods other than Thinformer were provided by \citet{zandieh2023kdeformer}, and our experiment code builds on their open-source repository \url{https://github.com/majid-daliri/kdeformer}. 

\begin{table}[h!]
    \centering
    \caption{\tbf{Configurations for the attention approximation methods of \cref{tab:imagenet-acc-time}.}}
    \begin{tabular}{ccc}
        \toprule
        {\bf Attention Algorithm} & {\bf Layer 1 Configuration} & {\bf Layer 2 Configuration}
        \\\midrule
        \multirow{1}{*}{\bf Performer} & \verb|num_features=49| & \verb|num_features=12| 
        \\[1mm]
        \multirow{2}{*}{\bf Reformer} & \verb|bucket_size=49| & \verb|bucket_size=12| \\
        & \verb|n_hashes=2| & \verb|n_hashes=2|
        \\[1mm]
        \multirow{2}{*}{\bf ScatterBrain} & \verb|local_context=49| & \verb|local_context=12| \\
        & \verb|num_features=48| & \verb|num_features=6|
        \\[1mm]
        \multirow{2}{*}{\bf KDEformer} & \verb|sample_size=64| & \verb|sample_size=56| \\
        & \verb|bucket_size=32| & \verb|bucket_size=32|
        \\[1mm]
        \multirow{1}{*}{\bf Thinformer (Ours)} & \verb|g=2| & \verb|g=4|
        \\
        \bottomrule
    \end{tabular}
    \label{tab:imagenet-configs}
\end{table}

\begin{table}[h!]
    \centering
    \caption{\tbf{Configurations for the attention approximation methods of \cref{tab:biggan}.}}
    \begin{tabular}{cc}
        \toprule
        {\bf Attention Algorithm} & {\bf Layer Configuration}
        \\\midrule
        \multirow{1}{*}{\bf Performer} & \verb|num_features=128|
        \\[1mm]
        \multirow{2}{*}{\bf Reformer} & \verb|bucket_size=64| \\
        & \verb|n_hashes=8|
        \\[1mm]
        \multirow{2}{*}{\bf ScatterBrain} & \verb|local_context=32| \\
        & \verb|num_features=128|
        \\[1mm]
        \multirow{2}{*}{\bf KDEformer} & \verb|sample_size=128| \\
        & \verb|bucket_size=64|
        \\[1mm]
        \multirow{1}{*}{\bf Thinformer (Ours)} & \verb|g=2|
        \\
        \bottomrule
    \end{tabular}
    \label{tab:imagenet-configs-biggan}
\end{table}

\subsection{Faster SGD training experiments}
\label{app:sgd_details}
The experiment of \cref{sec:theory-practice-gap} was carried out using Python 3.10, PyTorch 2.0.1, a Rocky Linux 8.9 server with 64 CPU cores (Intel(R) Xeon(R) Platinum 8358 CPU @ 2.60GHz), and a NVIDIA A100 GPU (40 GB memory, CUDA 12.4, driver version 550.54.15). 

Technically, the CD-GraB: SBW algorithm requires an a priori upper bound on the maximum Euclidean norm $b_{\max}$ of any stochastic gradient that it will encounter.  
To conduct our experiment, we first estimate $b_{\max}$ by calculating the maximum gradient Euclidean encountered across $10$ epochs of running SGD with $\khsgd$ reordering.
One would typically not choose to carry out such a two-step procedure in practice, but the experiment serves to demonstrate that the CD-GraB: SBW leads to overly conservative performance even if reasonable upper bound is known in advance.

The settings and implementation for both random reshuffling (RR) and CD-GraB: Greedy were those used in the original logistic regression on mortgage application experiment of  \citet{cooper2023coordinatingdistributedexampleorders}.
Our experiment code builds on the open-source CD-GraB repository \url{https://github.com/GarlGuo/CD-GraB}.
As in \citet{cooper2023coordinatingdistributedexampleorders}, optimization was carried out with a learning rate of $\alpha=0.01$, datapoints were loaded in batches of size $16$, and stochastic gradients were reordered for each datapoint individually.

\begin{figure}[htbp]
  \centering
  \begin{subfigure}[b]{0.43\textwidth} %
    \centering
    \includegraphics[width=\textwidth]{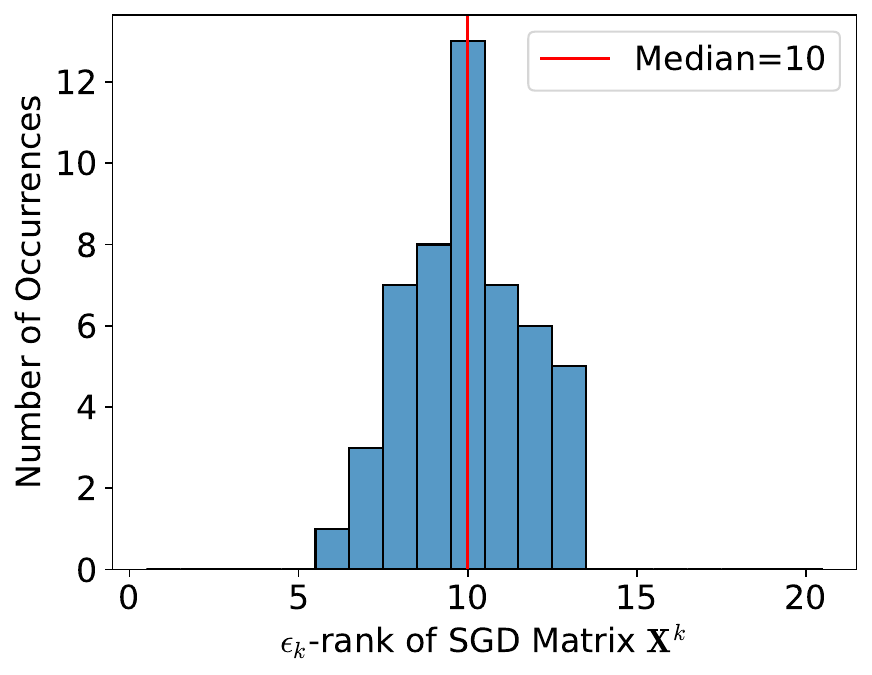} %
    \label{fig:subfig2}
  \end{subfigure}
  \hfill
  \begin{subfigure}[b]{0.47\textwidth} %
    \centering
    \includegraphics[width=\textwidth]{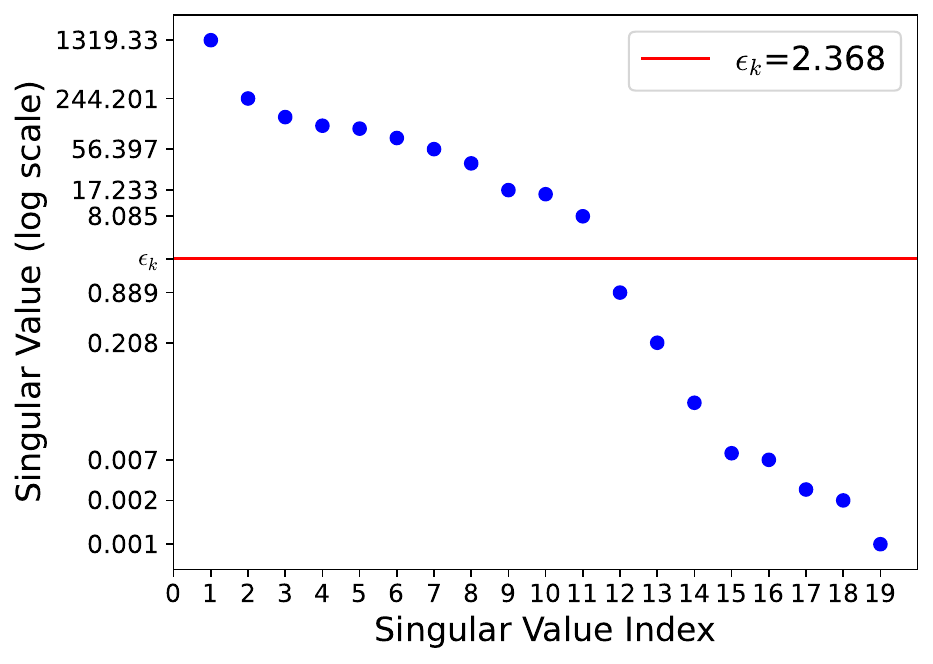} %
    \label{fig:subfig1}
  \end{subfigure}
  \caption{\textbf{Approximate low-rank structure of stochastic gradient matrices.} 
  \emph{Left.} For each epoch $k$ of the \khsgd experiment of \cref{sec:theory-practice-gap}, we record the $\epsilon_k$-rank of the stochastic gradient matrix $\X^k\defeq [\x_1^k,\dots,\x_n^k]^\top\in\reals^{n\times d}$ %
  (see \cref{def:eps-rank,thm:convergence}). Notably, the $\epsilon_k$-ranks are significantly smaller than the ambient dimension $d=19$.
  \emph{Right.} We display the singular values of the first stochastic gradient matrix, $\X^1$. The singular values drop off steeply, resulting in relatively small $\epsilon_k$-ranks.}
  \label{fig:mainfigure}
\end{figure}

\subsection{Cheap two-sample testing experiment}
\label{app:testing_details}

The experiment of \cref{sub:ctt_experiments} was carried out using Python 3.12.9, PyTorch 2.6.0+cu124, and an Ubuntu 20.04.6 LTS server with an AMD EPYC 9554 64-Core Processor, 100 GB RAM, and a single NVIDIA H100 GPU (96 GB memory, CUDA 12.2, driver version 535.154.05). 
Each test is run with replication count $\numperm=100$, nominal level $\alpha=0.05$, and failure probability $\delta=0.5$.
The neural network $\phi$ was trained exactly as in \citet{liu2020learning} (with learning rate $5\times10^{-5}$ and batch size equal to the full training sample size), and
runtime measurements exclude the time required to train $\phi$.
Our experiment code builds on the open-source deep kernel testing 
(\url{https://github.com/fengliu90/DK-for-TST}) and Compress Then Test
(\url{https://github.com/microsoft/goodpoints}) repositories.

\begin{figure}[htbp]
  \centering
    \includegraphics{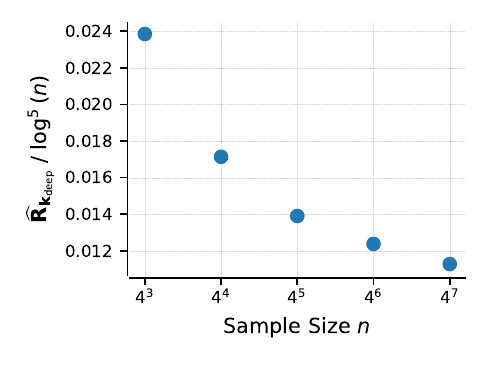} 
    \caption{\textbf{Slow-growing error inflation.} In the experimental setting of \cref{sub:ctt_experiments}, the error inflation factor $\textstyle \errorhat[\kdeep]$ enjoys $O(\log^5(n))$ growth as the ratio $\errorhat[\kdeep] / \log^5(n)$ decreases with $n$. Here, $\textstyle\errorhat[\kdeep]^2$ is defined by \cref{eq:error-inflation-factor-simplified} with $n=m$, $\nin\defeq \frac{2n}{\sblock}$ and $\nout\defeq 2^\ossymb \sqrt{\nin}$ for constants $\beta=0.05$, $\sblock=32$, and $\ossymb=4$.}
    \label{fig:Rk-growth}
\end{figure}